\documentclass[smallextended]{svjour3}       
\smartqed  
\usepackage{epsfig}
\usepackage{amssymb}
\usepackage{ifpdf}
\usepackage{grffile}
\ifpdf
\usepackage{graphicx}
\DeclareGraphicsExtensions{.eps,.pdf,.jpg,.jpeg,.png,.bmp}
\else
\usepackage{graphicx}
\DeclareGraphicsExtensions{.eps}
\fi
\usepackage{mathrsfs}
\usepackage{hyperref}
\usepackage[square,comma,numbers]{natbib}
\usepackage[cmex10]{amsmath}
\usepackage{javen}
\usepackage{bbm}
\usepackage[ruled,vlined,linesnumbered]{algorithm2e}
\usepackage{array}
\usepackage[table]{xcolor}
\usepackage{colortbl}
\usepackage[tight,normalsize,sf,SF]{subfigure}
\usepackage{url}

\usepackage{xr}
\usepackage{anyfontsize}
\usepackage{graphicx}
\usepackage{latexsym}
\usepackage{fixltx2e}
\usepackage[final]{pdfpages}
\newcommand{\authmark}[1]{\textsuperscript{#1}}
\graphicspath{{./}{./Figures/}{../../Figures/people/} }

\definecolor{Gray}{gray}{0.85}
\definecolor{LightCyan}{rgb}{0.88,1,1}

\newcolumntype{a}{>{\columncolor{Gray}}c}
\newcolumntype{b}{>{\columncolor{white}}c}
\newcommand{\tabincell}[2]{\begin{tabular}{@{}#1@{}}#2\end{tabular}}
  \SetAlgoSkip{}
\journalname{Submitted in April, 2014.}
\begin{document}

\title{Constraint Reduction using Marginal Polytope Diagrams for MAP LP Relaxations
}

\titlerunning{Constraint Reduction using Marginal Polytope Diagrams}        

\author{Zhen Zhang\authmark{1,2} \and Qinfeng Shi\authmark{2} \and Yanning Zhang\authmark{1} \and Chunhua
  Shen\authmark{2} \and Anton van den Hengel\authmark{2}   
}

\authorrunning{Z. Zhang\etal} 

\institute{
  1. School of Computer Science and Technology, Northwestern
  Polytechnical University, Xi'an, China, 710129\\
  Shaanxi Provincial Key Laboratory of Speech \& Image Information
  Processing,  Xi'an, China, 710129\\
  2. School of Computer Science, the University of Adelaide, Adelaide,
  Australia, SA, 5005
}

\date{Received: date / Accepted: date}

\maketitle
\begin{abstract}
%
 
 LP relaxation-based message passing algorithms provide an effective tool for MAP inference 
over 
Probabilistic Graphical Models. However, 
different
LP relaxations often have different objective functions and 
variables of differing dimensions,
which presents a barrier to effective comparison and analysis.
In addition, the computational complexity 
of LP relaxation-based methods
grows quickly with the number of constraints.   Reducing the number of constraints without sacrificing the quality of the solutions is thus desirable.   

We propose a unified formulation under which existing MAP LP relaxations 
may be
compared and analysed.
 Furthermore, we propose a new tool called Marginal Polytope Diagrams. Some properties of Marginal Polytope Diagrams are exploited such as node redundancy and edge equivalence. 
We show that 
using Marginal Polytope Diagrams
allows the
number of constraints
to be reduced 
without loosening the LP relaxations. Then, using  Marginal Polytope Diagrams and constraint reduction, we develop three novel message passing algorithms, and demonstrate that two of these show a significant improvement in speed over state-of-art algorithms while delivering a competitive, and sometimes higher, quality of solution.
\end{abstract}


\maketitle

\section{Introduction}


Linear Programming (LP) relaxations have been used to approximate the maximum a posteriori (MAP) inference of Probabilistic Graphical Models (PGMs) \cite{koller2009probabilistic} by enforcing local consistency over edges or clusters. An attractive property of this approach is that it is guaranteed to find the 
optimal MAP solution when the labels 
are integers.
This is particularly significant in light of the fact that 
Kumar \etal showed that LP relaxation provides 
a better
approximation than Quadratic Programming relaxation and Second Order Cone 
Programming relaxation \cite{kumar2009analysis}. Despite their success, 
there remain a variety of large-scale problems that off-the-shelf LP solvers can not solve~\cite{yanover2006linear}. 
%
 Moreover, it has been shown~\cite{yanover2006linear, SontagEtAl_uai08} that LP relaxations have a large gap between the dual objective and the decoded primal objective and fail to find the 
optimal MAP solution
in many real-world problems.    

In response to this shortcoming a number of dual message passing methods have been proposed including
Dual Decompositions \cite{komodakis2007mrf,SonGloJaa_optbook,SontagChoeLi_uai12} and Generalised Max Product Linear Programming (GMPLP) \cite{globerson2007fixing}. 
These methods can still be computationally expensive when there are a large number of constraints in the LP relaxations. It is desirable to reduce the number of constraints 
in order to reduce computational complexity
without sacrificing the quality of the solution. However, this is non-trivial, because for a MAP inference problem the dimension of the primal variable can be different in various LP relaxations. This also presents a barrier for effectively comparing 
the quality of 
two LP relaxations and their corresponding message passing methods. 
Furthermore, these message-passing methods may get stuck in
non-optimal 
solutions
due to the non-smooth dual objectives~\cite{SchwingNIPS2012, hazan2010norm,meshi2012convergence}. 

Our contributions are: 1) we propose a unified form for MAP LP
relaxations, under which existing MAP LP relaxations can be rewritten
as constrained optimisation problems with variables of the same
dimension and objective; 2) we present a new tool which we call the Marginal
Polytope Diagram to effectively compare different MAP LP relaxations.
%
 We show that any MAP LP relaxation in the above unified form has a Marginal Polytope Diagram, and vice versa. We establish propositions to conveniently show the equivalence of seemingly different Marginal Polytope Diagrams; 3) Using Marginal Polytope Diagrams,  we show how to safely reduce the number of constraints (and consequently the number of messages) without sacrificing the quality of the solution, and propose three new message passing algorithms in the dual; 4) we show how to perform message passing in the dual without computing and storing messages (via updating the beliefs only and directly); 5) we propose a new cluster pursuit strategy.  
\section{MAP Inference and LP Relaxations}
We consider MAP inference over factor graphs with discrete states. For
generality,  we will use higher order potentials (where possible) throughout the paper.
\subsection{MAP inference} 
Assume that there are $n$ variables $X_1, \cdots, X_n$, each taking discrete 
states $x_i\in$ Vals$(X_i)$. 
Let  $\Vcal = \{1,\cdots,n\}$ denote the node set, and let $\Ccal$ be a collection of subsets of $\Vcal$. 
$\Ccal$ has an associated
 group of potentials $\boldsymbol{\theta}=\{\theta_c(\xb_c)\in\mathbb{R}|c\in\Ccal\}$, where 
$\xb_c=[x_i]_{i\in c}$. 
Given a graph $G=(\Vcal,\Ccal)$ and potentials $\boldsymbol{\theta}$, we 
consider the following exponential family distribution 
\citep{wainwright2008graphical}:
\begin{equation}
  \label{eq:Bdist}
p(\xb|\boldsymbol{\theta})=\frac{1}{Z}\exp\left(\sum_{c\in\Ccal} 
\theta_c(\xb_c)\right),
\end{equation}
where  
$\xb = [x_1,x_2,\ldots,x_n]\in\Xcal$, and $Z=\sum_{\xb\in\Xcal}\exp(\sum_{c\in\Ccal}\theta_c(\xb_c))$ 
is known as a normaliser, or partition function. The goal of  MAP inference is 
to find the MAP assignment, $\xb^{*}$, that maximises 
$p(\xb|\boldsymbol{\theta})$. That is
\begin{align}
\label{eq:map-cluster}
\xb^{*} =
\mathop{\mathrm{argmax}}_{\xb}\sum_{c\in \mathcal{C}} 
\theta_{c}(\xb_c).
\end{align}
Here we slightly generalise the notation of $\xb_c$ to $\xb_s=[x_i]_{i\in s}$, $\xb_t=[x_i]_{i\in t}$ and $\xb_f=[x_i]_{i\in f}$ where $s,t,f$ are subsets of $\Vcal$ reserved for later use.


\subsection{Linear Programming Relaxations}
By introducing 
\begin{align}
\mub=(\mu_c(\xb_c))_{c\in\Ccal},  \label{def:mub}
\end{align}
the MAP inference problem can be written as an equivalent Linear Programming (LP) problem as follows
\begin{align}
     \displaystyle{\mub^{*} = \mathop{\mathrm{argmax}}_{\mub\in\Mcal(G)}\sum_{c\in\Ccal}\sum_{\xb_c}\mu_c(\xb_c)\theta_c(\xb_c)},\label{eq:LPProb}
\end{align}
in which the feasible set, $\Mcal(G)$, is known as the marginal polytope 
\citep{wainwright2008graphical},  defined as follows
\begin{align}
  \left\{
    \mub \left| 
      \begin{array}{l}
        q(\xb) \geqslant 0,~\forall \xb\\
        \sum_{\xb}q(\xb)=1\\
        \sum_{\xb_{\Vcal\setminus c}}q(\xb)=\mu_c(\xb_c), \forall c\in\Ccal,\xb_c\\
      \end{array}
      \right.
  \right\}.
\end{align}
Here the first two groups of constraints specify that $q(\xb)$ is a
distribution over $\Xcal$, and we refer to the last group of
constraints as the \textit{global marginalisation constraint}, which
guarantees that for arbitrary $\mub$ in $\Mcal(G)$, all $\mu_c(\xb_c),c\in\Ccal$ can be obtained by marginalisation from a common
distribution $q(\xb)$ over $\Xcal$. In general, exponentially
many 
inequality constraints (\ie $q(\xb) \geqslant 0,~\forall \xb$) are required to define a marginal polytope, which makes
the LP hard to solve. Thus \eq{eq:LPProb} is often relaxed with 
a
local marginal
polytope $\Mcal_L(G)$ to obtain 
the following LP relaxation
\begin{align}
\displaystyle{\mub^{*} = 
\mathop{\mathrm{argmax}}_{\mub\in\Mcal_L(G)}\sum_{c\in\Ccal}\sum_{\xb_c}
\mu_c(\xb_c)\theta_c(\xb_c)}.
\end{align}
Different LP relaxation schemes define different local marginal
polytopes. 
A typical local marginal polytope 
defined in
\cite{SonGloJaa_optbook,hazan2010norm} is as follows:
\begin{align}
\label{eq:lpr}
  \left\{ \mub \left|
      \begin{array}{l}
        \sum\limits_{\xb_{c\setminus \{i\}}}\mu_c(\xb_c)=\mu_i(x_i),\forall 
c\in\Ccal,i\in c,x_i\\
        \mu_c(\xb_c)\geqslant 0,
        \sum\limits_{\xb_c}\mu_c(\xb_c)=1,\forall c \in \Ccal, \xb_c
      \end{array}
    \right.
  \right\}.
\end{align}
Compared to the marginal polytope, for arbitrary $\mub$ in a local marginal polytope, all $\mub_c(\xb_c)$ may not be the marginal distributions 
of
a common distribution $q(\xb)$ over $\Xcal$, but there are much fewer constraints in local marginal polytope. As a result, 
the LP relaxation can be solved more efficiently. 
%
%
This is of particular practical significance because
state-of-the-art interior point or simplex LP solvers can only handle 
%
problems with up to a few hundred thousand variables and constraints
while many real-world datasets 
demand 
far more variables and constraints \citep{yanover2006linear,kumar2009analysis}.

Several message passing-based approximate algorithms
\citep{globerson2007fixing,SontagEtAl_uai08,SontagChoeLi_uai12} have been
proposed to solve large scale LP relaxations. Each of them applies
coordinate descent to the dual objective of an LP relaxation problem with a particular local marginal polytope. 
%
Different local marginal polytopes use 
different local marginalisation constraints, which leads to different dual
problems and hence different message updating schemes.  
 
\subsection{Generalised Max Product Linear Programming}
\citet{globerson2007fixing} showed that LP relaxations
can also be solved by message passing, known as Max Product LP (MPLP)
when only node and edge potentials are considered, or Generalised MPLP (GMPLP) (see Section 6 of \citep{globerson2007fixing}) when
potentials over clusters are considered.

In GMPLP, they define $\Ical=\{s|s=c\cap c';c,c'\in\Ccal\}$, and
 \begin{align}
   \mub^g=(\mu_c(\xb_c),\mu_s(\xb_s))_{c\in\Ccal,s\in\Ical}.
   \label{eq:mu_g}
 \end{align}
Then they consider the
following LP relaxation
 \begin{align}
   \displaystyle{\mub^{g*} = 
\mathop{\mathrm{argmax}}_{\mub^g\in\Mcal_L^g(G)}\sum_{c\in\Ccal}\sum_{\xb_c}
\mu_c(\xb_c)\theta_c(\xb_c)}\label{eq:OBJGMPLP},
 \end{align}
 where the local marginal polytope $\Mcal_L^g(G)$ is defined as
   \begin{equation}
  \label{eq:MLGOriginal}
  \left\{
      \boldsymbol{\mu}^g
    \left|
      \begin{array}{l}
        \mu_c(\xb_c) \geqslant 0, \forall c\in\Ccal ,\xb_c\\
        {\sum\limits_{\xb_c}\mu_{c}(\xb_{c})=1, ~\forall c \in \Ccal}\\
        {\sum\limits_{\xb_{c\setminus s}}
        \mu_{c}(\xb_{c})=\mu_{s}(\xb_s), \forall c \in\mathcal{C}, s\in\mathcal{S}_g(c), \xb_s}\\
      \end{array}
    \right.\right\}
\end{equation}
with $\Scal_g(c)=\{s|s\in\Ical,s\subseteq c\}$. To derive a desirable dual 
formulation, they replace the third group of constraints with the following 
equivalent constraints
 \begin{align}
      & \mu_{c\setminus s,s}(\xb_{c\setminus s},\xb_s) =\mu_{c}(\xb_c),  \quad\forall c\in\mathcal{C}, s\in\mathcal{S}_{g}(c), \xb_c\notag,\\
      &  \sum_{\xb_{c\setminus s}}\mu_{c\setminus s,s}(\xb_{c\setminus s},\xb_s)=\mu_{s}(\xb_s), \quad\forall c \in\mathcal{C}, s\in\mathcal{S}_{g}(c), \xb_s \notag
 \end{align}
where $\mu_{c\setminus s,s}(\xb_{c\setminus s},\xb_s)$ is known as the copy variable. Let $\beta_{c\setminus s,s}(\xb_{c \setminus s},\xb_s)$ be the dual variable associated with the first group of the new constraints above, using standard Lagrangian yields the following dual problem:
\begin{align}
\min_{\beta}& \sum_{s\in\Ical}\mathop{\mathrm{max}}_{\xb_s}\sum_{c\in\Ccal,s\in\Scal_g(c)}\mathop{\mathrm{max}}_{\xb_{c\setminus s}}\beta_{c\setminus s,s}(\xb_{c \setminus s},\xb_s)\notag\\
\mathrm{s.t.} &~\theta_c(\xb_c)=\sum_{s\in\mathcal{S}_g(c)}\beta_{c\setminus s,s}(\xb_{c\setminus s},\xb_s), ~\forall c\in\Ccal,\xb_c.
\end{align}
Let $\lambda_{c\rightarrow s}(\xb_s)=\mathop{\mathrm{max}}_{\xb_{c\setminus 
s}}\beta_{c\setminus s,s}(\xb_{c \setminus s},\xb_s)$, they use a coordinate 
descent method to minimise the dual 
by picking up a particular $c\in\Ccal$ and updating all 
$\lambda^*_{c\rightarrow s}(\xb_s)$ as following:
\begin{align}
  &\lambda^*_{c\rightarrow s}(\xb_s)  = -\lambda_{s}^{-c}(\xb_s)+\notag\\
&\frac{1}{\left|\mathcal{S}_g(c)\right|}\hspace{-0.02in}\mathop{\mathrm{max}}
\limits_ { \xb_{ c\setminus 
s}}\bigg[\theta_c(\xb_c)\hspace{-0.02in}+\hspace{-0.02in}\sum\limits_{\hat{s}
\in\mathcal{S}_g(c) }\lambda_{\hat{s }}^{-c}(\xb_{\hat{s}})\bigg],\forall 
s\in\Scal_g(c), \xb_s
\end{align}
where 
$\lambda_{s}^{-c}(\xb_s)=\sum_{\hat{c}\in\{\bar{c}|\bar{c}\in\Ccal,\bar{c}\neq 
c,s\in\Scal_g(\bar{c})\}}\lambda_{\hat{c}\rightarrow s}(\xb_s)$.
At each iteration the dual objective always decreases, thus guaranteeing convergence. Under certain conditions GMPLP finds the exact solution.
Sontag \etal\citep{SontagEtAl_uai08} extended this idea by iteratively adding clusters and reported faster convergence empirically.
\subsection{Dual Decomposition}

Dual Decomposition \citep{komodakis2007mrf,SonGloJaa_optbook} explicitly splits node potentials 
(those potentials of order 1) 
from cluster potentials with order greater than 1, and
rewrites the MAP objective \eq{eq:map-cluster} as 
\begin{equation}
\label{eq:map-dd}
\sum_{i\in \Vcal}\theta_i(x_i)+\sum_{f\in\mathcal{F}}\theta_f(\xb_f),
\end{equation}
where $\mathcal{F}=\{f|f\in\Ccal,|f|>1\}$. By defining $\boldsymbol{\mu}^d=(\mu_i(x_i),\mu_f(\xb_f))_{i\in \Vcal, f\in\Fcal}$, they consider the following LP relaxation:
\begin{align}
\hspace{-0.1in}  \max_{\boldsymbol{\mu}^d\in 
\mathcal{M}_{L}^d(G)}&f_d(\mub^d)\notag\\
\hspace{-0.1in}  f_d(\mub^d)=&\sum_{i\in 
\Vcal}\sum_{x_i}\mu_i(x_i)\theta_i(x_i)+\sum_{f\in\mathcal{F}}
\sum_{\xb_f}\mu_f(\xb_f)\theta_f(\xb_f)\label{eq:OBJDD}
\end{align}
with a different local marginal polytope $\Mcal_{L}^d(G)$ defined as 
\begin{equation}
\label{eq:Md}
\hspace{-0.0in}\left\{\boldsymbol{\mu}^d\geqslant 0
\left|
\begin{matrix}
\sum_{x_i}\mu_i(x_i)=1, \forall i \in \Vcal\\
\sum_{\xb_{f/\{i\}}} \mu_{f}(\xb_f)=\mu_{i}(x_i), \forall f\in 
\mathcal{F},i\in f,x_i\\
\end{matrix}
\right.
\right\}\hspace{-0.03in}.\hspace{-0.2in}
\end{equation}
Let $\lambda_{fi}(x_i)$ be the Lagrangian multipliers corresponding to each 
$\sum_{\xb_{f\setminus\{i\}}} \mu_{f}(\xb_f)=\mu_{i}(x_i)$ for each 
$f\in\Fcal,i\in f, x_i$, one can show that the standard Lagrangian duality is
\begin{align}
L(\lambda) &= \sum_{i\in \Vcal}\max_{x_i}\Big(\theta_i(x_i) 
+\sum_{f\in\{f'|f'\in\Fcal,i\in f'\}}\lambda_{fi}(x_i)\Big) 
\nonumber\\&+\sum_{f\in 
\mathcal{F}}\max_{\xb_f}\Big(\theta_{f}(\xb_f)-\sum_{i\in 
f}\lambda_{fi}(x_i)\Big).
\end{align}
Subgradient or coordinate descent can be used to minimise the dual objective. Since the Dual Decomposition using coordinate descent is closely related to GMPLP and the unified form which we will present, we give the update rule derived by coordinate descent below, 
%
\begin{align}
&\lambda^*_{fi}(x_i)=-\theta_i(x_i)-\lambda_{i}^{-f}(x_i)+\nonumber\\&\frac{1}{
|f|}\max_{\xb_{f\setminus\{i\}}}\hspace{-0.02in}\Big[\theta_{f}(\xb_f)\hspace{
-0.02in}+\hspace{ -0.02in}\sum_{ \hat{i}\in 
f}\theta_{\hat{i}}(x_{\hat{i}})\hspace{-0.02in}+\hspace{-0.02in}\sum_{\hat{i}
\in 
f}\lambda_{\hat{i}}^{-f}(x_{\hat{i}})\Big]\hspace{-0.02in},\hspace{-0.02in}
\forall i\hspace{-0.02in}\in\hspace{-0.02in} 
f,\hspace{-0.02in}x_i\hspace{-0.02in}
\end{align}
where $f$ is a particular cluster from $\Fcal$, and
$\lambda_{i}^{-f}(x_i)=\sum_{\hat{f}\in\{\bar{f}|\bar{f}\in\Fcal,\bar{f}\neq 
f,i \in \bar{f}\}}\lambda_{\hat{f}i}(x_i)$. 

Compared to GMPLP, the local marginal polytope in the Dual Decomposition has 
much fewer constraints. In general for an arbitrary graph 
$G=(\Vcal,\Ccal)$,  $\Mcal_L^d(G)$ is looser than $\Mcal_L^g(G)$ (\ie $\Mcal_L^d(G) \supseteq \Mcal_L^g(G)$) .

\subsection{Dual Decomposition with cycle inequalities}
Recently, Sontag \etal \cite{SontagChoeLi_uai12} proposed a Dual Decomposition with cycle inequalities considering the following LP relaxation
\begin{align}
\label{eq:cycle}
  \mub^{d*}=\max_{\mu\in\Mcal^o_L(G)}f_d(\mub^d)
\end{align}
with a local marginal polytope $\Mcal^o_{L}(G)$, 
\begin{equation}
\label{eq:M0}
\notag
\left\{\hspace{-0.03in}\boldsymbol{\mu}^d\hspace{-0.03in}
\left|
\begin{matrix}
\sum_{x_i}\mu_i(x_i)=1, \forall i \in \Vcal, x_i \\
\sum\limits_{\xb_{f/\{i\}}}\hspace{-0.03in}\mu_{f}(\xb_f)\hspace{-0.03in}
=\hspace{-0.03in}\mu_{i }(x_i),\forall f\in\Fcal, i\in 
f,x_i\\
\sum\limits_{\xb_{f/e}}\hspace{-0.03in}\mu_{f}(\xb_f)\hspace{-0.03in}
=\hspace{-0.03in}\mu_ {e}(\xb_e),\hspace{-0.03in} \forall 
f,e\hspace{-0.03in}\in\hspace{-0.03in}\Fcal,e\hspace{-0.03in}\subset\hspace{
-0.03in } f, \xb_e,|e|\hspace{-0.03in}=\hspace{-0.03in}2, 
|f|\hspace{-0.03in}\ge\hspace{-0.03in}3\\
\boldsymbol{\mu}^d\geqslant 0
\end{matrix}
\right.\hspace{-0.03in}
\right\}
\end{equation}
They added cycle inequalities to tighten the 
problem. Reducing the primal feasible set may reduce the maximum primal objective, which reduces the minimum dual objective. They showed that finding the ``tightest'' cycles, which maximise the decrease in the dual objective, is NP-hard. Thus, instead, they looked for the most ``frustrated'' cycles, which correspond to the cycles with the smallest LHS of their cycle inequalities. Searching for ``frustrated'' cycles, adding the cycles' inequalities and updating the dual is repeated until the algorithm converges. 
\section{A Unified View of MAP LP Relaxations}
In different LP relaxations, not only the formulations of the objective, but also the dimension of primal variable may vary, which makes 
comparison difficult.
By way of illustration, note that the primal variable in GMPLP is $\mub^g=(\mu_c(\xb_c),\mu_s(\xb_s))_{c\in\Ccal,s\in\Ical}$, while in Dual Decomposition the primal variable is $\mub^d=\{\mu_i(x_i),\mu_f(\xb_f)\}_{i\in \Vcal, f\in\Fcal}$. Although $\mub^d$ can be 
reformulated to
$(\mu_c(\xb_c))_{c\in\Ccal}$ if $\Ccal = \{\{i\}|i\in\Vcal\} \cup 
\Fcal$, the variables $\{\mu_s(\xb_s)\}_{s\in\Ical}$ (corresponding to 
intersections) in GMPLP still do not appear in Dual Decomposition. This shows 
that the dimensions of the primal variables in GMPLP and Dual Decomposition are 
different.

\subsection{A Unified Formulation}
When using the local marginal polytope the objective of the LP relaxation depends only 
on those $\mu_c(\xb_c),c\in\Ccal$. We thus reformulate the LP Relaxation into a 
unified formulation as follows:
\begin{align}
  \displaystyle{\mub^{*} = 
\mathop{\mathrm{argmax}}_{\mub\in\Mcal_L(G,\Ccal',\Scal(\Ccal'))}\sum_{c\in\Ccal
} \sum_ { \xb_c}\mu_c(\xb_c)\theta_c(\xb_c)}\label{eq:LPR},
\end{align}
where $\mub$ is defined in \eq{def:mub}. The local marginal polytope, 
$\Mcal_L(G,\Ccal',\Scal(\Ccal'))$, can be defined in a unified formulation as
\begin{align}
\left\{
  \mub\left|
    \begin{array}{l}
      \mu_c(\xb_c) \geqslant 0, \forall c \in \Ccal',\xb_c\\
      \sum\limits_{\xb_c}\mu_c(\xb_c) = 1, \forall c \in\Ccal'\\
      \sum\limits_{\xb_{c\setminus s}}\mu_c(\xb_c) = \mu_s(\xb_s),\forall c \in \Ccal', s\in\Scal(c),\xb_s
    \end{array}
    \right.
    \right\}.\label{eq:ML}
\end{align}
Here  $\Ccal'$ is what we call an extended cluster set, where each $c\in\Ccal'$ 
is called an extended cluster.
$\Scal(\Ccal')=(\Scal(c))_{c\in\Ccal'}$, where
each $s\in\Scal(c)$ is a subset of $c$, which we refer to as a \textit{sub-cluster}.
The choices of $\Ccal'$ and $\Scal(\Ccal')$ correspond to existing or even new 
inference algorithms, which will be shown later, and when specifying  
$\Ccal'$ and $\Scal(\Ccal')$,  we require $\Ccal'\cup 
(\mathop{\cup}_{c\in\Ccal'}\Scal(c))\supseteq \Ccal$.   

The first two groups of constraints in $\Mcal_L(G,\Ccal',\Scal(\Ccal'))$ 
ensure that $\forall c\in\Ccal'$, $\mu_c(\xb_c)$ is a distribution over 
Vals($\xb_c$). We refer to the third group of constraints as \textit{local 
marginalisation constraints}. 

\paragraph{Remarks} The LP formulation in (1) and (2) of
\citep{Werner2010} may look similar to ours. However, the work of
\citep{Werner2010} is in fact a special case of ours. In their
work, an additional restriction  $s \subset c$ for \eq{eq:ML} must be satisfied (see (4) in
\citep{Werner2010}). As a result, their work does not cover the LP
relaxations in  \cite{SontagEtAl_uai08} and GMPLP, where
redundant constraints like $\mu_c(\xb_c)=\mu_c(\xb_c)$ are used to
derive a message from one cluster to itself (see Figure 1 of
\cite{SontagEtAl_uai08}). Our approach, however, is in fact a
generalisation of \cite{SontagEtAl_uai08}, GMPLP and \cite{Werner2010}.
\subsection{Reformulating GMPLP and Dual Decomposition }
Here we show that both GMPLP and Dual Decomposition can be reformulated by \eq{eq:LPR}. 

Let us start with GMPLP first. 
Let $\Ccal'$ be $\Ccal$ and 
$\Scal(\Ccal')$ be $\Scal_g(\Ccal)=(\Scal_g(c))_{c\in\Ccal}$. GMPLP 
\eq{eq:OBJGMPLP} can be reformulated as follows
\begin{align}
  \displaystyle{\mub^{*} = 
\mathop{\mathrm{argmax}}_{\mub\in\Mcal_L(G,\Ccal,\Scal_g(\Ccal))}\sum_{c\in\Ccal
} \sum_ {\xb_c}\mu_c(\xb_c)\theta_c(\xb_c)}, \label{eq:GMPLPref}
\end{align}
where $\Mcal_L(G,\Ccal,\Scal_g(\Ccal))$ is defined as 
  \begin{align}
    \left\{
      \mub\left|
        \begin{array}{l}
          \mu_c(\xb_c) \geqslant 0, \forall c \in \Ccal,\xb_c\\
          \sum\limits_{\xb_c}\mu_c(\xb_c) = 1, \forall c \in\Ccal\\
          \sum\limits_{\xb_{c\setminus s}}\mu_c(\xb_c) = \mu_s(\xb_s),\forall c \in \Ccal, s\in\Scal_g(c),\xb_s
        \end{array}
      \right.
    \right\}\label{eq:MLGrefrom}.
  \end{align}
We can see that \eq{eq:MLGrefrom} and \eq{eq:MLGOriginal} only differ in the dimensions of their variables $\mub$ and $\mub^g$ (see \eq{def:mub} and \eq{eq:mu_g}). Since the objectives in \eq{eq:GMPLPref} and \eq{eq:OBJGMPLP} do not depend on $\mu_s(\xb_s),s\in\Ical$ directly, the solutions of the two optimisation problems \eq{eq:GMPLPref} and \eq{eq:OBJGMPLP} are the same on $\mub$.

For Dual Decomposition, we let $\Ccal_d = \{\{i\}|i\in\Vcal\} \cup \Fcal$, and 

\begin{align}
  \Scal_d(c)=\left\{\begin{array}{ll}
      \emptyset, &~|c| = 1\\
      \{\{i\}|i\in\Ccal\} & ~|c|>1
\end{array}\right..
\end{align} 
Let $\Ccal'$ be $\Ccal_d$ and 
$\Scal(\Ccal')$ be $\Scal_d(\Ccal_d)=(\Scal_d(c))_{c\in\Ccal_d}$. Dual 
Decomposition \eq{eq:OBJDD} can be reformulated as 
  \begin{align}
     \displaystyle{\mub^{*} = 
\mathop{\mathrm{argmax}}_{\mub\in\Mcal_L(G,\Ccal_d,\Scal_d(\Ccal_d))}\sum_{
c\in\Ccal
} \sum_ {\xb_c}\mu_c(\xb_c)\theta_c(\xb_c)}, \label{eq:DDref}
  \end{align}
where $\Mcal_L(G,\Ccal_d,\Scal_d(\Ccal_d))$ is defined as
\begin{align}
  \left\{ 
    \mub\left| 
      \begin{array}{l}
        \mu_c(\xb_c)\geqslant 0,~\forall c\in\Ccal_d,\xb_c\\
        \sum\limits_{\xb_c}\mu_c(\xb_c)=1,~\forall c\in\Ccal_d\\
        \sum\limits_{\xb_{c\setminus s}}\mu_c(\xb_c)=\mu_s(\xb_s),~\forall 
c\in\Ccal_d,s\in\Scal_d(c),\xb_s
      \end{array}
    \right.\hspace{-3mm}
  \right\}.\label{eq:MLDref}
\end{align}

Similarly, for Dual Decomposition with cycle inequalities in \eq{eq:cycle}, we 
define $\Scal_o(c)$ as follows
\begin{align}
  \Scal_o(c)=\left\{\begin{array}{l}
      \emptyset, ~|c|=1\\
      \{\{i,j\}|\{i,j\}\subset c\} ~|c|=3\\
      \{\{i\}\} ~|c|>1,|c|\neq 3
    \end{array}\right..
\end{align}
Let $\Ccal'$ be $\Ccal_d$ and 
$\Scal(\Ccal')$ be $\Scal_o(\Ccal_d)=(\Scal_o(c))_{c\in\Ccal_d}$, we 
reformulate 
the problem in \eq{eq:cycle} as
\begin{align}
  \displaystyle{\mub^{*} = 
\mathop{\mathrm{argmax}}_{\mub\in\Mcal_L(G,\Ccal_d,\Scal_o(\Ccal_d))}\sum_{
c\in\Ccal
} \sum_ {\xb_c}\mu_c(\xb_c)\theta_c(\xb_c)}.
\end{align}

\subsection{Generalised Dual Decomposition}
Note that $\Mcal_L^d(G)$ and $\Mcal_L^o(G)$ in Dual Decomposition are looser than $\Mcal_L^g(G)$. This suggests that for some $\boldsymbol{\theta}$ 
Dual Decomposition may achieve a 
lower quality
solution or slower convergence (in terms of number of 
iterations) than GMPLP \footnote{\scriptsize This does not contradict the 
result reported in \cite{SontagChoeLi_uai12}, where Dual Decomposition with cycle inequalities converges faster in terms 
of running time than GMPLP. In \cite{SontagChoeLi_uai12}, on all their datasets $\Mcal_L^o(G)=\Mcal_L^g(G)$ as 
the order of clusters are at most 3. Dual Decomposition with cycle inequalities runs faster  because it has a
better cluster pursuit strategy. On datasets with higher order potentials, it may have worse performance than GMPLP.}. 
We show using the unified formulation of LP Relaxation in \eq{eq:LPR}, Dual 
Decomposition can be derived on arbitrary local marginal polytopes (including 
those tighter than $\Mcal_L^g(G)$, $\Mcal_L^d(G)$ and $\Mcal_L^o(G)$). We refer 
to this 
new type of Dual Decomposition as Generalised Dual Decomposition (GDD), which forms 
a basic framework for more efficient algorithms to be presented in Section 
\ref{sec:algs}.

\subsubsection{GDD Message Passing}\label{subsubsec:tddmp}
Let $\lambda_{c\rightarrow s}(\xb_s)$ be the Lagrangian multipliers (dual 
variables) corresponding to the local marginalisation constraints 
$\sum_{\xb_{c\setminus s}}\mu_{c}(\xb_c)=\mu_s(\xb_s)$ for each $c\in\Ccal',s\in 
\Scal(c),\xb_s$. Define
\begin{align}
  \Tcal=\Ccal'\cup[\mathop{\cup}_{\hat{c}\in\Ccal'}\Scal(\hat{c})]\label{eq:tdefs},
\end{align}
and the following variables $\forall t\in\Tcal, \xb_t$:
\begin{subequations}
\begin{align}
&\begin{array}{ll}\hat{\theta}_t(\xb_t)&=\mathbbm{1}
(t\in\Ccal)\theta_t(\xb_t)\end{array},\label{eq:hthetadef}\\ 
&\begin{array}{ll}\gamma_t(\xb_t)&=\mathbbm{1}(t\in\Ccal')\sum\limits_{\hat{s} 
\in\Scal(t)\setminus \{t\}}\lambda_{t\rightarrow 
\hat{s}}(\xb_{\hat{s}})\end{array},\label{def:gamma}\\ 
&\begin{array}{ll}\lambda_t(\xb_t)&=\sum\limits_{c\in\{c'|c'\in \Ccal',
t\in\Scal(c')\setminus \{c'\}\} } 
\lambda_{c\rightarrow t}(\xb_t)\end{array}\label{def:lambda},\\
  &\begin{array}{ll}b_t(\xb_t)&= \hat{\theta}_t(\xb_t)+\lambda_t(\xb_t)-\gamma_t(\xb_t)\end{array},\label{eq:bcdef}
\end{align}\label{eq:GDDdefs}
\end{subequations}
where $\mathbbm{1}(S)$ is the indicator function, which is equal to 1 if the 
statement $S$ 
is true and 0 otherwise. Define 
$\lambdab=(\lambda_{c\rightarrow s}(\xb_s))_{c\in\Ccal',s\in\Scal(c)}$, we have
the dual problem  (see derivation in Section \ref{sec:der_tdd_mp} of the
supplementary).

\begin{align}
g(\boldsymbol{\lambda})=& \max\limits_{\begin{matrix} \forall t\in 
\Tcal,\xb_t, \mub_t(\xb_t) \geqslant 0,\\
\sum_{\xb_t}\mu_t(\xb_t)=1\end{matrix}}\bigg[\sum\limits_{c\in\mathcal{C}}
\sum\limits_{\xb_c}\mu_{c}(\xb_c)\theta_{c}(\xb_c)+\notag \\
 & \sum\limits_{c\in\mathcal{C}'}\sum\limits_{s\in \Scal(c)}
\sum\limits_{\xb_{s}}\bigg(\mu_s(\xb_s)-\sum_{\xb_{c\setminus s}}\mu_{c}(\xb_c)\bigg)\lambda_{c\rightarrow s}(\xb_s)\bigg]\notag\\
=& \sum_{t\in\Tcal}\max\limits_{\xb_t}b_t(\xb_t).\label{eq:DualNotArrange}
\end{align}
 %

 \noindent In \eq{eq:DualNotArrange}, if  $c\in\Scal(c)$ for some $c\in\Ccal'$, the variable $\lambda_{c\rightarrow c}(\xb_c)$ will always be cancelled out\footnote{\scriptsize In dual objective other than \eq{eq:DualNotArrange}, $\lambda_{c\rightarrow c}(\xb_c)$ 
 may not be cancelled out (\eg the dual objective used in GMPLP).}. As a result, $\lambda_{c\rightarrow c}(\xb_c)$ can be set to arbitrary value.
To optimise \eq{eq:DualNotArrange}, we use coordinate descent. 
For any $c\in\Ccal'$ fixing all 
$\lambda_{c'\rightarrow s}(\xb_s),c'\in\Ccal',s\in\Scal(c')$ except $\lambdab_{c,\Scal(c)}=(\lambda_{c\rightarrow 
s}(\xb_s))_{s\in\Scal(c)\setminus \{c\}}$ yeilds a sub-optimisation problem, 
\begin{align}
  &\begin{array}{c}
\argmin\limits_{\lambdab_{c,\Scal(c)}}g_c(\lambdab_{c,\Scal(c)})\end{array},
\notag\\
&\begin{array}{l}g_c(\lambdab_{c,\Scal(c)})\hspace{-0.02in}=\hspace{-0.02in}\bigg[\hspace{-0.03in}
\max\limits_
{\xb_c}\hspace{-0.03in}\big[\hat{\theta}_c(\xb_c)\hspace{-0.02in}-\hspace{-0.07in}
    \sum\limits_{s\in\Scal(c)\setminus\{c\}}\hspace{-0.03in}\lambda_{c\rightarrow 
s}(\xb_s)\hspace{-0.02in}+\hspace{-0.02in}\lambda_{c}(\xb_c)\big]+\end{array}\notag\\ 
&\begin{matrix}\sum\limits_{s\in\Scal(c)\setminus\{c\}}\hspace{-0.05in}\max\limits_{\xb_s}\hspace{-0.03in}\big[\hat
{ \theta} _s(\xb_s)\hspace{-0.01in}-\hspace{-0.01in}\gamma_{
s}(\xb_s)\hspace{-0.01in}+\hspace{-0.01in}\lambda_{s}^{-c}(\xb_s)\hspace{-0.01in}+\hspace{-0.01in}\lambda_{c\rightarrow s}(\xb_s)\hspace{-0.01in}\big]\hspace{-0.03in}\bigg]
  \end{matrix}\hspace{-0.03in},\hspace{-0.11in}\label{eq:subopt}
\end{align}

\noindent where $\forall 
s\in\Scal(c)\setminus \{c\},\xb_s$ 
\begin{align}
\hspace{-0.1in}\lambda_s^{-c}(\xb_s)&\hspace{-0.03in}=\hspace{-0.03in}\sum_{
\hat{c}\in\{ c'|c'\in\Ccal,c'\neq 
c,s\in\Scal(c')\setminus \{c'\}\}}\hspace{-0.07in}\lambda_{\hat{c}\rightarrow 
s}(\xb_s).\label{
eq:lambdasminusc}
\end{align}

\noindent A solution is provided in the proposition below.
\begin{proposition}
\label{prop:solution}
$\forall s\in\Scal(c)\setminus \{c\}, \xb_s$, let 
\begin{align}
     &\lambda_{c\rightarrow 
s}^*(\xb_s)=-\hat{\theta}_s(\xb_s)+\gamma_s(\xb_s)-\lambda_s^{-c}
(\xb_s)\notag\\
 &+\frac{1}{|\Scal(c)\setminus\{c\}|}\max_{\xb_{c\setminus 
s}}\bigg[\hat{\theta}_c(\xb_c)+\lambda_c(\xb_c)+\notag\\
 &\sum_{\hat{s}\in\Scal(c)\setminus \{c\}}\left(\hat{\theta}_{\hat{s}}(\xb_{\hat{s}}
)-\gamma_{\hat{s}}(\xb_{\hat{s}})+\lambda_{\hat{s}}^{-c}(\xb_{\hat{s}}
)\right)\bigg],
 \label{eqn:MSGUPD}
\end{align}
then $\lambdab_{c,\Scal(c)}^{*}=(\lambda_{c\rightarrow 
s}^{*}(\xb_s))_{s\in\Scal(c)\setminus \{c\}}$
 is a solution of \eq{eq:subopt}.
\end{proposition}

The derivation of \eq{eq:DualNotArrange} and \eq{eq:subopt}, and the proof of 
Proposition \ref{prop:solution} are provided in Section \ref{sec:der_tdd_mp} in 
the supplementary material. The
$b_t(\xb_t)$ are often referred to as beliefs, and $\lambda_{c\rightarrow 
s}(\xb_s)$ messages (see \cite{globerson2007fixing,SontagEtAl_uai08}). 
In \eq{eqn:MSGUPD}, $\lambda_c(\xb_c)$ and $\gamma_s(\xb_s)$, 
$\lambda_s^{-c}(\xb_s), \forall s\in\Scal(c)\setminus\{c\}$ are known, and they do not depend on $\lambda_{c\rightarrow s}(\xb_s)$.
We summarise the message updating procedure in Algorithm \ref{algo:MSGUPGDD}. 
Dual Decomposition can be seen as a special case of GDD with a specific 
local marginal polytope $\Mcal_L(G,\Ccal,\Scal_d(\Ccal))$ in 
\eq{eq:MLDref}. 


\paragraph{Decoding} The beliefs $b_t(\xb_t), t\in\Tcal$ are computed via \eq{eq:bcdef} to evaluate the dual objective 
and decode an integer solution of the original MAP problem. For a $g(\boldsymbol{\lambda})$ obtained via GDD 
based message passing, we find $\xb^*$ (so called decoding) via 
\begin{align}
\xb_t^*\in\mathop{\mathrm{argmax}}_{\xb_t}b_t(\xb_t),\forall t \in \Tcal.
\label{eq:decode_t}
\end{align}

\noindent Here we use $\in$ instead of $=$ is because there may be multiple maximisers.
In fact, if a node $i \in \Vcal$ is also an extended cluster or sub-cluster 
(\ie 
$\exists t\in \Tcal$, s.t. $t=\{i\}$), 
then we perform more efficient decoding via
\begin{align}
x_i^*\in\mathop{\mathrm{argmax}}_{x_i}b_i(x_i).
\label{eq:decode}
\end{align}

\noindent Further discussion on decoding is deferred to Proposition 
\ref{propos:LocalConsistency} and Section \ref{subsec:stealth}. 
\setlength{\intextsep}{0pt}

\begin{algorithm}[!t]
  \scriptsize{
\caption{{GDD Message Passing}}
\label{algo:MSGUPGDD}
 \SetKwInOut{Input}{input}\SetKwInOut{Output}{output}

 \Input{ $G=(\Vcal,\Ccal)$, $\Mcal_L(G,\Ccal',\Scal(\Ccal'))$, $\Ccal'$, 
$\Scal(\Ccal')$, 
$\boldsymbol{\theta}$, $T_{g}$, $K_{max}$}
 \Output{$\boldsymbol{\lambda} = (\lambda_{c\rightarrow 
s}(\xb_s))_{c\in\Ccal',s\in\Scal(c)\setminus \{c\}}$}
 \BlankLine
 $k=0$,  $g^{0}(\boldsymbol{\lambda})=+\infty$, $\lambdab = \mathbf{0}$\;
 \Repeat{$|g^{k}(\boldsymbol{\lambda})-g^{k-1}(\boldsymbol{\lambda})|<T_g$ or 
$k> K_{max}$}{
   $k=k+1$\;
   \For{$c\in\mathcal{C}'$ }{
     Compute $\lambdab_{c,\Scal(c)}^{*}$ using \eq{eqn:MSGUPD}; update     $\lambdab_{c,\Scal(c)}=\lambdab_{c,\Scal(c)}^{*}$\;
   }
   Computing $\bb$ using \eq{eq:GDDdefs}\;
   $g^{k}(\boldsymbol{\lambda}) = \sum_{t\in \Tcal}\max_{\xb_t}b_{t}(\xb_t)$\;
   
   \tcc{By Proposition \ref{propos:converge}, 
$g(\boldsymbol{\lambda})$ always converges.}
 }}
\end{algorithm}

\subsubsection{Convergence and Decoding Consistency}\label{subsec:converge}
In this part we analyse the convergence and decoding consistency of GDD message passing. 

GDD essentially iterates over $c\in\Ccal'$, and updates the messages via 
\eq{eqn:MSGUPD}. The dual decrease defined below  
\begin{align}
    d(c) = &g_c(\lambdab_{c,\Scal(c)})-g_c(\lambdab^*_{c,\Scal(c)})
    \label{eq:DualDecreaseDef}
  \end{align}
  plays a role in the analysis of GDD. 


\begin{proposition}[Dual Decrease]\label{propos:DualDecrease}
 For any $c\in\Ccal'$, the dual decrease
  \begin{align}
    d(c)  =&\max_{\xb_c}b_c(\xb_c)+\sum_{s\in\Scal(c)\setminus\{c\}}\max_{\xb_s}b_{s}(\xb_s) 
\notag\\
  &- 
\max_{\xb_c}\bigg[b_c(\xb_c)+\sum_{s\in\Scal(c)\setminus\{c\}}b_{s}(\xb_s)\bigg]\geqslant 
0.\label{eq:DualDecrease}
  \end{align}
\end{proposition}
The proof is provided in Section \ref{sec:der_dual_decrease} of the supplementary.
A natural question is whether GDD is convergent, which is answered by the following proposition. 
\begin{proposition} [Convergence]
  \label{propos:converge}
GDD always converges.
\end{proposition}
\begin{proof}
  According to duality and LP relaxation, we have for arbitrary 
$\boldsymbol{\lambda}$, 
  \begin{align}
g(\boldsymbol{\lambda})\geqslant\max_{\xb}\sum_{c\in\Ccal}
\theta_c(\xb_c).
  \end{align}
 By Proposition \ref{propos:DualDecrease} in each single step of coordinate descent,
the dual decrease $d(c)$ is non-negative. Thus GDD message passing produces a monotonically decreasing sequence of 
$g(\boldsymbol{\lambda})$.  Since the sequence has a lower bound, the sequence 
must converge.
\end{proof}

Note that Proposition \ref{propos:converge} does not guarantee 
$g(\boldsymbol{\lambda})$ reaches the limit in finite steps in GDD (GMPLP and 
Dual Decomposition have the same issue). However, in practice we observe that
GDD often converges in finite steps. 
The following proposition in part explains why the decoding in \eq{eq:decode} is 
reasonable.
\begin{proposition} [Decoding Consistency]
\label{propos:LocalConsistency}
If GDD reaches a fixed point in finite steps, then  $\forall c\in \Ccal', s \in \Scal(c)\setminus\{c\}$, there 
exist $\hat{\xb}_c \in \mathop{\mathrm{argmax}}_{\xb_c}b_c(\xb_c)$, and 
$\bar{\xb}_s \in \mathop{\mathrm{argmax}}_{\xb_s}b_s(\xb_s)$, 
$\mathrm{s.t.}~\hat{\xb}_s=\bar{\xb}_s.$
\end{proposition}
\begin{proof}
If GDD reaches a fixed point, $\forall 
c\in\Ccal'$, 
$d(c) = 0$  (see \eq{eq:DualDecrease} ). Otherwise a non-zero dual decrease 
means GDD would not stop.  Thus $\forall c\in \Ccal'$,
\begin{align}
  & \max_{\xb_c}b_c(\xb_c)+\sum_{s\in\Scal(c)\setminus\{c\}}\max_{\xb_s}b_{s}(\xb_s) =\notag\\
&\max_{\xb_c}\bigg[b_c(\xb_c)+\sum_{s\in\Scal(c)\setminus\{c\}}b_{s}(\xb_s)\bigg],\vspace{
-0.05in}
\end{align}
which completes the proof.
\end{proof}
\IncMargin{1em}
\begin{algorithm}[!t]
\caption{{Belief Propagation Without Messages}}
\label{algo:MSGUPD}
\scriptsize{
 \SetKwInOut{Input}{input}\SetKwInOut{Output}{output}

 \Input{ $G=(\Vcal,\Ccal)$, $\Mcal_L(G,\Ccal',\Scal(\Ccal'))$, $\Ccal'$, 
$\Scal(\Ccal')$, 
$\boldsymbol{\theta}$, $T_g$, $K_{max}$,  }
 \Output{ $\bb=(b_t(\xb_t))_{t\in \Tcal}$;}
 \BlankLine
 $\bb=(\hat{\theta}_t(\xb_t))_{t\in \Tcal}$, $k=0$, $g^{0}(\boldsymbol{\lambda})=+\infty$\;
 \Repeat{$|g^{k}(\boldsymbol{\lambda})-g^{k-1}(\boldsymbol{\lambda})|<T_g$ or 
$k\geqslant K_{max}$}{
   $k=k+1$\;
   \For{$c\in\mathcal{C}'$ }{
     Compute $\bb_{c,\Scal(c)}^*$ using \eq{eq:bupdating}; update    $\bb_{c,\Scal(c)}=\bb_{c,\Scal(c)}^*$\;
   }
   $g^{k}(\boldsymbol{\lambda}) = \sum_{t\in 
\mathcal{T}}\max_{\xb_t}b_{t}(\xb_t)$\;
   
   \tcc{By Proposition \ref{propos:converge}, 
$g(\boldsymbol{\lambda})$ always converges.}
 }
}
\end{algorithm}\DecMargin{1em}
Proposition \ref{propos:LocalConsistency} essentially states that there exist two 
maximisers that agree on $\xb_s$. This in part justifies decoding via 
\eq{eq:decode} (\eg $s$ is a node). Further discussion on decoding is
provided in Section \ref{subsec:stealth}. 

It's obvious that the solution of GDD is exact, if the gap between the dual 
objective and the decoded primal objective is zero. Here we show 
that the other requirements for the exact solution also hold.

\begin{proposition} \label{propos:exact}
If there exists $\xb$ that maximises  $b_t(\xb_t),\forall t\in \Tcal$, the solution of GDD is 
exact.
\end{proposition}
The proof is provided in Section \ref{sec:proof_propos_exact} of 
the supplementary. 

Proposition 4 and 5 generalise the results of Section 1.7 in \citep{SonGloJaa_optbook}. 


\subsection{Belief Propagation Without Messages}\label{subsec:tddbp}
GDD involves updating many messages (\eg  $\lambda_c(\xb_c)$ 
and 
$\gamma_s(x_s), \lambda_c^{-s}(x_s),\forall c \in\Ccal', s\in \Scal(c)\setminus\{c\}$). These messages are 
then used to compute the beliefs $b_t(\xb_t), \forall t \in \Tcal$ (see 
\eq{eq:bcdef}). Here we show that we can directly update the beliefs without 
computing and storing messages.

When optimising \eq{eq:subopt}, $b_c(\xb_c)$ 
and $b_s(\xb_s),s\in\Scal(c)\setminus\{c\}$ are determined by 
$\lambdab_{c,\Scal(c)}(\xb_c)$ (see \eq{eq:basfunctionlambda} in 
supplementary). Thus let
$\bb_{c,\Scal(c)}=(b_c(\xb_c),b_s(\xb_s))_{s\in\Scal(c)\setminus\{c\}}$ be the 
beliefs determined by $\lambdab_{c,\Scal(c)}$, and 
$\bb_{c,\Scal(c)}^*$ be the beliefs determined by 
$\lambdab_{c,\Scal(c)}^{*}$. 
We have the 
following proposition.
\begin{proposition}\label{propos:bp}
When optimising \eq{eq:subopt}, the beliefs 
$\bb_{c,\Scal(c)}^*$  can be computed from a
$\bb_{c,\Scal(c)}$ determined by arbitrary $\lambdab_{c,\Scal(c)}$ as 
following:
 \begin{align}
 b_s^{*}(\xb_s)\hspace{-0.01in}\hspace{-0.03in}  =& 
\frac{\hspace{-0.02in}\max\limits_{\xb_{c\setminus 
s}}\hspace{-0.04in}\bigg[\hspace{-0.02in}b_c(\xb_c)\hspace{-0.02in}+\hspace{-0.02in}\sum_{\hat{s
}\in 
\Scal(c)\setminus\{c\}}\hspace{-0.02in}b_{\hat{s}}(\xb_{\hat{s}})\hspace{-0.02in}\bigg]}{|\Scal(c)\setminus\{c\}|}\hspace{-0.02in},
\forall s\hspace{-0.03in}\in \hspace{-0.03in}
\Scal(c)\hspace{-0.03in}\setminus\hspace{-0.03in}\{\hspace{-0.01in}c\hspace{-0.01in}\}\hspace{-0.02in},\hspace{-0.02in}\xb_s\hspace{-0.07in}\notag\\
b_c^{*}(\xb_c)\hspace{-0.01in} =& b_c(\xb_c)\hspace{-0.02in}+\hspace{-0.07in}\sum_{\hat{s}\in 
\Scal(c)\setminus\{c\}}\hspace{-0.07in}b_{\hat{s}}(\xb_{\hat{s}}) - 
\hspace{-0.07in}\sum_{\hat{s}\in\Scal(c)\setminus\{c\}}\hspace{-0.07in}b^*_{\hat{s}}(\xb_{\hat{s}}),\forall 
\xb_c.\hspace{-0.05in} \label{eq:bupdating}
 \end{align}
\end{proposition}
The proof is provided in Section \ref{sec:der_BP} of the supplementary.

The reason that $\bb_{c,\Scal(c)}^*$ can be computed using  
$\bb_{c,\Scal(c)}$ from an arbitrary $\lambdab_{c,\Scal(c)}$, is because 
$\lambda_{c\rightarrow s}(\xb_s), s\in\Scal(c)\setminus\{c\}$ is cancelled out (\eg see 
\eq{eq:bsbssum} in supplementary) 
during the calculation. 
In the first iteration, beliefs are initialised via 
$b_t(\xb_t)=\hat{\theta}_t(\xb_t),\forall t\in\Tcal $, and then we can
use $\bb_{c,\Scal(c)}$ from previous iterations to compute $\bb^*_{c,\Scal(c)}$ without messages and the potentials.

\paragraph{Memory Conservation} For message updating based
methods such as GDD message passing in Algorithm \ref{algo:MSGUPGDD},
Max-Sum Diffusion \citep{kovalevsky1975diffusion,Werner2010}, GMPLP\citep{globerson2007fixing} and etc,
both messages and potentials need to be stored in order to compute new
messages (see \eqref{eqn:MSGUPD} for example). However, the proposed
belief propagation without messages (summarised in Algorithm
\ref{algo:MSGUPD}) only needs to store beliefs. The beliefs can simple take the space of potentials (\ie initialisation), and then update.
For a graph $G$, we assume that each node takes $k = |$  Vals$(X_i)|$
states,  and let $\Mcal_L(G,\Ccal',\Scal(\Ccal'))$ be a local marginal polytope with specific $\Ccal'$ and
$\Scal(\Ccal')$. Then we need the following space to store beliefs,
potentials and messages: 
\begin{subequations}\label{eq:mem_b_p_m}
  \begin{align}
    \text{Mem}_{beliefs}& = \sum_{t\in\Tcal}k^{|t|},\label{eq:mem_b}\\
    \text{Mem}_{potentials}& = \sum_{c\in\Ccal}k^{|c|},\\
    \text{Mem}_{messages}& =
    \sum_{c\in\Ccal}\sum_{s\in\Scal(c)\setminus \{c\}}k^{|s|}.
  \end{align}
\end{subequations}
Recall the definition of $\Tcal$ in \eqref{eq:tdefs}, it is easy to
show that 
\begin{align}
  \text{Mem}_{beliefs} =& \sum_{c\in\Ccal}k^{|c|} +
  \sum_{s\in\Tcal\setminus \Ccal}k^{|s|}\notag\\
  = &  \sum_{c\in\Ccal}k^{|c|} +
  \sum_{s\in(\Ccal\cup (\cup_{c\in\Ccal}\Scal(c))\setminus
    \Ccal}k^{|s|}\notag\\
  \leqslant & \sum_{c\in\Ccal}k^{|c|} + \sum_{s\in
    \cup_{c\in\Ccal}\Scal(c)}k^{|s|}\notag\\
  \leqslant &  \sum_{c\in\Ccal}k^{|c|} + \sum_{c\in\Ccal}\sum_{s\in\Scal(c)\setminus \{c\}}k^{|s|}\notag\\
  = &\text{Mem}_{potentials} + \text{Mem}_{messages}.
\end{align}
This means that belief propagation without messages always uses less memory than message passing. If we choose $\Tcal = \Ccal$, and $\Scal(t) =
\{t'|t'\subset t\}$, the memory for message passing has a simpler form
\begin{align}
  \label{eq:mem_msgupd}
  \text{Mem}_{potentials} + \text{Mem}_{messages} &=
  \sum_{t\in\Tcal}k^{|t|} + \sum_{t\in\Tcal}\sum_{s \subset
    t}k^{|s|},\notag\\
  & = \sum_{t\in\Tcal}k^{|t|} +
  \sum_{t\in\Tcal}\big[(1+k)^{|t|}-k^{|t|} - 1\big]\notag\\
  & = \sum_{t\in\Tcal}\big[(1+k)^{|t|} - 1\big].
\end{align}
%
  Let $k=2$, $|t|=10$, we have
  $k^{|t|}=1024$, and $(k+1)^{|t|}-1=59048$, where message updating
  based methods uses approximately $59$ times memory as belief propagation without messages. 

\begin{figure}[t]
\centering
\includegraphics[width=40mm,height=30mm]{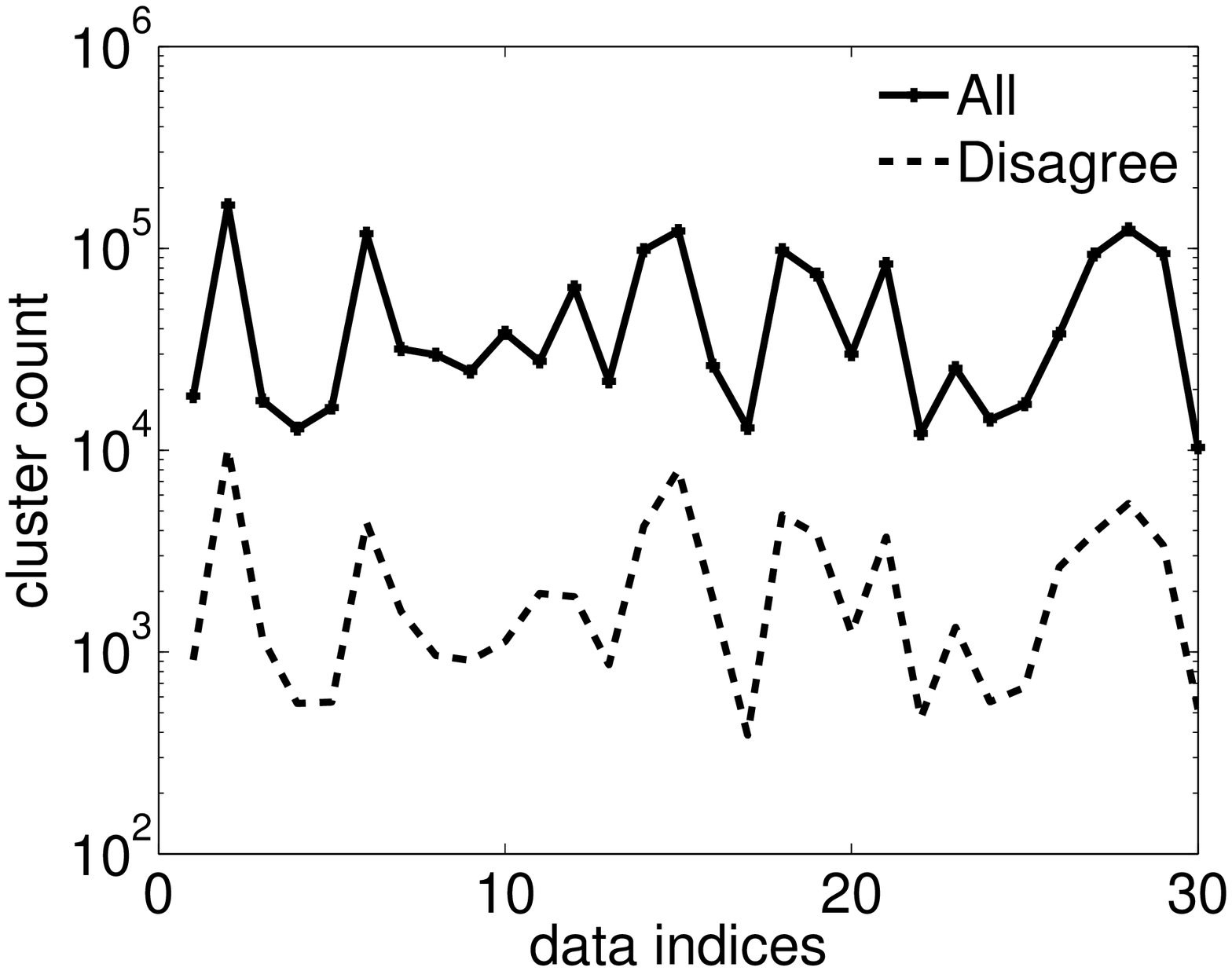}
\hspace{0.15in}
\includegraphics[width=40mm,height=30mm]{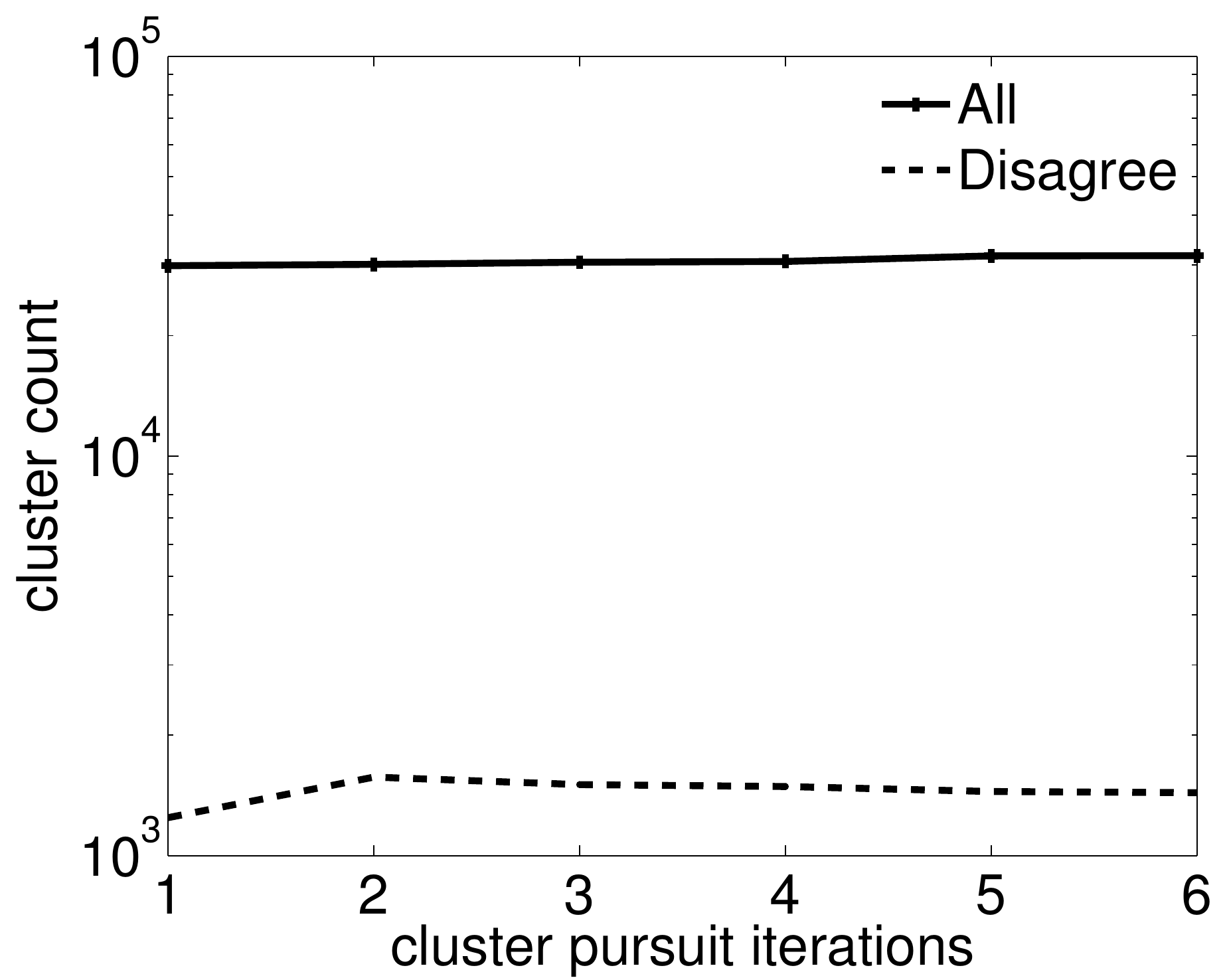}

\caption{{ ``stealth'' clusters (solid) v.s. disagreeing 
``stealth'' 
clusters (dashed) on dataset \textit{sidechain}. {\bf Left}: cluster pursuit for 
the first time for all problems (alphabetically sorted). {\bf Right}: multiple 
cluster pursuits on problem \textit{1qb7}.}}
\label{fig:cluaddcmp}
\end{figure}

\subsection{``Stealth'' Cluster Pursuit}\label{subsec:stealth}
If GDD does not find the exact solution, then there exists a gap
between the dual objective of GDD and the decoded primal
objective. Various approaches, including
\citep{komodakis2008beyond,Werner2008,SontagEtAl_uai08,batra2011tighter},
try to tighten the gap. 
These approaches typically involve two steps: 1) creating a dictionary of clusters, and then 2) search for a best cluster in the dictionary w.r.t. some score function. 
They typically use a \emph{fixed} dictionary of clusters. For example, \citet{SontagEtAl_uai08} consider the
dictionary as all possible triplets; \citet{Werner2008} considers
all order 4 cycles for grids type graphs. 

Dual decrease is a popular choice of the score function. For example, \citet{SontagEtAl_uai08}, brutal force search over the dictionary for the cluster with maximum dual decrease. \citet{batra2011tighter}
accelerate the process by computing primal dual gap for all clusters in the dictionary first (computing primal dual gap is much cheaper than computing dual decrease), and then only search over the clusters with non-zero primal dual gaps for the cluster with maximum dual decrease. 


In this section, we show a new cluster pursuit strategy, which
dynamically generates a dictionary of clusters instead of using a fixed one. The new cluster pursuit strategy is based on a special type of cluster which we call ``stealth'' clusters.
\begin{definition}[``Stealth'' cluster]
\label{def:stealth}
 $\forall c_1,c_2\in\Ccal'$, we say $c=c_1\cup c_2$ is a ``stealth'' cluster
 if
\begin{enumerate} 
\item $\exists s\in\Tcal$ s.t. $s\in\Scal(c_1)\setminus \{c_1\}, s\in\Scal(c_2)\setminus \{c_2\}$, and
\item $\nexists \hat{c}\in\Ccal'$ s.t. $c_1\in\Scal(\hat{c})\setminus\{\hat{c}\}, c_2\in\Scal(\hat{c})\setminus\{\hat{c}\}$.
\end{enumerate} 
\end{definition}

Proposition \ref{propos:LocalConsistency} essentially says there exist two 
maximisers that agree on $\xb_s$. There is however a situation where the decoding from different clusters may 
disagree. Let us consider two clusters $c_1,c_2$ in Definition \ref{def:stealth}. According to Proposition 
\ref{propos:LocalConsistency} there exists a maximiser $\xb_s'$ of 
$b_{s}(\xb_s)$ that agrees with $b_{c_1}(\xb_{c_1})$, and a maximiser $\xb_s''$ 
of $b_{s}(\xb_s)$ that agrees with $b_{c_2}(\xb_{c_2})$. However, if  
$b_s(\xb_s)$ has multiple maximisers, the maximisers $\xb_s'$ and $\xb_s''$ 
could be different. This means  $b_{c_1}(\xb_{c_1})$ and $b_{c_2}(\xb_{c_2})$ 
may disagree on their overlap. Adding $c_1\cup c_2$ into $\Ccal'$ with 
$\Scal(c_1\cup c_2)=\{s|s\in\Tcal, s\subset(c_1\cup c_2) \}$ at least one 
maximiser of $b_{c_1}(\xb_{c_1})$ and $b_{c_2}(\xb_{c_2})$ will become the same 
according to Proposition \ref{propos:LocalConsistency}. This
observation yields a strategy to
dynamically generate dictionary of clusters to tighten the relaxation.

With similar derivation as in \cite{SontagEtAl_uai08}, adding a new ``stealth'' cluster $c=c_1\cup c_2$ with 
$\Scal(c)=\{s|s\in\Tcal,s\subset c\}$, the dual decrease after one message 
updating for 
 $c$ is
\begin{align}
d_1(c) =& g_c(\lambdab_{c,\Scal(c)})-g_c(\lambdab^*_{c,\Scal(c)})\notag\\
 =&\sum_{s\in\mathcal{S}(c)\setminus\{c\}}\max_{\xb_s}b_s(\xb_s)\hspace{-0.5mm}-\hspace{-0.5mm}\max_{\xb_c}\sum_{s\in 
\mathcal{S}(c)\setminus\{c\}}b_s(\xb_s)\label{eqn:decreasepursuit}.
\end{align}
In practice, we add clusters that will lead to largest dual decrease. With these observations, a new cluster pursuit strategy, which we call ``stealth'' cluster
pursuit strategy, is summarised in Algorithm \ref{algo:CP}.
\begin{algorithm}[!t]
\scriptsize{
  \caption{{GDD with ``Stealth'' Cluster Pursuit}}
  \label{algo:CP}
  \SetKwInOut{Input}{input}\SetKwInOut{Output}{output}
  \Input{ $G=(\Vcal,\Ccal)$, $\Mcal_L(G,\Ccal',\Scal(\Ccal'))$, $\Ccal'$, 
$\Scal(\Ccal')$, 
$\boldsymbol{\theta}$,\\
threshold $T_{g}, T_{a}$, max iterations $K^1_{max}$ and $K^2_{max}$, \\
max running time$\mathscr{T}_{max}$, cluster addition size $n$ }
 \Output{$\xb^{*}$}
  $l=0$\;
  Initialise  $\boldsymbol{\lambda}^0 = \mathbf{0}$; 
$\bb^{0}:=(b^{0}_t(\xb_t))_{t\in\Tcal}=(\hat{\theta}_t(\xb_t))_{t\in\Tcal}$\;
  \Repeat{$|g(\boldsymbol{\lambda})-p(\xb^*|\boldsymbol{\theta})|\leqslant 
T_a$ or running time$>\mathscr{T}_{max}$}
  {
  $l=l+1$;   $\Tcal=\Ccal'\cup(\cup_{c\in\Ccal'}\Scal(c))$;   $\mathcal{P}=\emptyset$\;
  $K=\mathbbm{1}(l=1)K^1_{max}+\mathbbm{1}(l>1)K^2_{max}$\;
  \BlankLine
  Run Algo. \ref{algo:MSGUPGDD}    
$(\boldsymbol{\lambda}^{l},\bb^l)=$GDD$(G,\Mcal_L,
\Ccal', \Scal(\Ccal'),\boldsymbol{\theta}, T_{g}, K, 
\boldsymbol{\lambda}^{l-1})$\;
 or Algo. \ref{algo:MSGUPD}    
$\bb^l=$BP$(G,\Mcal_L,
\Ccal', \Scal,\boldsymbol{\theta}, T_{g}, K, \bb^{l-1})$\;
  \BlankLine
  
  \For{$t\in \Tcal$}{
     \For{$c_1,c_2\in\{c|c\in\Ccal', t\in\Scal(c)\setminus \{c\}\}$}{
    $\{\bar{\xb}_{c_1}\}=\mathop{\mathrm{argmax}}_{\xb_{c_1}}b^l_{c_1}(\xb_{c_1}) 
$\;
    $\{\hat{\xb}_{c_2}\}=
\mathop{\mathrm{argmax}}_{\xb_{c_1}}b^l_{c_2}(\xb_{c_2})$\;
    \If{$\nexists \bar{\xb}_{c_1},\hat{\xb}_{c_2}, ~\mathrm{s.t.}~ \hat{\xb}_t = \bar{\xb}_t 
$}{
      $\mathcal{P}=\mathcal{P}\cup\{c_1\cup c_2\}$\;
      $\Scal(c_1\cup c_2)=\{s|s\in\Tcal, s\subset(c_1\cup c_2) \}$\;
      Compute $d_1(c)$ according to \eq{eqn:decreasepursuit}\;
     }}
   }
   Add the $n$ clusters in $\mathcal{P}$ with largest $d_1(c)$ to  
$\mathcal{C}'$\;
   For all new added $c$, $\forall s\in\Scal(c)\setminus\{c\},\xb_s, \lambda_{c\rightarrow 
s}^{l}(\xb_s)=0$\;
  \BlankLine
   Decode $\xb^{*}$ using \eq{eq:decode_t} or \eq{eq:decode}\;
   $g(\boldsymbol{\lambda})=\sum_{t\in\Tcal}\max_{\xb_t}b_t(\xb_t)$\;
   }
  }
  
    
\end{algorithm}
In a nutshell, we search for disagreeing clusters which maximise the dual 
decrease. In the worst case, this can be slow if too many disagreeing 
``stealth'' clusters exist. However, in practice it is very fast. As we can see 
from Figure \ref{fig:cluaddcmp}, the number of disagreeing ``stealth'' clusters 
is far less (about $1\%$ to $10\%$) than the total number of ``stealth'' 
clusters, which leads to a 
significant speed up. More importantly, ``stealth'' cluster pursuit makes our 
 feasible set tighter than that of LP relaxation in \eq{eq:lpr} and GMPLP, 
which 
in turn are tighter than Dual Decomposition \cite{SonGloJaa_optbook}.

``Stealth'' cluster pursuit  may get bigger and bigger clusters which would become prohibitively expensive to solve. In our experiments, it's always computationally affordable. When it isn't, one can use low order terms to approximate $b_t$. Furthermore both the frustrated cycle search strategy in \cite{SontagChoeLi_uai12} and acceleration via evaluating primal dual gap first in \cite{batra2011tighter} are applicable to GDD as well.

\paragraph{Convergence and Consistency} ``Stealth'' cluster pursuit 
can be seen as adding new constraints only. No matter 
which clusters are added to $\Ccal'$, the dual decrease $d_1(c)$ is always 
no-negative. Thus GDD with ``stealth'' cluster pursuit still have the same 
convergence and consistency properties as original GDD presented in Section
\ref{subsec:converge}.


\section{Marginal Polytope Diagrams}

LP relaxation based message passings can be slow if there are too many constraints. This motivates us to 
seek ways of reducing the number of constraints to reduce computational 
complexity without sacrificing the quality of the solution. Here we first 
propose a new tool which we call marginal polytope diagrams. Then with this 
tool, we show how to reduce constraints without loosening the optimisation 
problem.  

\begin{definition}[Marginal Polytope Diagram]
  Given a graph $G=(\Vcal,\Ccal)$, $G^M=(\Vcal^M,\Ecal^{M})$ with node set 
$\Vcal^M$ and edge set $\Ecal^M$ is said to be a 
  \emph{marginal polytope diagram}
  of $G$ if 
  \begin{enumerate}
  \item $\Ccal\subseteq\Vcal^M\subseteq 2^{\Vcal}$ and
  \item a directed edge from $c$ to $s$ deonted by $(c\rightarrow s)$, belongs 
to $\Ecal^M$ only nodes if $c,s\in\Vcal^M,~s\subseteq c$.
  \end{enumerate}
\end{definition}
\paragraph{Remarks} Previous work in this vein includes
Region graphs \cite{yedidia2005constructing,koller2009probabilistic} and Hasse 
diagrams (a.k.a poset diagrams) 
\cite{pakzad2005estimation, wainwright2008graphical, mceliece2003belief}. 
What distinguishes marginal polytope diagrams, however, is the fact that
%
the receivers of an edge can be a subset of the senders, where 
Region graphs and Hasse diagrams require that the edge's receivers
must be a \emph{proper} subset of the senders. For example, the definition of
region graph in Page 419 of \citep{koller2009probabilistic}, requires that a receiver must be a \emph{proper} subset of a
sender in region graph. Likewise in Page 16 of \citep{yedidia2005constructing}, the authors state that a
region graph must be a directed acyclic graph, which means
that the edge's receivers
must be a \emph{proper} subset of the senders (otherwise there would be
a loop from a region to itself). 
This is particularly significant since some dual message passing algorithms (including GMPLP) send messages from a cluster $c$ to itself.
Hasse diagrams \cite{pakzad2005estimation} have a further restriction, which corresponds to a special case of  Marginal Polytope Diagram, where for arbitrary $v_1,v_2\in\Vcal^M,v_1\subset v_2$ edge 
$(v_1\rightarrow 
v_2)\in\Ecal^M$ if and only if $\nexists v_3\in\Vcal^M$, s.t. 
$v_2\subset v_3 \subset 
v_1$. 
 In some LP relaxation based message passing algorithms, some local marginalisation constraints from a 
cluster to itself 
are required, thus violating the proper subset requirement.
Both Hasse diagrams and Region graphs are 
inapplicable in this case.
For example, in Section 6 of 
\cite{globerson2007fixing}, GMPLP sends messages from one cluster to 
itself, which requires 
a local marginalisation constraint from one cluster to itself. In 
\cite{SontagEtAl_uai08}, the message $\lambda_{ij \to ij}$ (in their Figure 1) is from the 
edge $ij$ to itself, which requires a local marginalisation 
constraint from the edge to itself. The proposed 
marginal polytope diagram not only handles the above situations, but also provides 
a natural correspondence among marginal polytope diagram, local marginal 
polytope and MAP message passing.

\begin{figure}[!t]
\centering
 \includegraphics[width=0.5\textwidth]{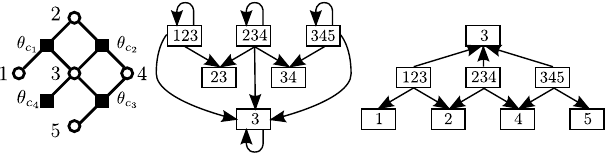} 
  \caption{{An example of marginal polytope diagram.  {\bf Left:} 
a factor graph 
$G=(\Vcal,\Ccal)$ with $\Vcal=\{1,2,3,4,5\}$ and 
$\Ccal=\{c_1,c_2,c_3,c_4\}$, where 
$c_1=\{1,2,3\},~c_2=\{2,3,4\}$,  $c_3=\{3,4,5\}$ and $c_4=\{3\}$. {\bf Middle:} 
a marginal 
polytope diagram associated with the local marginal polytope  
$\Mcal_L(G,\Ccal,\Scal_g(\Ccal))$ in GMPLP. {\bf Right:} a marginal polytope 
diagram 
associated with the local marginal polytope $\Mcal_L(G,\Ccal,\Scal_d(\Ccal))$ 
in dual 
decomposition. 
}}\label{fig:MPDExamples}

\end{figure}
In marginal polytope diagrams, we use rectangles to represent nodes (to differentiate from a graph of graphical models) similar to Hasse 
diagrams (see Section 4.2.1 in \cite{wainwright2008graphical}). An example is given in 
Figure \ref{fig:MPDExamples}. For 
the marginal polytope diagram associated with GMPLP (in Figure \ref{fig:MPDExamples} middle) has edges from a node to itself, which are not 
allowed in 
both Region graphs and Hasse diagrams. Also edges like 
$(\{2,3,4\}\rightarrow \{3\})$ are not allowed in Hasse diagrams. We choose to use the term diagram instead of graph in order to distinguish from graphs in graphical models.

\paragraph{From local marginal polytope to diagram} Given a graph 
$G=(\Vcal,\Ccal)$ and arbitrary local marginal polytope 
$\Mcal_L(G,\Ccal',\Scal(\Ccal'))$, 
the corresponding 
marginal polytope diagram
$G^M=(\Vcal^M,\Ecal^{M})$ can be constructed as:
\begin{enumerate}
\item $\Vcal^M=\Tcal$ ($\Tcal$ is defined in \eq{eq:tdefs}),
\item $\forall c\in\Ccal',$ if $s\in\Scal(c)$, then $(c\rightarrow 
s)\in\Ecal^{M}$.
\end{enumerate}

\paragraph{From diagram to local marginal polytope}
Given a graph $G=(\Vcal,\Ccal)$ and a marginal polytope diagram 
$G^M=(\Vcal^M,\Ecal^{M})$, the corresponding local marginal polytope  
$\Mcal_L(G,\Ccal',\Scal(\Ccal'))$ can be recovered as follows:
\begin{align}
  \label{eq:ReconstructML}
  \Ccal'&=\{c|c\in\Vcal^M,\exists (c\rightarrow s) \in \Ecal^{M}\},\notag\\
  \Scal(c) &= \{s|(c\rightarrow s)\in \Ecal^{M}\},\forall 
c\in\Ccal'.
\end{align}

\subsection{Equivalent Edges}\label{sec:eqedges}

\begin{definition}[Edge equivalence]\label{def:EQE}
For arbitrary $G=(\Vcal,\Ccal)$, and marginal polytope diagram 
$G^M=(\Vcal^M,\Ecal^M)$ of $G$,  $\forall c_1,c_2,t\in\Vcal^M,t\subseteq 
c_1, 
t\subseteq c_2$, given
   \begin{align}
     \Ucal = \big\{&\sum_{\xb_{c\setminus 
s}}\mu_c(\xb_c)=\mu_s(\xb_s), \forall (c\rightarrow s)\in\notag 
\\
     &(\Ecal^M\setminus\{(\hat{c}\rightarrow t)|\hat{c}\in\Vcal^M,t\subseteq 
\hat{c}\}),\xb_s\big\},\label{eq:Ucaldef}
   \end{align}
   edges $(c_1\rightarrow t)$ and $(c_2\rightarrow t)$ are said to be 
\emph{equivalent} w.r.t $G^M$ denoted by $(c_1\rightarrow t)\Leftrightarrow 
(c_2\rightarrow 
t)$, if the following holds:
   \begin{align}
     &\Big\{\mub|\Ucal\cup\{\sum_{\xb_{c_1\setminus 
t}}\mu_{c_1}(\xb_{c_1})=\mu_t(\xb_t),\forall \xb_t\}\Big\}\notag\\=&\Big\{\mub 
| 
\Ucal \cup \{\sum_{\xb_{c_2\setminus 
t}}\mu_{c_2}(\xb_{c_2})=\mu_t(\xb_t),\forall \xb_t\}\Big\}.
   \end{align}
\end{definition}Note that the 
definition of edge equivalence does not require $(c_1\rightarrow t)$ and 
$(c_2\rightarrow t)$ from $\Ecal^M$.
Checking whether two edges are equivalent via Definition 
\ref{def:EQE} might be inconvenient. In fact, edge equivalence can be read directly from a marginal polytope diagram. 
\begin{figure}
\centering
\begin{minipage}[t]{0.4\textwidth}\centering
\includegraphics[width=0.8\textwidth]{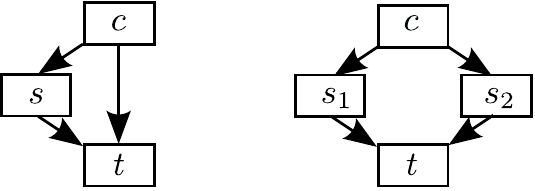}

 \caption{{Two types of edge equivalence in Proposition 
\ref{propos:eq_edge}. {\bf Left:} Type 1 $(c\rightarrow t)\Leftrightarrow 
(s\rightarrow t)$. {\bf Right:} Type 2 $(s_1\rightarrow t)\Leftrightarrow 
(s_2\rightarrow 
t)$.}}\label{fig:edge_equi_two_type}
\end{minipage}
\quad\quad
\begin{minipage}[t]{0.4\textwidth}
 \centering
\includegraphics[width=0.8\textwidth]{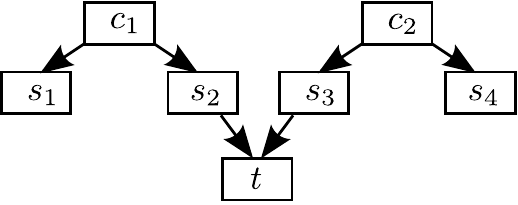}

 \caption{{A counter example for the simpler edge equivalence 
definition.} }\label{fig:EquiExam}
 
\end{minipage}
\end{figure}
\begin{proposition}
 \label{propos:eq_edge}
Given a graph $G=(\Vcal, \Ccal)$ and a marginal polytope diagram 
$G^M=(\Vcal^M,\Ecal^M)$ of $G$, we have
\begin{enumerate}
\item $\forall c,s,t\in\Vcal^M$, $t\subset s\subset c$, if $(c\rightarrow 
s)\in\Ecal^M$, then $(c\rightarrow t)\Leftrightarrow (s\rightarrow t)$;
\item If $(c\rightarrow s_1),(c\rightarrow s_2)\in\Ecal^M$, then $\forall 
t\in\Vcal^M, t\subset s_1,t\subset s_2$, $(s_1\rightarrow t)\Leftrightarrow 
(s_2\rightarrow t)$.
\end{enumerate}
\end{proposition}
The proof is provided in Section \ref{sec:Proofofeqedge} of supplementary. In 
Figure \ref{fig:edge_equi_two_type} we give examples of the two types of edge 
equivalence in Proposition \ref{propos:eq_edge}. Furthermore, the following 
proposition always holds.
\begin{proposition}\label{propos:edge_equi_relation}
 For arbitrary $G=(\Vcal,\Ccal)$, and marginal polytope diagram 
$G^M=(\Vcal^M,\Ecal^M)$ of $G$, edge equivalence w.r.t. $G^M$ is an 
equivalence relation.
\end{proposition}
The proof is provided in Section \ref{sec:proof_equi_relation} of 
supplementary. 

\paragraph{Simpler edge equivalence definition?} Readers may wonder why in Definition \ref{def:EQE}, all edges sent to $t$ are removed (see \eq{eq:Ucaldef}), instead of only removing two edges $(c_1\rightarrow t),(c_2\rightarrow t)$. The answer is that if we did so, the resulting edge equivalence (we call it simpler edge equivalence) would no longer be an equivalence relation. 
To see this, we can replace $\Ecal^M\setminus \{(\hat{c}\rightarrow t)|\hat{c}\in\Vcal^M\}$ 
with $\Ecal^M\setminus\{(c_1\rightarrow t),(c_2\rightarrow t)\}$ in 
\eq{eq:Ucaldef}, and see
a counter example in Figure \ref{fig:EquiExam}.
In considering whether $(s_1\rightarrow t)\notin \Ecal^M$ and $(s_4\rightarrow t)\notin 
\Ecal^M$ are equivalent, we need to consider two constraint sets 
$\Ucal_{a}=\{\sum_{\xb_{c\setminus s}}\mu_c(\xb_c)=\mu_s(\xb_s), \forall (c\rightarrow s) \in \Ecal^M\cup\{(s_1\rightarrow t)\},\xb_s\}$ and 
$\Ucal_{b}=\{\sum_{\xb_{c\setminus s}}\mu_c(\xb_c)=\mu_s(\xb_s), \forall (c\rightarrow s) \in \Ecal^M\cup\{(s_4\rightarrow t)\},\xb_s\}$. Then
by the fact $(s_1\rightarrow t) \Leftrightarrow (s_2\rightarrow t)$,
$(s_3\rightarrow t) \Leftrightarrow (s_4\rightarrow t)$, $(s_2\rightarrow 
t)\in\Ecal^M, (s_3\rightarrow t)\in\Ecal^M$, we have $\{\mub|\Ucal_{a}\}= \{\mub|\Ucal_{b}\}$. Thus $(s_1\rightarrow t)\Leftrightarrow(s_4\rightarrow t)$ by the simpler edge equivalence definition. However, by the simpler edge equivalence definition $(s_2\rightarrow t)$, $(s_3\rightarrow t)$ are not equivalent  in general. This means transitivity does not hold. Thus the simpler edge equivalence is not an equivalence relation. Figure \ref{fig:EquiExam} is not a counter example for Definition \ref{def:EQE}, because $(s_1\rightarrow t)$, 
$(s_4\rightarrow t)$ are not equivalent by Definition \ref{def:EQE}.

With edge equivalence, we can see that given a marginal polytope diagram 
$G^M=(\Vcal^M,\Ecal^M)$ of a graph $G=(\Vcal,\Ccal)$, the following two 
operations would not change the corresponding local marginal polytope. 
\begin{enumerate}
\item Adding a new edge that is equivalent to an existing edge in $\Ecal^M$ 
\label{oper_add};
\item Removing one of two existing equivalent edges in 
$\Ecal^M$\label{oper_remove}.
\end{enumerate}

By composing the two operations above, we are able to derive a series of 
operations which would not change the corresponding local marginal polytope. 
For illustration, using Operation \ref{oper_add}) first and then using 
Operation \ref{oper_remove}) would lead to an operation: replacing an existing 
edge in $\Ecal^M$ with an equivalent edge. Repeating Operation 
\ref{oper_remove}) we can merge several equivalent edges in $\Ecal^M$ to one 
edge. Repeating  Operation \ref{oper_remove}) and then using 
Operation \ref{oper_add}) we can replace a group of equivalent edges in 
$\Ecal^M$ with a new equivalent edge.

\subsection{Redundant Nodes}

There is a type of node,
the removal of which from
a marginal polytope diagram does not 
change the local marginal polytope. 
\begin{definition}[Redundant Node]
  For any graph $G=(\Vcal,\Ccal)$, and  marginal polytope diagram 
$G^M=(\Vcal^M,\Ecal^M)$ of $G$, we say $v\in\Vcal^M\setminus \Ccal$ is a 
\emph{redundant node} w.r.t. $G^M$, if 
  \begin{align}
    &\Big\{\mub| \sum_{\xb_{c\setminus s}}\mu_c(\xb_c)=\mu_s(\xb_s),\forall 
(c\rightarrow s)\in\Ecal^M,\xb_s\Big\}\notag\\
    =&\Big\{\mub| \sum_{\xb_{c\setminus s}}\mu_c(\xb_c)=\mu_s(\xb_s),\forall 
(c\rightarrow s)\in\hat{\Ecal}^M,\xb_s\Big\}\notag
  \end{align}
  where 
  \begin{align*}
  \hat{\Ecal}^M=&\Big[\Ecal^M\setminus (\{(c\rightarrow v)\in\Ecal^M\} \cup 
\{(v 
\rightarrow s)\in\Ecal^M\})\Big] \\
  &\cup\Big\{(c\rightarrow s)|(c\rightarrow v)\in\Ecal^M,(v\rightarrow 
s)\in\Ecal^M\Big\}.
  \end{align*}
\end{definition}
The following proposition provides an easy way to find redundant nodes in a 
marginal polytope diagram. 
\begin{proposition}
 \label{propos:renodes}
Given a graph $G=(\Vcal,\Ccal)$ and a marginal polytope diagram 
$G^M=(\Vcal^M,\Ecal^M)$, $v\in\Vcal^M\setminus \Ccal$ is a redundant node 
w.r.t. $G^M$ if 
either
of the following statements is true:
\end{proposition}
\begin{enumerate}
\item There is only one $(c\rightarrow v)\in\Ecal^M$;
\item All $(c\rightarrow v)\in\Ecal^M$ are equivalent w.r.t. $G^M$ according to 
Definition 
\ref{def:EQE}.
\end{enumerate}
\begin{proof}
Let us consider the first case where there is only one $(c\rightarrow 
v)\in\Ecal^M$. It is easy to check that
  \begin{subequations}
  \begin{eqnarray}
  \left\{
      \mub\left|
        \begin{array}{l}
        \forall(v\rightarrow 
s)\in\Ecal^M,\xb_s, \sum\limits_{\xb_{c\setminus 
s}}\mu_c(\xb_c)=\mu_s(\xb_s);\\
        \forall \xb_v,~~\sum\limits_{\xb_{c\setminus 
v}}\mu_c(\xb_c)=\mu_v(\xb_v)
        \end{array}
        \right.
    \right\}\label{eqn:UB}\\
    = \left\{\mub\left|
        \begin{array}{l}
          \forall \xb_v,~\sum\limits_{\xb_{c\setminus 
v}}\mu_c(\xb_c)=\mu_v(\xb_v);\\
          \forall (v\rightarrow s) \in \Ecal^M, 
\xb_s, ~\sum\limits_{\xb_{v\setminus s}}\mu_v(\xb_v)=\mu_s(\xb_s)
        \end{array}
      \right.
    \right\}.\label{eqn:UA}
  \end{eqnarray}
\end{subequations}
\begin{figure}[!t]
  \centering
  \includegraphics[width=0.5\textwidth]{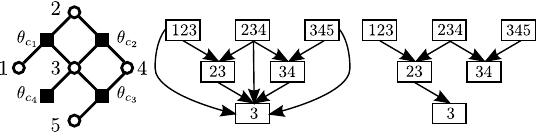}
  \caption{{An example of equivalent edges removal.  {\bf Left:} 
  a factor graph 
$G=(\Vcal,\Ccal)$ with $\Vcal=\{1,2,3,4,5\}$ and 
$\Ccal=\{c_1,c_2,c_3,c_4\}$, where 
$c_1=\{1,2,3\},~c_2=\{2,3,4\}$,  $c_3=\{3,4,5\}$ and $c_4=\{3\}$. {\bf Middle:} 
a marginal polytope diagram whose edges $(\{1,2,3\} \rightarrow 
\{3\})$,$(\{2,3,4\} \rightarrow \{3\})$, $(\{3,4,5\} \rightarrow 
\{3\})$,$(\{2,3\} \rightarrow \{3\})$ and $(\{3,4\}\rightarrow\{3\})$ are 
equivalent.  {\bf Right:} a marginal polytope diagram after removal of 
equivalent edges. 
}}
  \label{fig:RCEdges}
\end{figure}
Since $v\notin \Ccal$, $\mu_v(\xb_v)$ is not part of $\mub$ in \eq{def:mub}. 
Note that removing $\sum_{\xb_{c\setminus v}}\mu_c(\xb_c) = 
\mu_v(\xb_v),~\forall 
\xb_v$ from \eq{eqn:UB} would not change the set of $\mub$ described by 
\eq{eqn:UB} as no other variables depend on $\mu_v(\xb_v)$. That is, 

   \begin{align}
  &\Big\{\mub|\sum_{\xb_{\hat{c}\setminus 
\hat{s}}}\mu_{\hat{c}}(\xb_{\hat{c}})=\mu_{\hat{s}}(\xb_{\hat{s}}),\forall 
(\hat{ c}
\rightarrow \hat{s})\in\Ecal^M_1, \xb_{\hat{s}}\Big\}\notag\\
    =&\Big\{\mub|\sum_{\xb_{\hat{c}\setminus 
\hat{s}}}\mu_{\hat{c}}(\xb_{\hat{c}})=\mu_{\hat{s}}(\xb_{\hat{s}}),\forall 
(\hat{c}
\rightarrow \hat{s})\in\Ecal^M_2,\xb_{\hat{s}}\Big\}\notag,
  \end{align}
where $\Ecal^M_1=\{(c\rightarrow s)|(v\rightarrow 
s)\in\Ecal^M\}$, and 
$\Ecal^M_2=\{(c\rightarrow v)\}\cup\{(v\rightarrow s)\in\Ecal^M\}$.
Let $A = \{\mub|\sum_{\xb_{\hat{c}\setminus 
\hat{s}}}\mu_{\hat{c}}(\xb_{\hat{c}})=\mu_{\hat{s}}(\xb_{\hat{s}}),\forall 
(\hat{ c } \rightarrow \hat{s})\in\Ecal^M\setminus \Ecal^M_2, \xb_{\hat{s}}\}.$
We have  
\begin{align}
    &\Big\{\mub|\sum_{\xb_{\hat{c}\setminus 
\hat{s}}}\mu_{\hat{c}}(\xb_{\hat{c}})=\mu_{\hat{s}}(\xb_{\hat{s}}),\forall 
(\hat{c}
\rightarrow \hat{s})\in\Ecal^M_1,\xb_{\hat{s}}\Big\} \cap A \notag\\
    =&\Big\{\mub|\sum_{\xb_{\hat{c}\setminus 
\hat{s}}}\mu_{\hat{c}}(\xb_{\hat{c}})=\mu_{\hat{s}}(\xb_{\hat{s}}),\forall(\hat{
c}
\rightarrow \hat{s})\in\Ecal^M_2, \xb_{\hat{s}}\Big\} \cap A 
\notag.
\end{align}
Since  ${\Ecal}^M=(\Ecal^M\setminus\Ecal^M_2)\cup\Ecal^M_2$ and 
$\hat{\Ecal}^M=(\Ecal^M\setminus\Ecal^M_2)\cup\Ecal^M_1$,
  \begin{align}
    &\Big\{\mub| \sum_{\xb_{\hat{c}\setminus 
\hat{s}}}\mu_{\hat{c}}(\xb_{\hat{c}})=\mu_{\hat{s}}(\xb_{\hat{s}}),\forall 
(\hat{c}\rightarrow \hat{s})\in\hat{\Ecal}^M,\xb_{\hat{s}}\Big\}\notag\\
    =&\Big\{\mub| \sum_{\xb_{\hat{c}\setminus 
\hat{s}}}\mu_{\hat{c}}(\xb_{\hat{c}})=\mu_{\hat{s}}(\xb_{\hat{s}}),\forall 
(\hat{c}\rightarrow 
\hat{s})\in\Ecal^M,\xb_{\hat{s}}\Big\}.\notag
  \end{align}
Hence $v$ is a redundant node by definition.

Now let us consider the second case, where there are multiple equivalent 
$(c\rightarrow v)\in\Ecal^M$. By edge equivalence we can keep one of 
$(c\rightarrow v)\in\Ecal^M$ and remove the rest without changing the local 
marginal polytope, which becomes the first case.
\end{proof}

\section{Constraint Reduction}
\begin{figure}[!t]
  \centering
  \includegraphics[width=0.5\textwidth]{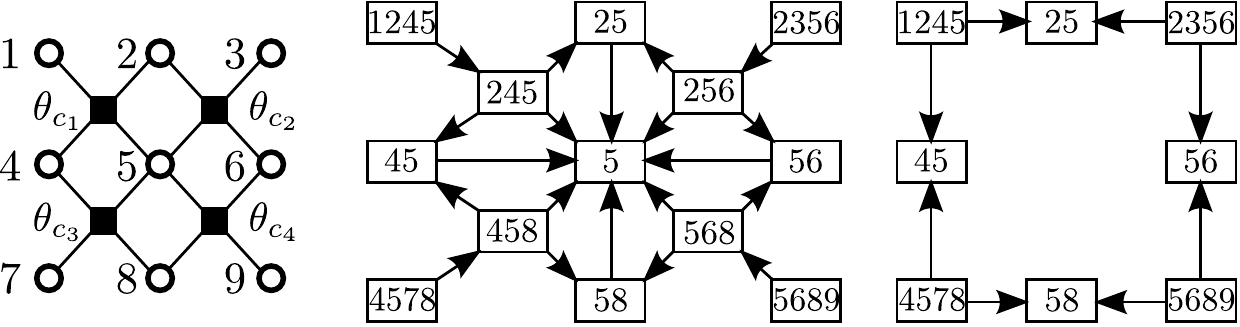}

  \caption{{An example of redundant node removal. {\bf Left:} 
  a factor graph $G=(\Vcal,\Ccal)$ with $\Vcal=\{1,2,3,4,5,6,7,8,9\}$, and 
$\Ccal=\{c_1,c_2,c_3,c_4\}$, where 
$c_1=\{1,2,4,5\},~c_2=\{2,3,5,6\},~c_3=\{4,5,7,8\}$, and $c_4=\{5,6,8,9\}$. 
{\bf Middle:} a marginal polytope diagram with  
  redundant nodes $\{2,4,5\}$, $\{2,5,6\}$, $\{4,5,8\}$, $\{5,6,8\}$ and 
$\{5\}$. {\bf Right:} marginal polytope diagram after redundant nodes removal.}}
  \label{fig:RCExample}
\end{figure}

Equivalent edges offer an effective way to reduce constraints. By composing the 
two basic operations in Section \ref{sec:eqedges}, we can get a series of 
operations which would not change the corresponding local marginal polytope. 
For illustration we can partition 
$\Ecal^M$ into several equivalent classes by equivalence between edges, and 
we can simply pick up arbitrarily many edges in each equivalent 
class to get a new marginal polytope diagram with 
fewer
edges. As each edge 
corresponds to a local marginalisation constraint, the number of constraints can 
be efficiently reduced by the above operations.  As shown in Figure 
\ref{fig:RCEdges}, all edges to node $\{3\}$ are equivalent, thus we can keep 
just one of them in the diagram to keep the local marginal polytope unaltered yet
with fewer constraints. 

Given a graph $G=(\Vcal,\Ccal)$, and a marginal polytope diagram 
$G^M=(\Vcal^M,\Ecal^M)$ of $G$, if a node $v\in\Vcal^M$ is a redundant node, 
we 
can reduce the number of constraints by
 removing $v$ from $\Vcal^M$ to obtain a 
marginal polytope diagram as follows:
\begin{align}
 G^M_R=&(\Vcal^M_R,\Ecal^M_R),  \Vcal^M_R= \Vcal^M\setminus \{v\},\notag\\
 \Ecal^M_R=&[\Ecal^M\setminus (\{(c\rightarrow v)\in\Ecal^M\}\cup 
\{(v\rightarrow s)\in\Ecal^M\})]&\notag\\
 &\cup\{(c\rightarrow s)|(c\rightarrow v),(v\rightarrow 
s)\in\Ecal^M\},
\end{align}
and according to the definition of redundant nodes, $G^M_R$ and $G^M$ correspond to 
the same local marginal polytope.

Figure \ref{fig:RCExample} gives an example of redundant node removal. In the 
middle diagram, one can see that nodes  $\{2,4,5\}$, $\{2,5,6\}$, $\{4,5,8\}$, 
and $\{5,6,8\}$ are redundant  because their indegrees are 1. For node 
$\{5\}$, as all edges to $\{5\}$ are equivalent, node $\{5\}$ is also 
redundant. 
Removal of these redundant nodes leads to fewer constraints without changing 
the local marginal polytope.

It is worth mentioning that \cite{Kolmogorov2012} is perhaps closest idea to ours in spirit. However, \citep[Proposition 2.1]{Kolmogorov2012} considers the removal of edges only, whereas ours considers the removal of both edges and nodes.

\subsection{Is the Minimal Number of Constraints Always a Good Choice?}
Using marginal polytope diagrams one can safely reduce the number of constraints without 
altering the local marginal polytope. 
On one hand,  fewer constraints means  fewer belief updates on 
$b_t(\xb_t),t\in\Tcal$, which leads to lower run time per iteration. On 
the 
other hand,  fewer constraints means  fewer coordinates (\ie search 
directions), 
and thus that the algorithm 
is more likely to get stuck at corners due to the 
non-smoothness of the dual objective and the nature of coordinate descent 
(this has been often observed empirically too). This means that minimal number of constraints is not 
always a good choice. As we shall see in the next section, a trade-off between 
the minimal number of constraints and the maximal number of constraints 
performs best.

\section{From Constraint Reduction To New Message Passing Algorithms}
\label{sec:algs}
Here we propose three new efficient algorithms for MAP inference, using different 
constraint reduction strategies (via marginal polytope diagrams). All three 
algorithms are based on the GDD belief propagation procedure which is 
equivalent 
to GDD message passing, thus the theoretical properties in Section 
\ref{subsec:converge} also hold for these three algorithms.

To derive new algorithms, we first construct a local marginal polytope as 
an initial local marginal polytope, and then by different 
constraint reduction strategies we get three different algorithms. For arbitrary graph 
$G=(\Vcal,\Ccal)$, we let $\Ccal'$ be $\Ccal_0=\{c'|c'\subseteq 
c,c\in\Ccal\}$ and $\Scal(\Ccal')$ be 
$\Scal_0(\Ccal_0)=(\Scal_0(c))_{c\in\Ccal_0}$ with 
$\Scal_0(c)=\{s|s\in\Ccal_0,s\subset c\}$ to  construct a local marginal 
polytope $\Mcal_L(G,\Ccal_0,\Scal_0(\Ccal_0))$ as a initial local marginal 
polytope. Thus the marginal polytope diagram is 
$G^M_0=(\Vcal^M_0,\Ecal^M_0)$ where
\begin{align}
  \Vcal^M_0&=\{v|v\subseteq c,c\in\Ccal\}\notag,\\
  \Ecal^M_0&=\{(c\rightarrow s)|c,s\in\Vcal^M_0,s\subset 
c\}.\label{eq:vsesdef}
\end{align}

The marginal polytope diagram $G^M_0$ and GDD provide the base for all three 
algorithms. 
%
 \subsection{Power Set Algorithm}
 In the first algorithm which we call Power Set algorithm, we do not remove any 
redundant nodes in $G^M_0$. One can see that $\forall (c\rightarrow 
t)\in\Ecal^M_0,|c|-|t| > 1$, $\exists s\in\Vcal^M_0, s.t. |s|=|t|+1, 
(c\rightarrow s)\in\Ecal^M_0$, which suggests $(c \rightarrow t)\Leftrightarrow 
(s\rightarrow t)$. Using equivalent edges we get a marginal polytope 
diagram $G^M_p=(\Vcal^M_p,\Ecal^M_p)$ as follows
\begin{align}
  \Vcal^M_p&=\Vcal^M_0\notag,\\
  \Ecal^M_p&=\{(c\rightarrow s)|c,s\in\Vcal^M_p,s\subset 
c,|c|=|s|+1\},
\end{align}
which corresponds to the same local marginal polytope as $G^M_0$.
Thus we define  $
  \Ccal_p=\{c|c\subseteq\hat{c},\hat{c}\in\Ccal\}$, 
$\Scal_p(c)=\{s|(c\rightarrow s) \in\Ecal^M_p\}$, and let  
$\Ccal'$ be $\Ccal_p$, and $\Scal(\Ccal')$ be
$\Scal_p(\Ccal_p)=(\Scal_p(c))_{c\in\Ccal'}$, the 
corresponding 
local marginal polytope becomes $\Mcal_L(G,\Ccal_p,\Scal_p(\Ccal_p))$.
By the fact that $c\notin \Scal_p(c), \forall c\in\Ccal_p$, applying belief propagation without messages to 
$\Mcal_L(G,\Ccal_p,\Scal_p(\Ccal_p))$ yields the following belief propagation 
form $\forall c\in\Ccal_p$:
\begin{align}
  b_s^{*}(\xb_s)&=\frac{1}{|c|}\max_{\xb_{c\setminus 
s}}[b_c(\xb_c)+\sum_{\hat{s}\in\Scal_p(c)}b_{\hat{s}}(\xb_{\hat{s}})],\forall 
s\in\Scal_p(c),\xb_s\notag\\
b_c^*(\xb_c)&=b_c(\xb_c)\hspace{-0.03in}+\hspace{-0.03in}\sum_{\hat{s}
\in\Scal_p(c)}b_{\hat{s}} (\xb_{\hat{s}}
)\hspace{-0.03in}-\hspace{-0.03in}\sum_ { \hat{s}\in\Scal_p(c)
}b_{\hat{s}}^*(\xb_{\hat{s}}),\forall \xb_c.
\end{align}

\subsection{$\pi$-System Algorithm}
\label{sec:pi-system-message}
The second algorithm which we call $\pi$-System algorithm, is based on  the
$\pi$-system \citep{kallenberg2002foundations} extended from $\Ccal$. Such a 
$\pi$-system denoted by $\Ccal_{\pi}$ has the following properites:
\begin{enumerate}
  \item if $c\in\Ccal$, then $c\in\Ccal_{\pi}$;
  \item if $c_1,c_2\in\Ccal_{\pi}$, then $c_1\cap c_2 \in\Ccal_{\pi}$.
\end{enumerate}
We can construct $\Ccal_{\pi}$ using the properties above by assigning all 
elements in $\Ccal$ to $\Ccal_{\pi}$ and adding intersections repeatedly to 
$\Ccal_{\pi}$.
\begin{proposition}
\label{prop:pi}
All $v\in\Vcal^M_0\setminus \Ccal_{\pi}$ are redundant nodes w.r.t $G^M_0$.
\end{proposition} 
\begin{proof}
Since $\Ccal\subseteq \Ccal_{\pi}$, for any $v\in\Vcal^M_0\setminus 
\Ccal_{\pi}$,  we have $v\in \Vcal^M_0\setminus \Ccal$. Now we prove the 
proposition by proving that all edges to $v$ are equivalent.

Let $P_v=\{p|(p\rightarrow v)\in\Ecal^M_0\}$, and $\Ccal_v=\{c|c\in\Ccal, 
v\subset c\}$. We let $s=\cap_{c\in\Ccal_v}c$, and we must have $v\subseteq s$. 
Moreover, if $v=s$ we have $v=\cap_{c\in\Ccal_v}c\in\Ccal_{\pi}$, this 
contradicts the fact that $v\in \Vcal^M\setminus \Ccal_{\pi}$. Thus we must 
have $v\subset s$. Then, by the definition of $\Vcal^M_0$ and $\Ecal^M_0$ in 
\eq{eq:vsesdef}, $\forall p\in P_v$, $\exists c\in\Ccal_v$, s.t. 
$p\subseteq c$. By the fact that $s=\cap_{\hat{c}\in\Ccal_v}\hat{c}$, we have 
$s\subseteq c$. Thus if $p=c=s$, then $(p\rightarrow v)\Leftrightarrow (s 
\rightarrow v)$ naively holds. If only one of $p$ and $s$ is 
equal to $c$,  we have $(p\rightarrow v)\Leftrightarrow 
(s \rightarrow v)$ by Proposition \ref{propos:eq_edge} (the first case). If 
both $p$ and $s$ are not equal to $c$, by Proposition 
\ref{propos:eq_edge} (the second case) we have $(p\rightarrow v)\Leftrightarrow 
(s \rightarrow v)$. As a result, all $(p\rightarrow v), p\in P_v $ are 
equivalent, which implies that $v$ is redundant node w.r.t. $G^M_0$ by 
Proposition \ref{propos:renodes}.
\end{proof}

By edge equivalence, we construct a marginal polytope diagram 
$G^M_{\pi}=(\Vcal^M_{\pi},\Ecal^M_{\pi})$ below,
\begin{align}
  \Vcal^M_{\pi} &= \Ccal_{\pi},\\
  \Ecal^M_{\pi} &= \{(c\rightarrow s)|c,s\in\Ccal_{\pi},s\subset 
c,\nexists 
t\in\Ccal_{\pi},\mathrm{s.t.}~s\subset t\subset c\}.\notag
\end{align}
Thus we define $\Scal_{\pi}(c)=\{s|(c\rightarrow s) 
\in\Ecal^M_{\pi}\}$.
Let $\Ccal'$ be $\Ccal_{\pi}$, and $\Scal(\Ccal')$ be 
$\Scal_{\pi}(\Ccal_{\pi})=(\Scal_{\pi}(c))_{c\in\Ccal_{\pi}}$ the corresponding 
marginal polytope becomes $ \Mcal_L(G,\Ccal_{\pi},\Scal_{\pi}(\Ccal_{\pi}))$. 
By the fact that $c\notin \Scal_\pi(c), \forall c\in\Ccal_\pi$, applying belief propagation without messages on 
$\Mcal_L(G,\Ccal_{\pi},\Scal_{\pi}(\Ccal_{\pi}))$ results in the following 
belief propagation form $\forall c\in\Ccal_{\pi}$:
 \begin{align}
b_s^{*}(\xb_s)\hspace{-0.03in}&=\hspace{-0.03in}\frac{1}{|\Scal_{\pi}(c)|}\max_{
\xb_{c\setminus 
s}}[b_c(\xb_c)\hspace{-0.02in}+\hspace{-0.07in}\sum_{\hat{s}\in\Scal_{\pi}(c)}b_
{\hat{s}}(\xb_{ \hat{s}})],
\forall  s\in\Scal_\pi(c),\xb_s\notag\\
b_c^*(\xb_c)\hspace{-0.03in}&=\hspace{-0.03in}b_c(\xb_c)+\hspace{-0.03in}\sum_{
\hat{s}\in\Scal_{ \pi}(c)}b_{\hat{ s}}(\xb_{\hat{s}
})-\hspace{-0.03in}\sum_{\hat{s}\in\Scal_{\pi}(c)}b_{\hat{s}}^*(\xb_{\hat{s}}),
\forall \xb_c.
 \label{eqn:MSGUPDTRADE}
\end{align}

Note that a node in the $\pi$-system may still be a redundant node.

\subsection{Maximal-Cluster Intersection algorithm}
Here we remove more redundant nodes by 
introducing the notion of maximal clusters.
\begin{definition} [maximal cluster] Given a graph $G = (\Vcal, \Ccal)$, a 
cluster $c$ is said to be a \emph{maximal cluster} of $\Ccal$, if 
$c\in\Ccal,\nexists \hat{c}\in\Ccal, ~\mathrm{s.t.}~ c\subset \hat{c}$.
\end{definition}
The intersection of all maximal clusters is 
\begin{align}
  \Ical_m=\{s|s=c\cap c',~c,c'\in\Ccal_m\},
\end{align}
where \begin{align}
  \Ccal_m=\{c|c\in\Ccal,\nexists \hat{c}\in\Ccal, ~\mathrm{s.t.}~ c\subset 
\hat{c}\}\label{eq:Ccalm}.
\end{align}

\begin{proposition}\label{propos:renodes_mi}
All $v\in\Vcal^M_0\setminus \{\Ccal \cup \Ical_m\}$ are redundant 
nodes w.r.t $G^M_0$. 
\end{proposition} 
The proof is provided in Section  
\ref{sec:proof_propos_mi} of the supplementary material.

By edge equivalence, we construct a marginal polytope diagram 
$G^M_m=(\Vcal^M_m,\Ecal^M_m)$ with
\begin{align}
  \Vcal^M_{m} &= \Ccal\cup\Ical_{m}\notag\\
  \Ecal^M_{m} &= \{(c\rightarrow 
s)|c\in\Ccal_{m},s\in\Ccal\cup\Ical_m,s\subset c\}.
\end{align}
Thus we define $\Scal_m(c)=\{s|(c\rightarrow s)\in \Ecal^M_m\}$. Let 
$\Ccal'$ be $\Ccal_m$, and $\Scal(\Ccal)$ be 
$\Scal_m(\Ccal_m)=(\Scal_m(c))_{c\in\Ccal_m}$, the corresponding local marginal 
polytope becomes $\Mcal_L(G,\Ccal_m,\Scal_m(\Ccal_m))$. By the fact that $c\notin \Scal_m(c), \forall c\in\Ccal_m$, applying 
belief propagation without messages to $\Mcal_L(G,\Ccal_m,\Scal_m(\Ccal_m))$ 
yields the following belief propagation $\forall c\in\Ccal_M$:
\begin{align}
b_s^{*}(\xb_s)\hspace{-0.03in}&=\hspace{-0.03in}\frac{1}{|\Scal_m(c)|}\hspace{
-0.03in}\max_{\xb_ {c\setminus 
s}}[b_c(\xb_c)\hspace{-0.03in}+\hspace{-0.09in}\sum_{\hat{s}\in\Scal_m(c)}
\hspace{-0.03in}b_{ \hat{s}}(\xb_s)], \forall
s\in\Scal_m(c),\xb_s\notag\\
b_c^*(\xb_c)\hspace{-0.03in}&=\hspace{-0.03in}b_c(\xb_c)\hspace{-0.03in}+\hspace
{-0.03in}\sum_{ \hat{s} \in\Scal_m(c)}b_{\hat{s}} 
(\xb_s)\hspace{-0.03in}-\hspace{-0.03in}\sum_{
\hat{s}\in\Scal_m(c)}b_{\hat{s}}^*(\xb_{\hat{s}}),\forall \xb_c.
\end{align}
\begin{table}[!t]
  \centering
  \caption{{Comparison of different algorithms.
}}
  \label{tab:CMPMP}
{
  \begin{tabular}{p{3cm}p{2.5cm}}
    \hline
    Methods  & Cluster Pursuit \\
    \hline
    Sontag12 & ``Triplet''+``Cycle''\cite{SontagChoeLi_uai12} \\
     GMPLP+T  &``Triplet''\cite{SontagEtAl_uai08}\\
    \tabincell{l}{GMPLP+S}  &``Stealth''(Ours)\\
    PS  & ``Stealth'' (Ours) \\
    $\pi$-S  &``Stealth'' (Ours)\\
    MI  &``Stealth'' (Ours)\\
    \hline
  \end{tabular}}
\end{table}

\section{Experiments}
MAP LP relaxation can be solved using standard LP solvers such as
CPLEX, Gurobi, LPSOLVE \etc. However, for the inference 
problems in our experiments the LP relaxations typically have
more than $10^5$ variables and $10^6$ constraints. It is very slow to use 
standard LP solvers in this case. Even state-of-the-art commercial LP solvers such as CPLEX 
have been reported to be slower than message passing based algorithms 
\cite{yanover2006linear}. Thus we only
compare our methods against message passing based algorithms.

We compare the proposed algorithms (all with ``stealth'' cluster pursuit) which are 
Power Set algorithm (PS), $\pi$-System algorithm ($\pi$-S) and Maximal-Cluster 
Intersection algorithm (MI), with 3 competitors: GMPLP \cite{globerson2007fixing} with 
``triplet'' cluster pursuit \cite{SontagEtAl_uai08} (GMPLP+T), GMPLP with our 
``stealth'' cluster pursuit (GMPLP+S), and Dual Decomposition with ``triplet'' and ``cycle'' cluster pursuit 
\cite{SontagChoeLi_uai12} (Sontag12). All algorithms run belief
propagation/message passing with all original constraints (including the ones with higher order
potentials).  After several iterations of belief
propagation, if there is a gap between the dual and decoded
primal, different cluster pursuit strategy are applied to tighten the
LP relaxations. A brief summary of these methods is provided in Table 
\ref{tab:CMPMP}.  Max-Sum Diffusion (MSD) \cite{Werner2008,
  Werner2010} has been shown empirically inferior to GMPLP \cite[see]
[Figure  1.5]{SonGloJaa_optbook}. Similarly TRW-S of
\cite{Kolmogorov2012} has been shown to be inferior to GMPLP in the
higher order potential case  \citep[see][Section 5]{Kolmogorov2012}. Thus we
compare primarily with Sontag12 and GMPLP. Note that Sontag12 is
considered the state-of-the-art.


We implement our algorithms and GMPLP in C++. For Sontag12, we use 
their released C++ code\footnote{\scriptsize For computational efficiency, we optimised Sontag \etal's released code (achieving the same 
output but with 2-3 times speed up). This is done for GDD and GMPLP as well to ensure a fair comparison. All algorithms are compiled with option 
``-O3 -fomit-frame-pointer -pipe -ffast-math'' using ``clang-4.2'', and all 
experiments are running in single thread with I7 3610QM and 16GB RAM.}. As all algorithms use a framework similar to 
Algorithm \ref{algo:CP} (with different the message updating and cluster 
pursuit), we can describe the termination criteria for these algorithms using the notation from 
Algorithm \ref{algo:CP}. For all algorithms, the threshold of inner loop $T_g$ is set to $10^{-8}$, and the maximum number of
iterations $K^{1}_{max}=1000$. We adopt $K^{2}_{max}=20$ in light of the faster convergence observed empirically in \cite{SontagEtAl_uai08}. The threshold for the 
outer loop $T_a$ is set to $10^{-6}$. In each cluster pursuit we add $n=20$ 
new clusters, and the maximum running time 
$\mathscr{T}_{max}$ is set to 1 hour.
\begin{figure}[t]
\centering
\subfigure[\scriptsize graph 
structure]{\includegraphics[width=4cm]{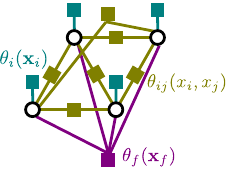}\label{fig:graphstruct}
}
\quad
\subfigure[\scriptsize primal and 
dual]{\label{fig:randomresult}
\includegraphics[width=4cm]{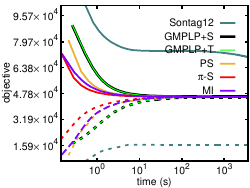}}

\caption{{Experiment on synthetic data. (a) Graph structure; (b) 
Dual 
objectives (solid line) and the values of decoded integer solution in the 
primal 
 (dashed line).}}
\label{Fig:ResultRandom}
\end{figure}

GDD based algorithms can be implemented as either a message passing procedure or a belief 
propagation procedure without messages. We implemented both, and observed that both have similar speed (see Section \ref{cmp_bp_mp} in 
the supplementary material). Of course, the latter uses less storage. For presentation clarity, we only report the result of GDD using belief 
propagation without messages here.

\subsection{Synthetic data}
We generate a synthetic graphical model with a structure commonly used
in image segmentation and denoising. The structure is a $128\times
128$ grid shown in Figure \ref{fig:graphstruct} with 3 types of
potentials: node potentials, edge potentials and higher order potentials. 
We consider the problem below,
\begin{align*}
\max_{\xb}\big[\sum_{i\in \Vcal}\theta_{i}(x_i)+\sum_{ij\in 
\Ecal}\theta_{ij}(x_i,x_j)+\sum_{f\in\mathcal{F}}\theta_{f}(\xb_f)
\big],
\end{align*}
where each $x_i \in \{1,2,3\}$ and $|f| = 4$.  All potentials are
generated from normal distribution $N(0,1)$. For clusters with order $\ge 4$, Sontag12 only 
enforces local marginalisation constraints from clusters to nodes, thus its 
initial local marginal polytope is looser than that of  GMPLP and our methods. 
As shown in Figure \ref{fig:randomresult}, the proposed methods,
converge much faster than GMPLP+S, GMPLP+T and Sontag12. 
Also Sontag12's gap between the dual objective and the decoded primal objective is 
much larger than that of our methods and GMPLP (even with cluster pursuit to tighten the local marginal 
polytope). 

\subsection{Protein-Protein Interaction}

\begin{figure}[!t]
  \centering  
  \addtolength{\subfigcapskip}{-0.06in}
\subfigure[\scriptsize\textit{protein2}]{
\includegraphics[width=0.4\textwidth]{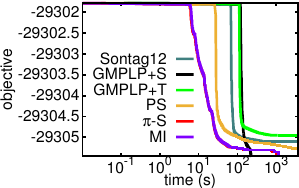}
\label{subfig:ppi2}}
 \quad
\subfigure[\scriptsize\textit{protein4}]{
\includegraphics[width=0.4\textwidth]{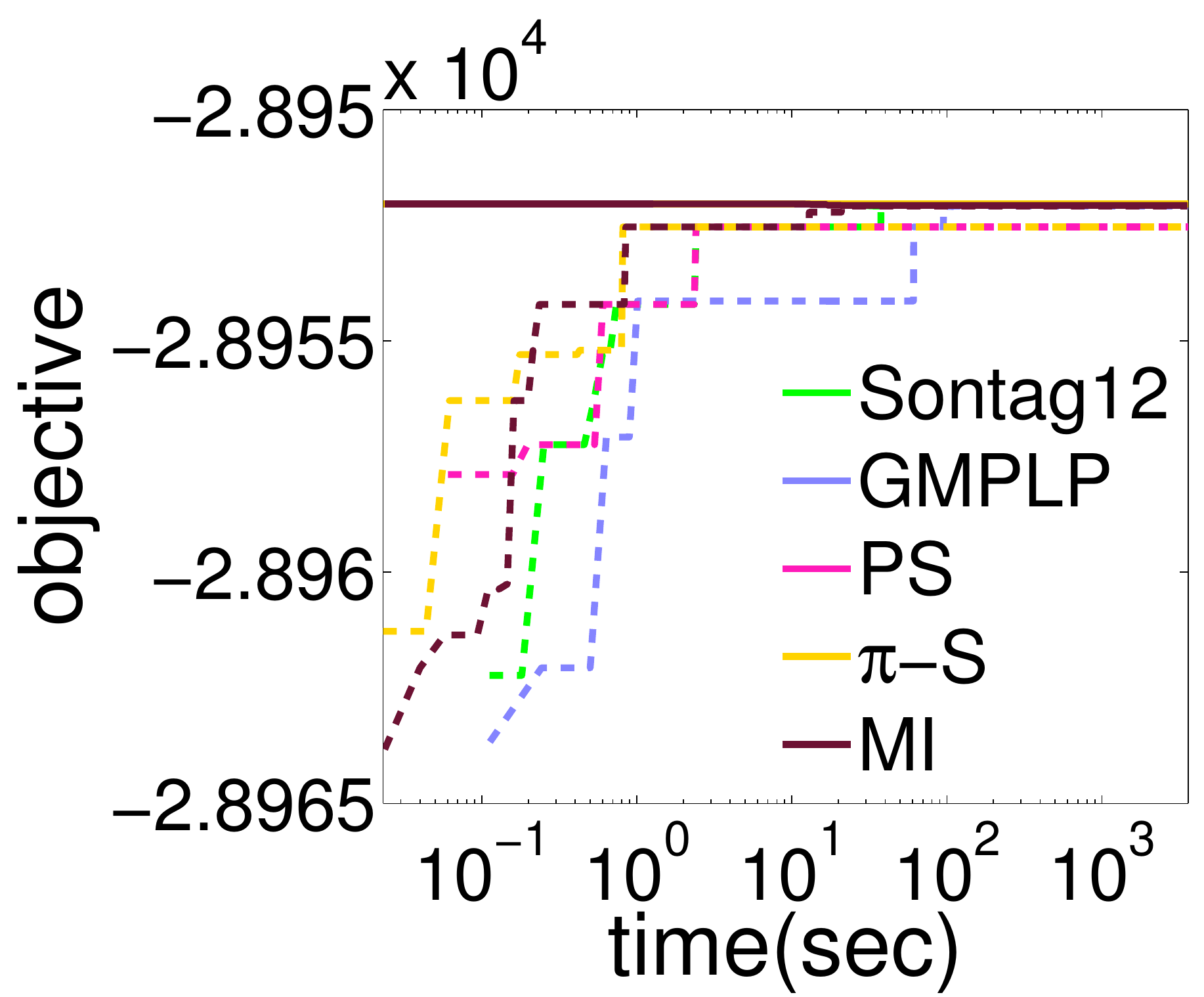}
\label{subfig:ppi5}}

  \caption{{Dual objectives on 2 PPI problems. (a) On
\textit{protein2}, all algorithms get exact solution except Sontag12, GMPLP+T and PS; (b) On \textit{protein4}, all 
algorithms find the exact solution except GMPLP+\{T,S\}.}}
  \label{fig:resppi}
\end{figure}

\begin{table}[!t]
\caption{{Average running time ($\pm$ standard deviation) for  one iteration of updating all beliefs or messages on PPI dataset}}

  \label{tab:RTCMP}
  \centering
{
  \begin{tabular}{p{2cm}<{\centering}p{4.5cm}<{\centering}}
    \hline
    Methods  &  Average Running Time (sec) \\
    \hline
    Sontag12 & 0.0765$\pm$0.0113\\
GMPLP+S & 0.1189$\pm$0.0093\\
GMPLP+T & 0.1118$\pm$0.0020\\
PS & 0.0359$\pm$0.0052\\
$\pi$-S & 0.0120$\pm$0.0030\\
MI & {\sffamily \fontseries{bx}\selectfont 0.0112$\pm$0.0020}\\
    \hline
  \end{tabular}}
\end{table}

Here we consider 8 Protein-Protein Interaction (PPI) inference problems (from \textit{protein1} to \textit{protein8}) used in 
\cite{SontagChoeLi_uai12}. In each problem, there are 
typically over 14000 nodes, and more than 42000 
potentials defined on nodes, edges and triplets. Since the highest order of the potentials is only 3 (triplets), the local marginal polytopes (without cluster 
pursuit) of all methods here are the same tight. Thus performance difference 
here is mainly due to different cluster pursuit strategies and computational 
complexity per iteration. 

We test all methods on all 8 problems. The average running time for one iteration of updating all beliefs  or messages in Table \ref{tab:RTCMP} (\ie  steps 4-7 in Algorithm \ref{algo:MSGUPD} for ours, and the counterpart for the competitors similar to steps 4-8 in Algorithm \ref{algo:MSGUPGDD}). We can see that the proposed methods have the smallest average running time, followed by Sontag12, and then by GMPLP.

We present dual objective plots on two problems in Figure \ref{fig:resppi} here, and provide the results for all problems in the supplementary (Section 8.2). Overall, the proposed methods converge fastest and two of them ($\pi$-S and MI) find exact solutions on 3 problems: \textit{protein2}, \textit{protein4} and \textit{protein8}. Sontag12 finds exact solutions on \textit{protein4} and \textit{protein8}, and achieved the smallest dual objective values on the problems where all methods failed to find the exact solutions. In terms of convergence, Sontag12 converges slower than proposed methods and faster than GMPLP.  GMPLP+T does not find an exact solution on any of the 8 problems, and GMPLP+S finds the exact solution on \textit{protein2} only.


%

\subsection{Image Segmentation}

\begin{figure}[!t]
\centering

\addtolength{\subfigcapskip}{-0.03in}

\includegraphics[height=2.7cm,width=3.6cm]{banana1}
~
\includegraphics[height=2.7cm,width=3.6cm]{bool}
~
\includegraphics[height=2.7cm,width=3.6cm]{book}

\vspace{0.25mm}

\includegraphics[height=2.7cm,width=3.6cm]{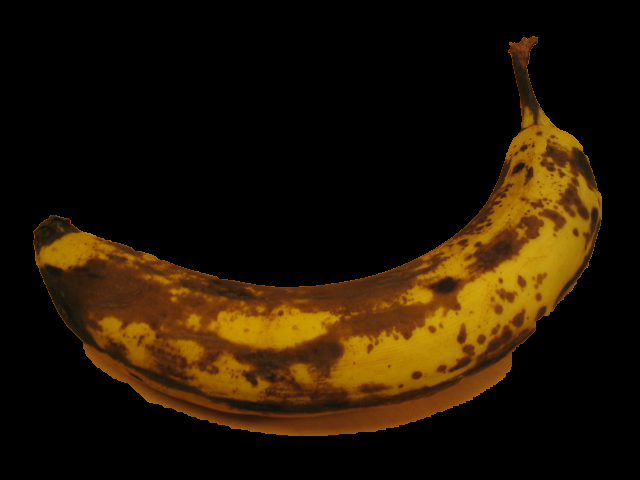}
~
\includegraphics[height=2.7cm,width=3.6cm]{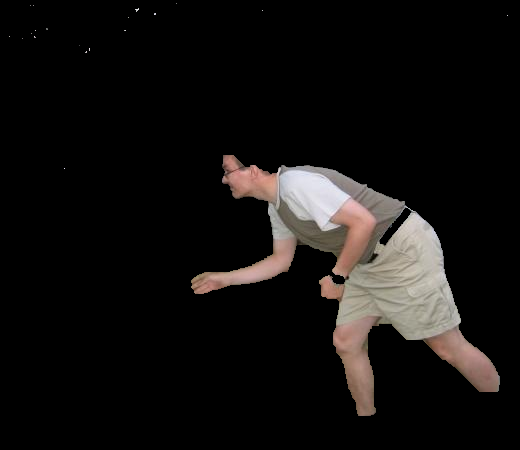}
~
\includegraphics[height=2.7cm,width=3.6cm]{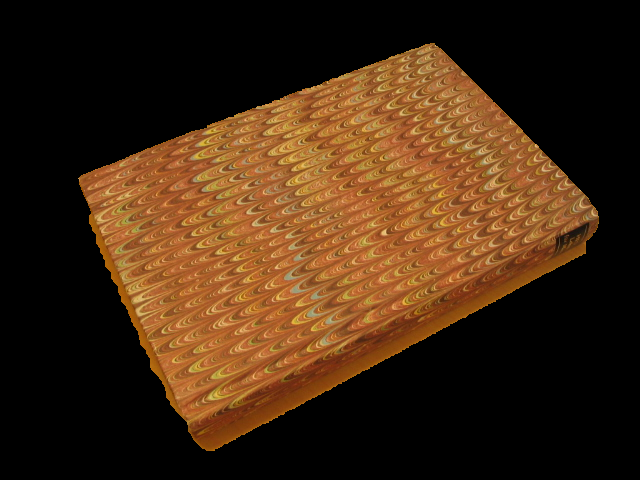}

\vspace{0.5mm}

\includegraphics[width=3.6cm]{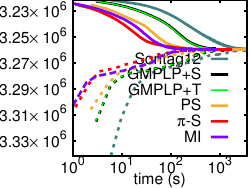} 
\includegraphics[width=3.6cm]{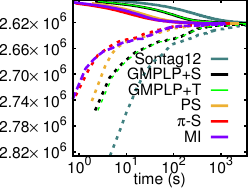} 
\includegraphics[width=3.6cm]{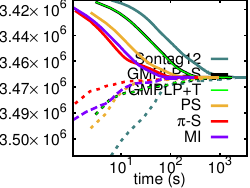} 

\vspace{0.25mm}

\includegraphics[width=3.6cm]
{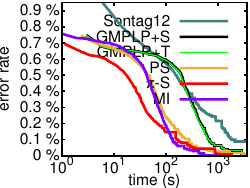} 
\includegraphics[width=3.6cm]
{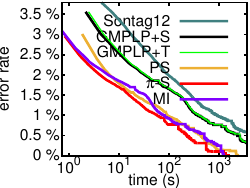} 
\includegraphics[width=3.6cm]
{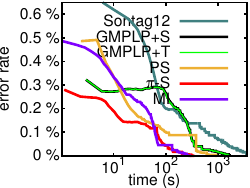} 

\footnotesize{\textit{~~~~banana1~~~~~~~~~~~~~~~~~~~~~~~~~~~bool~~~~~~~~~~~~~~~~~~~~~~~~~~~~~book}}

%
%

\caption{{Image segmentation results. {\bf First row}: original 
images; {\bf Second row}:  segmentations corresponding exact solutions of the MAP inference problems; {\bf Third row}: decoded primal objectives (dashed line) and dual objectives (solid line); {\bf Fourth row}: error rate plots.} Each column corresponds to the results of one image. }
\label{fig:Seg}
\end{figure}



Image segmentation is often seen as a MAP inference problem over PGMs. 
Following \cite{kohli2009robust}, we consider the MAP 
problem below 
\[
\max_{\xb}\big[\sum_{i\in \Vcal}\theta_{i}(x_i)+\sum_{ij\in 
\Ecal}\theta_{ij}(x_i,x_j)+\sum_{f\in\mathcal{F}}\theta_{f}(\xb_f)
\big].
\]
where $|f|=4$. We use the same graph structure as in Figure 
\ref{fig:graphstruct}, and follow the potentials in \cite{kohli2009robust}, where the colour terms in $\theta_i(x_i)$ are computed as in \cite{blake2004interactive}. More details including 
parameter settings are provided in the supplementary material. 
Here we segment three images: 
\textit{banana1}, \textit{book} and \textit{bool} in the  MSRC Grabcut dataset
\footnote{\scriptsize
\url{
http://research.microsoft.com/en-us/um/cambridge/projects/visionimagevideoediting/segmentation/grabcut.htm}}. The resolution of the images varies from 
$520\times 450$ to $640\times 480$, 
and each of the inference problems has about $2\times 10^5$ to $3\times 
10^5$ nodes and more than $10^6$ potentials. 
Sontag12 
failed to find exact solutions
in all 3 images. GMPLP+S and GMPLP+T find exact solution 
on \textit{book} only. The proposed methods, PS, $\pi$-S and MI, find 
the exact solution on all three problems.
The result is shown in Figure \ref{fig:Seg}. From the third row of Figure \ref{fig:Seg} (primal-dual objetives), we can see that the proposed method converges much faster than the competitors. From the fourth row of Figure \ref{fig:Seg}, we can see that the inference error rate (against the exact solution) reduced quickest to zero for the proposed methods.

\subsection{Image Matching}
Here we consider key point based image matching between two images (a source image and a destination image). 
First we detect key points from both images via SIFT \cite{lowe1999object} detector. 
Assume that there are $m$ key points from the source image and $n$ key points from the destination image. Let $\{p(i)\in\mathbb{R}^2\}_{i=1,2, \ldots, m}$
and $\{q(i)\in\mathbb{R}^2\}_{i=1,2, \ldots, n}$ be the coordinates of key points in the source and the destination images respectively. Let $\{h(i)\}_{i=1,2,\ldots, m}$ and $\{g(i)\}_{i=1,2,\ldots, n}$ be the SIFT feature vectors of the source and the destination images respectively. Assume $m \leq n$ (otherwise swap the source image and the destination image to guarantee so). The task is for each key point $i \in \{1,2,\ldots, m \}$ in the source image, to find a corresponding key point $x_i \in  \{1,2,\ldots, n \}$ in the destination image. When there are a large number of points involved, a practical way is to restrict a corresponding key point $x_i \in \{1, \ldots, k\}\cup\{-1\}$. Here if $i$ has a corresponding point, we restrict it from its $k$-nearest neighbours of the SIFT feature vector $h(i)$ in $\{q(i)\in\mathbb{R}^2\}_{i=1,2, \ldots, n}$. If $i$ has no corresponding point, we let $x_i = -1$.  

Let $\Vcal=\{1,2,\ldots m\}$,  and the matching problem can be formulated as 
a MAP problem similar to 
\cite{li2010object}, 
\begin{align}
\label{eq:map-match}
\max_{\xb} \Big\{\sum_{i\in 
\Vcal}\theta_i(x_i)+\sum_{f\in\mathcal{F}}\theta_f(\xb_f)\Big\},
\end{align}

\noindent where $|f|=4$, $x_i \in \{1, \ldots, k\}\cup\{-1\}$, and constructing of $\Fcal$ is provided in the supplementary.
In \cite{li2010object}, key points in the source image are filtered and reduced to less than 100 (see Section 3 of \cite{li2010object}), thus they often have corresponding key points in the destination image. As a result they did not use $-1$ to handle the no correspondence case. However, this strategy gives rise to a danger that potentially important key points may be removed too. 
 Also the small scale of their problem allows them to let $x_i$ take all $n$ states. In our experiment, we keep all key points (often over $10^3$ in both source and destination images). In that situation, we face two issues: 1) each $f\in\Fcal$ needs $\Ocal((k+1)^4)$ storage for potentials and beliefs; 2)  some key points in the source image have no corresponding points in the destination image. For the first issue, we set $k=4$ for computational efficiency. For 
the second issue, we extend the models in \cite{li2010object} to handle the potential lack of correspondence.  The 
node and higher order potentials are defined as follows:
\begin{align*}
  \theta_i(x_i)&=\left\{ \begin{array}{cl}
      -\eta & x_i=-1\\
-\delta_i||h(i)-g(x_i)||_2^2 & \mathrm{otherwise}
    \end{array}
\right.\\
\theta_f(\xb_f)&=\left\{\begin{array}{cl}
    0 & \exists i \in f, s.t. x_i = -1\\
    -||P_{\xb_f}W_f||_1 & \mathrm{otherwise}
\end{array}
\right.
\end{align*}

\begin{figure*}[t]
\renewcommand{\thesubfigure}{}
\addtolength{\subfigcapskip}{-0.03in}
\centering

\subfigure[{\scriptsize  \textit{bikes}}]{\includegraphics[height=1.3cm,width=3cm]{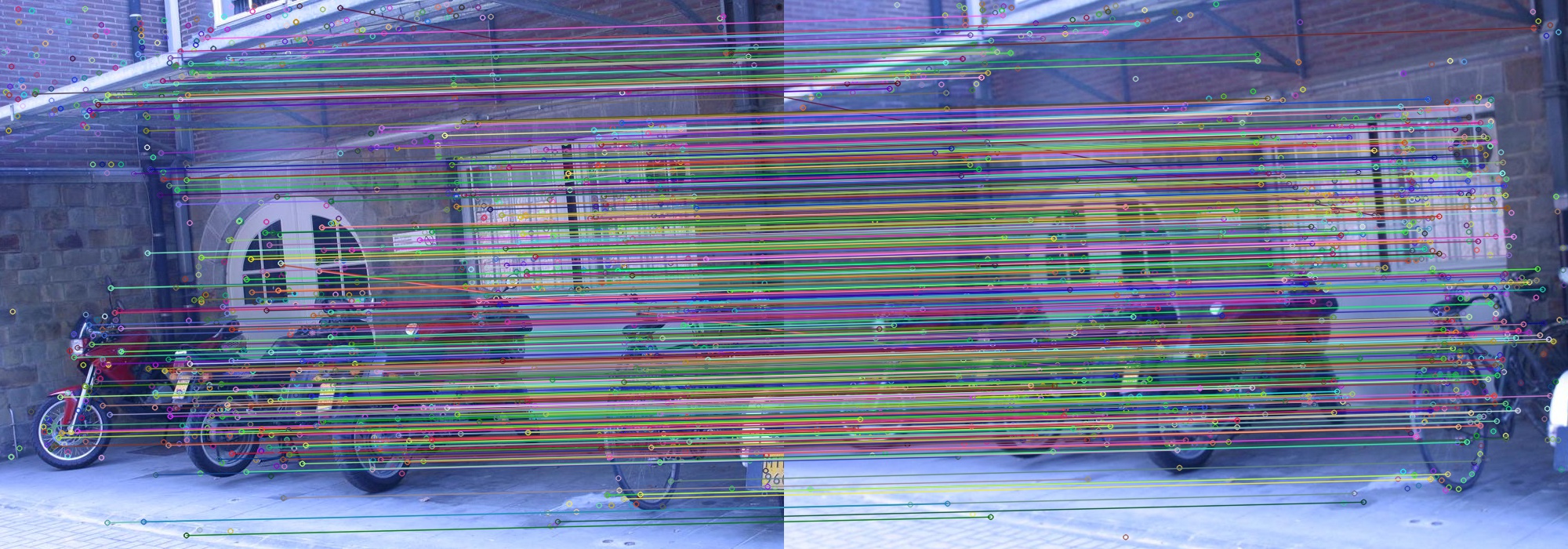}}
\quad\quad
\subfigure[{\scriptsize  \textit{wall}}]{\includegraphics[height=1.3cm,width=3cm]{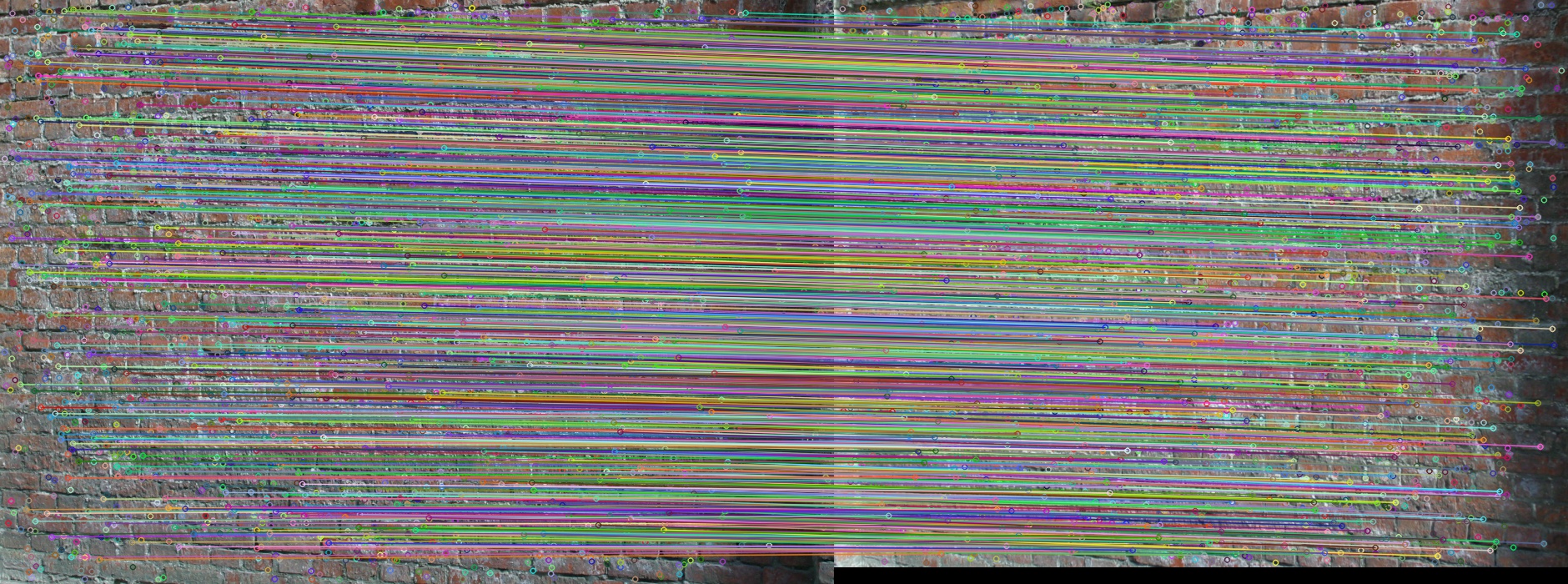}}
\quad\quad
\subfigure[{\scriptsize  \textit{trees}}]{\includegraphics[height=1.3cm,width=3cm]{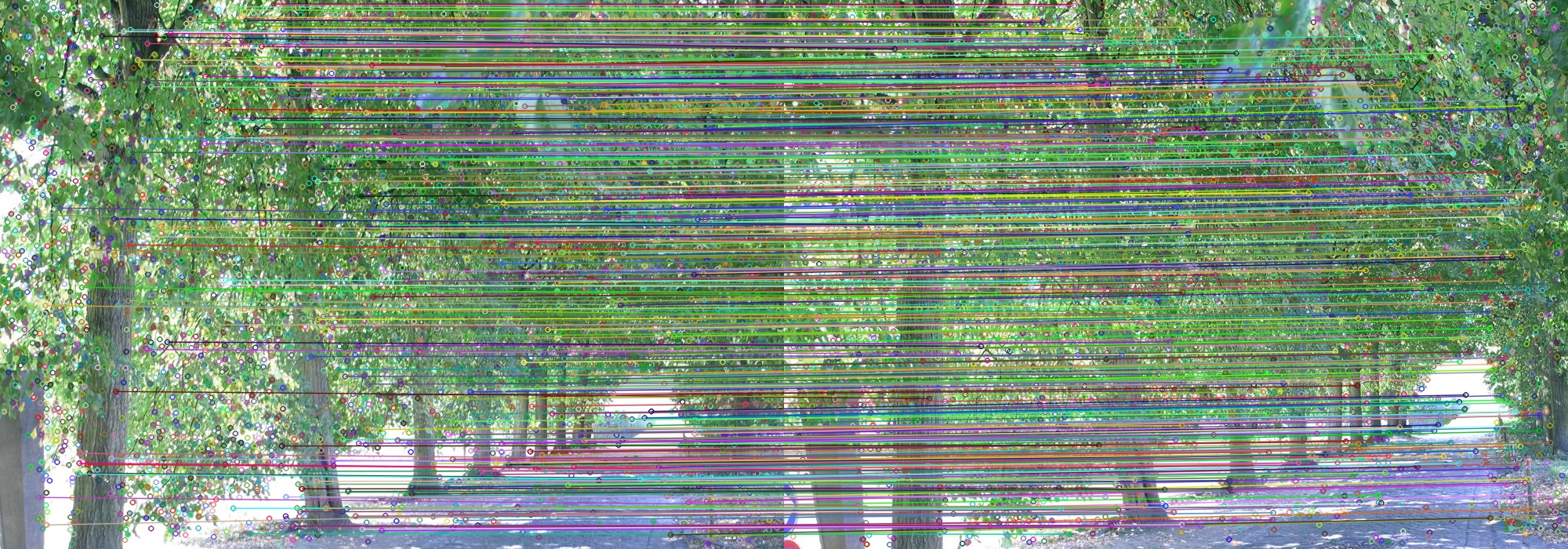}}

\vspace{-0.1in}

\subfigure[\scriptsize  \textit{graf}]{\includegraphics[height=1.3cm,width=3cm]{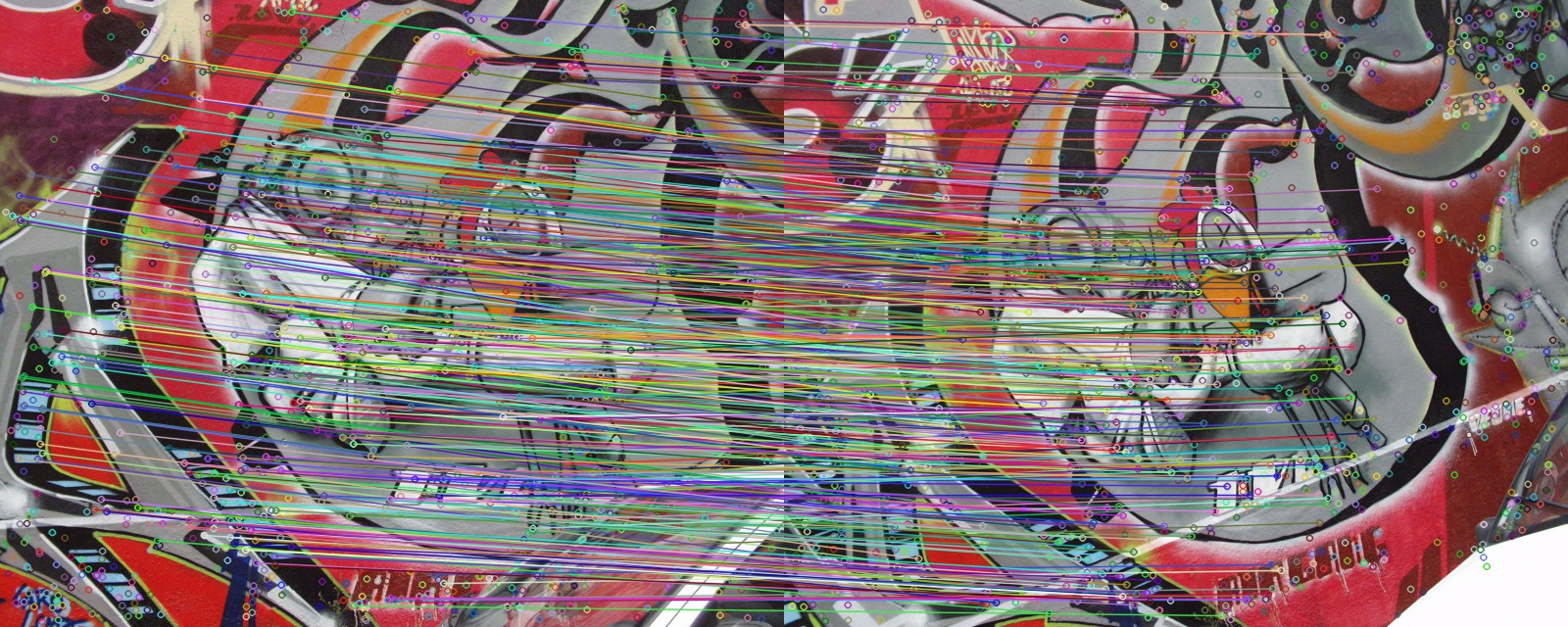}}
\quad\quad
\subfigure[\scriptsize  \textit{bark}]{\includegraphics[height=1.3cm,width=3cm]{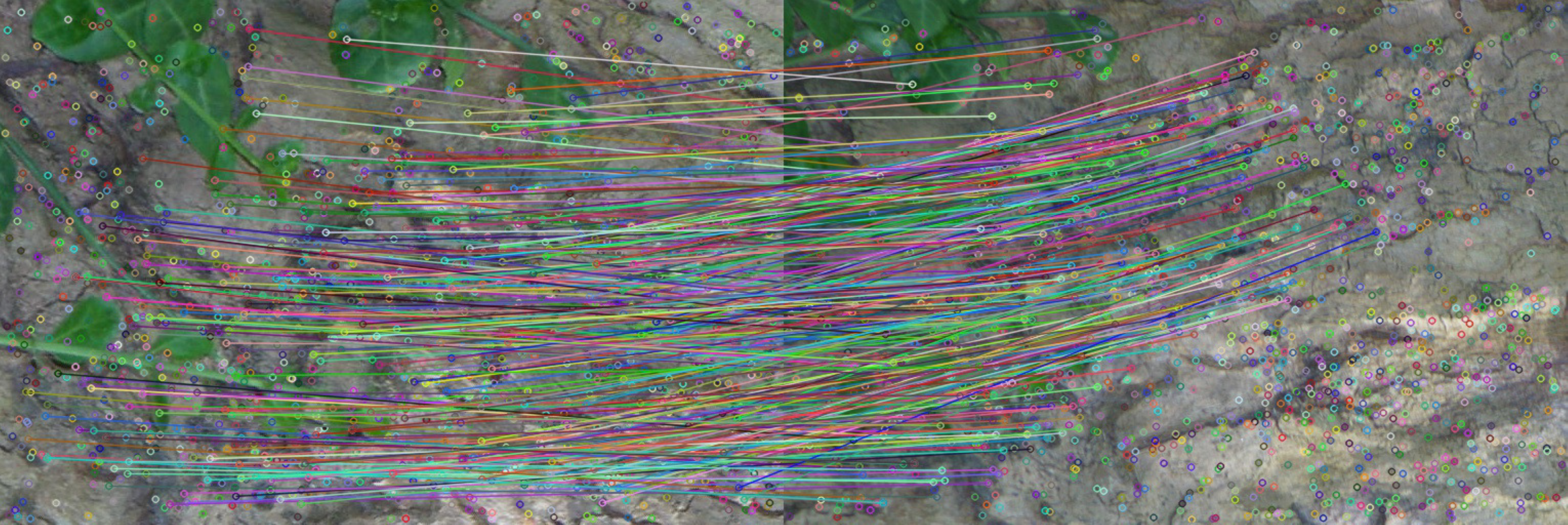}}
\quad\quad
\subfigure[\scriptsize  \textit{ubc}]{\includegraphics[height=1.3cm,width=3cm]{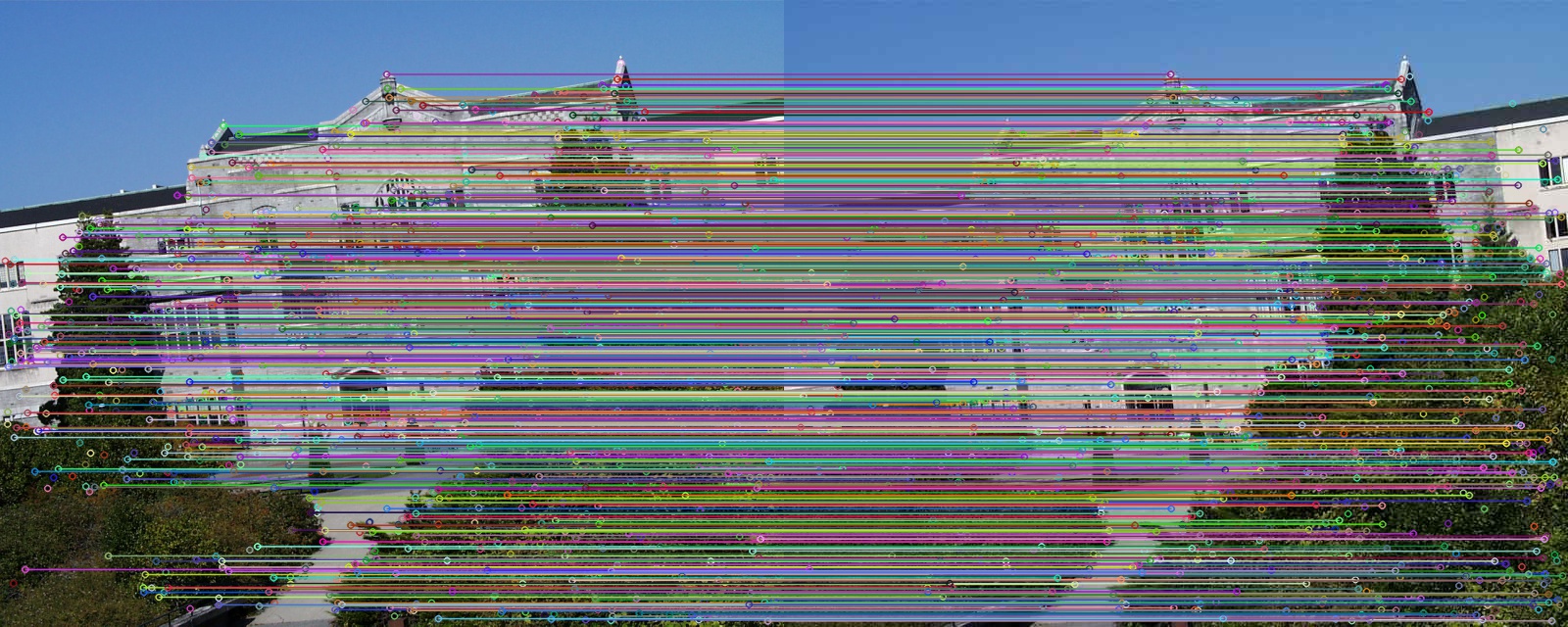}}

\vspace{-0.1in}

\subfigure[\scriptsize 
\textit{bikes12(1219)}]{\includegraphics[height=25mm]{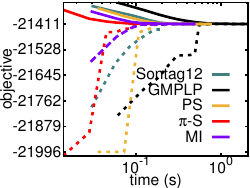}}
\subfigure[\scriptsize 
\textit{wall12(2184)}]{\includegraphics[height=25mm]{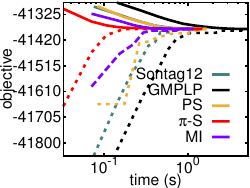}}
\subfigure[\scriptsize 
\textit{trees12(3284)}]{\includegraphics[height=25mm]{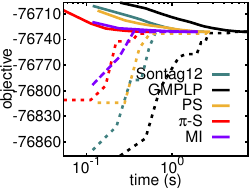}}

\vspace{-0.1in}

\subfigure[\scriptsize 
\textit{graf12(1203)}]{\includegraphics[height=25mm]{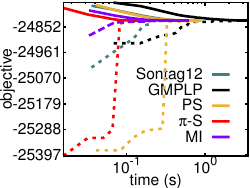}}
\subfigure[\scriptsize 
\textit{bark12(1486)}]{\includegraphics[height=25mm]{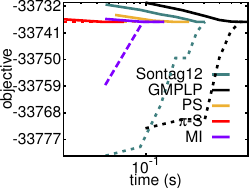}}
\subfigure[\scriptsize 
\textit{ubc12(1182)}]{\includegraphics[height=25mm]{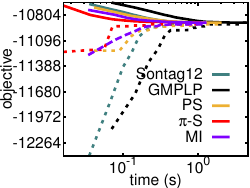}}

\vspace{-0.1in}

\subfigure[\scriptsize 
\textit{bikes12(1219)}]{\includegraphics[height=25mm]{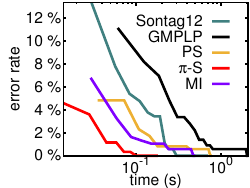}}
\subfigure[\scriptsize 
\textit{wall12(2184)}]{\includegraphics[height=25mm]{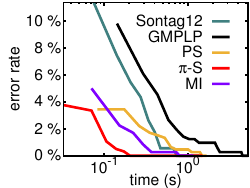}}
\subfigure[\scriptsize 
\textit{trees12(3284)}]{\includegraphics[height=25mm]{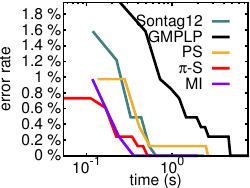}}

\vspace{-0.1in}

\subfigure[\scriptsize 
\textit{graf12(1203)}]{\includegraphics[height=25mm]{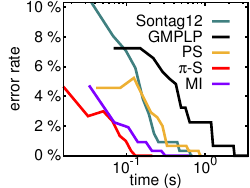}}
\subfigure[\scriptsize 
\textit{bark12(1486)}]{\includegraphics[height=25mm]{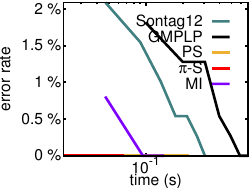}}
\subfigure[\scriptsize 
\textit{ubc12(1182)}]{\includegraphics[height=25mm]{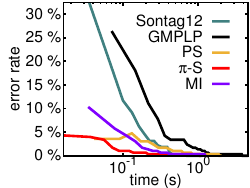}}

\vspace{-0.1in}

\caption{{Image matching result. {\bf Top 2 rows}: images and 
matching 
results; {\bf Middle 2 rows}: decoded primal (dashed line) and dual 
(solid line); {\bf Bottom 2 rows}: error rate plots. $[d]ij(n)$ means 
matching between $i^{th}$ image  and $j^{th}$ image in dataset $[d]$, with a 
PGM 
of $n$ nodes.}}
\label{fig:match}
\vspace{-0.1in}
\end{figure*}

\noindent where $\eta$ and $\delta_i$ are user specified parameters, 
$P_{\xb_f}=[q(x_i)]_{i\in f}\in\mathbb{R}^{2\times 4}$, and $W_f\in\mathbb{R}^4$ is a column 
vector computed via solving (5) in \cite{li2010object} (details provided in supplementary).  

We set $\eta=-25$, and 
$\delta_i=100/\max_{x_i}||h(i)-g(x_i)||_2^2,\forall i\in\Vcal$. We test all algorithms on 6 image sequences from \textit{Affine Covariant 
Regions Datasets} \footnote{\scriptsize\url{http://www.robots.ox.ac.uk/~vgg/data/data-aff.html}}.
Each inference problem has about $1\times10^3$ to $3\times 10^3$ nodes and 
$2\times 10^3$ to $6\times 10^3$ potentials. 
All algorithms find exact solutions. GMPLP converges before using
cluster pursuit, thus GMPLP+S and GMPLP+T became the same (reported as GMPLP). From Figure \ref{fig:match} we can see that our $\pi$-S converges 
fastest among all methods in all images, followed by our MI. 
\begin{table*}
  \caption{Running time comparison. Total running time and the averaged ranking (Avg. Rank) of the speed are reported. Best results are in boldface (the smaller the better).}
\label{tab:RTCmp}

\centering

\begin{tabular}{p{1.5cm}<{\centering}p{9mm}<{\centering}p{9mm}<{\centering}
p{9mm}<{\centering}p{9mm}<{\centering}p{9mm}<{\centering}p{9mm}<{
\centering}
p{9mm}<{\centering}p{9mm}<{\centering}}
\hline\noalign{\smallskip}
&\footnotesize{bark12} &\footnotesize{bark23} &\footnotesize{bark34}
&\footnotesize{bark45} &\footnotesize{bark56} &\footnotesize{bikes12}
&\footnotesize{bikes23} &\footnotesize{bikes34} \\
\noalign{\smallskip}\hline\noalign{\smallskip}
Sontag12 &0.30&1.00&0.38&1.29&0.66&0.70&0.65&0.61\\
GMPLP &0.68&2.07&2.10&5.35&2.77&2.10&2.44&1.25\\
PS(Ours) &0.23&1.30&0.33&1.37&0.77&0.79&1.00&0.41\\
$\pi$-S(Ours) &{\sffamily \fontseries{bx}\selectfont 0.07}&{\sffamily
  \fontseries{bx}\selectfont 0.60}&{\sffamily
  \fontseries{bx}\selectfont 0.13}&{\sffamily
  \fontseries{bx}\selectfont 0.51}&{\sffamily
  \fontseries{bx}\selectfont 0.24}&{\sffamily
  \fontseries{bx}\selectfont 0.18}&1.14&{\sffamily
  \fontseries{bx}\selectfont 0.10}\\
MI(Ours) &0.14&0.78&0.23&0.92&0.44&0.50&{\sffamily \fontseries{bx}\selectfont 0.59}&0.25\\
\noalign{\smallskip}\hline
\end{tabular}

\vspace{1mm}

\begin{tabular}{p{1.5cm}<{\centering}p{9mm}<{\centering}p{9mm}<{\centering}
p{9mm}<{\centering}p{9mm}<{\centering}p{9mm}<{\centering}p{9mm}<{
\centering}
p{9mm}<{\centering}p{9mm}<{\centering}}
\hline\noalign{\smallskip}
&\footnotesize{bikes45} &\footnotesize{bikes56} &\footnotesize{graf12}
&\footnotesize{graf23} &\footnotesize{graf34} &\footnotesize{graf45}
&\footnotesize{graf56} & \footnotesize{trees12}\\
\noalign{\smallskip}\hline\noalign{\smallskip}
Sontag12 & 0.30&0.55&0.69&0.54&1.31&2.52&1.20&2.36\\
GMPLP & 2.11&0.76&3.54&3.66&5.51&3.52&3.28&7.71\\
PS(Ours) & 0.59&0.25&0.90&1.27&1.89&1.68&1.19&2.76\\
$\pi$-S(Ours) & {\sffamily \fontseries{bx}\selectfont
  0.26}&0.21&{\sffamily \fontseries{bx}\selectfont 0.21}&{\sffamily
  \fontseries{bx}\selectfont 0.17}&{\sffamily
  \fontseries{bx}\selectfont 0.74}&{\sffamily
  \fontseries{bx}\selectfont 0.33}&{\sffamily
  \fontseries{bx}\selectfont 0.25}&{\sffamily
  \fontseries{bx}\selectfont 0.58}\\ 
MI(Ours) & 0.40&{\sffamily \fontseries{bx}\selectfont
  0.20}&0.62&0.85&1.43&0.91&0.59&1.99\\
\noalign{\smallskip}\hline
\end{tabular}

\vspace{1mm}

\begin{tabular}{p{1.5cm}<{\centering}p{9mm}<{\centering}p{9mm}<{\centering}
p{9mm}<{\centering}p{9mm}<{\centering}p{9mm}<{\centering}p{9mm}<{\centering} 
p{9mm}<{\centering}p{9mm}<{\centering}}
\hline\noalign{\smallskip}
 &\footnotesize{trees23} &\footnotesize{trees34}
 &\footnotesize{trees45} &\footnotesize{trees56} &\footnotesize{ubc12}
 &\footnotesize{ubc23} &\footnotesize{ubc34} &\footnotesize{ubc45}  \\
\noalign{\smallskip}\hline\noalign{\smallskip}
Sontag12 &3.21&0.87&{\sffamily \fontseries{bx}\selectfont
  0.18}&1.13&1.15&0.77&{\sffamily \fontseries{bx}\selectfont
  0.46}&0.33\\
GMPLP &6.95&3.46&0.76&3.88&4.65&4.88&4.33&1.98\\
PS(Ours) &1.58&1.61&1.41&1.43&2.04&1.96&1.87&0.48\\
$\pi$-S(Ours) &1.71&{\sffamily \fontseries{bx}\selectfont
  0.40}&0.41&{\sffamily \fontseries{bx}\selectfont 0.49}&{\sffamily
  \fontseries{bx}\selectfont 1.09}&{\sffamily
  \fontseries{bx}\selectfont 0.59}&0.55&0.97\\ 
MI(Ours) &{\sffamily \fontseries{bx}\selectfont
  1.14}&1.28&0.72&1.59&1.28&1.33&1.29&{\sffamily
  \fontseries{bx}\selectfont 0.30}\\
\noalign{\smallskip}\hline
\end{tabular}

\vspace{1mm}

{
\hspace{2pt}
\begin{tabular}{p{1.5cm}<{\centering}p{9mm}<{\centering}p{9mm}<{\centering}
p{9mm}<{\centering}p{9mm}<{\centering}p{9mm}<{\centering}p{9mm}<{\centering}>{\columncolor{Gray}} 
p{20.75mm}<{\centering}}
\hline
\begin{minipage}{1mm}
\vspace{1mm}
$~~$\\
\vspace{1mm}
\end{minipage}
&\footnotesize{ubc56} &\footnotesize{wall12} &\footnotesize{wall23}
&\footnotesize{wall34} &\footnotesize{wall45} &\footnotesize{wall56} &
{  Avg. Rank}\\
\hline
Sontag12 & 0.05&1.38&{\sffamily \fontseries{bx}\selectfont
  1.39}&1.21&0.97&{\sffamily \fontseries{bx}\selectfont 0.74}&{
  2.600}\\
GMPLP & 0.19&5.32&7.93&5.68&2.52&2.26&{ 4.933}\\
PS(Ours) & 0.11&1.68&3.33&2.05&0.85&1.15&{ 3.567}\\
$\pi$-S(Ours) &{\sffamily \fontseries{bx}\selectfont
  0.04}&1.49&10.75&{\sffamily \fontseries{bx}\selectfont
  1.06}&0.95&0.94&{\sffamily \fontseries{bx}\selectfont  1.700}\\
MI(Ours) & 0.10&{\sffamily \fontseries{bx}\selectfont 1.17}&1.77&1.75&{\sffamily \fontseries{bx}\selectfont 0.52}&0.93&{2.200}\begin{minipage}{0.05mm}$~$\\\vspace{1mm}
\end{minipage}
\\
\hline
\end{tabular}
}
\vspace{2pt}
\end{table*}
\begin{table*}
\centering
  \caption{Iterations comparison. Number of iterations and the averaged ranking (Avg. Rank) are reported. Best results are in boldface (the smaller the better). }
\label{tab:ITCmp}
\centering 

\begin{tabular}{p{1.5cm}<{\centering}p{9mm}<{\centering}p{9mm}<{\centering}
p{9mm}<{\centering}p{9mm}<{\centering}p{9mm}<{\centering}p{9mm}<{
\centering}
p{9mm}<{\centering}p{9mm}<{\centering}}
\hline\noalign{\smallskip}
&\footnotesize{bark12} &\footnotesize{bark23} &\footnotesize{bark34}
&\footnotesize{bark45} &\footnotesize{bark56} &\footnotesize{bikes12}
&\footnotesize{bikes23} &\footnotesize{bikes34} \\
\noalign{\smallskip}\hline\noalign{\smallskip}
Sontag12 &7&24&9&31&18&26&23&23\\
GMPLP &7&22&23&58&34&34&38&21\\
PS(Ours) &4&24&6&25&16&22&27&12\\
$\pi$-S(Ours) &{\sffamily
  \fontseries{bx}\selectfont 3}&28&6&24&13&{\sffamily
  \fontseries{bx}\selectfont 13}&80&{\sffamily
  \fontseries{bx}\selectfont 7}\\
MI(Ours) &{\sffamily \fontseries{bx}\selectfont3}&{\sffamily \fontseries{bx}\selectfont17}&{\sffamily \fontseries{bx}\selectfont5}&{\sffamily \fontseries{bx}\selectfont20}&{\sffamily \fontseries{bx}\selectfont11}&17&{\sffamily \fontseries{bx}\selectfont19}&8\\
\noalign{\smallskip}\hline
\end{tabular}

\vspace{1mm}

\begin{tabular}{p{1.5cm}<{\centering}p{9mm}<{\centering}p{9mm}<{\centering}
p{9mm}<{\centering}p{9mm}<{\centering}p{9mm}<{\centering}p{9mm}<{
\centering}
p{9mm}<{\centering}p{9mm}<{\centering}}
\hline\noalign{\smallskip}
 &\footnotesize{bikes45} &\footnotesize{bikes56} &\footnotesize{graf12} &\footnotesize{graf23} &\footnotesize{graf34} &\footnotesize{graf45} &\footnotesize{graf56} & \footnotesize{trees12}\\
\noalign{\smallskip}\hline\noalign{\smallskip}
Sontag12 & {\sffamily
  \fontseries{bx}\selectfont16}&41&22&15&33&64&28&21\\
GMPLP & 51 & 25 &51&45&62&40&34&32\\
PS(Ours) &24&14&22&27&36&33&21&20\\
$\pi$-S(Ours)&26&30&{\sffamily
  \fontseries{bx}\selectfont13}&{\sffamily
  \fontseries{bx}\selectfont9}&35&{\sffamily
  \fontseries{bx}\selectfont16}&{\sffamily
  \fontseries{bx}\selectfont11}&{\sffamily \fontseries{bx}\selectfont
  10}\\
MI(Ours) 8&19&{\sffamily \fontseries{bx}\selectfont13}&18&21&{\sffamily \fontseries{bx}\selectfont32}&21&12&16\\
\noalign{\smallskip}\hline
\end{tabular}

\vspace{1mm}

\begin{tabular}{p{1.5cm}<{\centering}p{9mm}<{\centering}p{9mm}<{\centering}
p{9mm}<{\centering}p{9mm}<{\centering}p{9mm}<{\centering}p{9mm}<{\centering} 
p{9mm}<{\centering}p{9mm}<{\centering}}
\hline\noalign{\smallskip}
 &\footnotesize{trees23} &\footnotesize{trees34}
 &\footnotesize{trees45} &\footnotesize{trees56} &\footnotesize{ubc12}
 &\footnotesize{ubc23} &\footnotesize{ubc34} &\footnotesize{ubc45}  \\
\noalign{\smallskip}\hline\noalign{\smallskip}
Sontag12 &29&{\sffamily \fontseries{bx}\selectfont6}&{\sffamily
  \fontseries{bx}\selectfont1}&8&{\sffamily
  \fontseries{bx}\selectfont36}&{\sffamily
  \fontseries{bx}\selectfont23}&{\sffamily
  \fontseries{bx}\selectfont12}&9\\
GMPLP &29&12&2&12&66&66&52&24\\
PS(Ours) &11&10&7&8&49&45&38&10\\
$\pi$-S(Ours) &30&{\sffamily \fontseries{bx}\selectfont6}&5&{\sffamily
  \fontseries{bx}\selectfont6}&67&34&28&49\\
MI(Ours) &{\sffamily
  \fontseries{bx}\selectfont9}&9&4&10&37&36&29&{\sffamily
  \fontseries{bx}\selectfont7}\\
\noalign{\smallskip}\hline
\end{tabular}

\vspace{1mm}

{
\hspace{0pt}
\begin{tabular}{p{1.5cm}<{\centering}p{9mm}<{\centering}p{9mm}<{\centering}
p{9mm}<{\centering}p{9mm}<{\centering}p{9mm}<{\centering}p{9mm}<{\centering}>{\columncolor{Gray}} 
p{20.75mm}<{\centering}}
\hline
\begin{minipage}{1mm}
\vspace{1mm}
$~~$\\
\vspace{1mm}
\end{minipage}
&\footnotesize{ubc56} &\footnotesize{wall12} &\footnotesize{wall23}
&\footnotesize{wall34} &\footnotesize{wall45} &\footnotesize{wall56} &
{  Avg. Rank}\\
\hline
Sontag12 &{\sffamily \fontseries{bx}\selectfont1}&20&{\sffamily \fontseries{bx}\selectfont21}&{\sffamily \fontseries{bx}\selectfont19}&15&{\sffamily \fontseries{bx}\selectfont11} & { 2.600}\\
GMPLP &2&35&54&40&17&15 & { 4.200}\\
PS(Ours) &2&18&37&25&10&13 & { 2.967}\\
$\pi$-S(Ours) &2&44&288&32&27&27 & {2.867}\\
MI(Ours) &2&{\sffamily \fontseries{bx}\selectfont16}&24&25&{\sffamily \fontseries{bx}\selectfont7}&12 & {\sffamily \fontseries{bx}\selectfont  1.833}\begin{minipage}{0.05mm}$~$\\\vspace{1mm}
\end{minipage}\\
\hline
\end{tabular}
}
\end{table*}
Both the total running time and the number of iterations required to reach an exact solution for the matching problems are reported in Tables \ref{tab:RTCmp} and \ref{tab:ITCmp}. 
In several cases $\pi$-S 
takes an abnormally long time because it was trapped at a local optimum and 
cluster pursuit had to be applied to escape its basin of attraction. Two of the proposed 
methods, PS and $\pi$-S, take less running time and iterations in most cases 
as number of constraints and variables is sufficiently reduced without loosening 
the local marginal polytope. 
\section{Conclusion}
We have proposed a unified formulation of MAP LP relaxations which allows to conveniently compare different LP relaxation with different formulations of objectives and different dimensions of primal variables. With the unified formulation, a new tool, the Marginal Polytope Diagram, is proposed to describe LP relaxations. With a group of propositions, we can easily find equivalence between different marginal polytope diagrams. Thus constraint reduction can be carried out via the removal of redundant nodes and replacement of equivalent edges in the marginal polytope diagram. Together with the unified formulation and constraint reduction, we have also proposed three new message passing algorithms, two of which have shown significant speed up over the state-of-the-art
methods. Extension to marginal inference is of future work.

\bibliographystyle{spbasic}      
{\bibliography{ref}}

\begin{thebibliography}{26}
\providecommand{\natexlab}[1]{#1}
\providecommand{\url}[1]{{#1}}
\providecommand{\urlprefix}{URL }
\expandafter\ifx\csname urlstyle\endcsname\relax
  \providecommand{\doi}[1]{DOI~\discretionary{}{}{}#1}\else
  \providecommand{\doi}{DOI~\discretionary{}{}{}\begingroup
  \urlstyle{rm}\Url}\fi
\providecommand{\eprint}[2][]{\url{#2}}

\bibitem[{Batra et~al(2011)Batra, Nowozin, and Kohli}]{batra2011tighter}
Batra D, Nowozin S, Kohli P (2011) Tighter relaxations for map-mrf inference: A
  local primal-dual gap based separation algorithm. In: International
  Conference on Artificial Intelligence and Statistics, pp 146--154

\bibitem[{Blake et~al(2004)Blake, Rother, Brown, Perez, and
  Torr}]{blake2004interactive}
Blake A, Rother C, Brown M, Perez P, Torr P (2004) Interactive image
  segmentation using an adaptive gmmrf model. In: Computer Vision-ECCV 2004,
  Springer, pp 428--441

\bibitem[{Globerson and Jaakkola(2007)}]{globerson2007fixing}
Globerson A, Jaakkola T (2007) {Fixing max-product: Convergent message passing
  algorithms for MAP LP-relaxations}. In: NIPS, vol~21

\bibitem[{Hazan and Shashua(2010)}]{hazan2010norm}
Hazan T, Shashua A (2010) {Norm-product belief propagation: Primal-dual
  message-passing for approximate inference}. Information Theory, IEEE
  Transactions on 56(12):6294--6316

\bibitem[{Kallenberg(2002)}]{kallenberg2002foundations}
Kallenberg O (2002) Foundations of modern probability. Springer Verlag

\bibitem[{Kohli et~al(2009)Kohli, Ladick{\`y}, and Torr}]{kohli2009robust}
Kohli P, Ladick{\`y} L, Torr PH (2009) {Robust higher order potentials for
  enforcing label consistency}. IJCV 82(3):302--324

\bibitem[{Koller and Friedman(2009)}]{koller2009probabilistic}
Koller D, Friedman N (2009) Probabilistic graphical models: principles and
  techniques. MIT press

\bibitem[{Kolmogorov and Schoenemann(2012)}]{Kolmogorov2012}
Kolmogorov V, Schoenemann T (2012) Generalized sequential tree-reweighted
  message passing. arXiv preprint arXiv:12056352

\bibitem[{Komodakis and Paragios(2008)}]{komodakis2008beyond}
Komodakis N, Paragios N (2008) Beyond loose lp-relaxations: Optimizing mrfs by
  repairing cycles. In: Computer Vision--ECCV 2008, Springer, pp 806--820

\bibitem[{Komodakis et~al(2007)Komodakis, Paragios, and
  Tziritas}]{komodakis2007mrf}
Komodakis N, Paragios N, Tziritas G (2007) {MRF optimization via dual
  decomposition: Message-passing revisited}. In: ICCV, IEEE, pp 1--8

\bibitem[{Kovalevsky and Koval(1975)}]{kovalevsky1975diffusion}
Kovalevsky V, Koval V (1975) A diffusion algorithm for decreasing energy of
  max-sum labeling problem. Glushkov Institute of Cybernetics, Kiev, USSR

\bibitem[{Kumar et~al(2009)Kumar, Kolmogorov, and Torr}]{kumar2009analysis}
Kumar MP, Kolmogorov V, Torr PH (2009) {An analysis of convex relaxations for
  MAP estimation of discrete MRFs}. The Journal of Machine Learning Research
  10:71--106

\bibitem[{Li et~al(2010)Li, Kim, Huang, and He}]{li2010object}
Li H, Kim E, Huang X, He L (2010) {Object matching with a locally
  affine-invariant constraint}. In: CVPR, IEEE, pp 1641--1648

\bibitem[{Lowe(1999)}]{lowe1999object}
Lowe DG (1999) Object recognition from local scale-invariant features. In: ICCV
  1999, IEEE, vol~2, pp 1150--1157

\bibitem[{McEliece and Yildirim(2003)}]{mceliece2003belief}
McEliece RJ, Yildirim M (2003) Belief propagation on partially ordered sets.
  In: Mathematical systems theory in biology, communications, computation, and
  finance, Springer, pp 275--299

\bibitem[{Meshi et~al(2012)Meshi, Jaakkola, and
  Globerson}]{meshi2012convergence}
Meshi O, Jaakkola T, Globerson A (2012) Convergence rate analysis of map
  coordinate minimization algorithms. In: Advances in Neural Information
  Processing Systems 25, pp 3023--3031

\bibitem[{Pakzad and Anantharam(2005)}]{pakzad2005estimation}
Pakzad P, Anantharam V (2005) Estimation and marginalization using the kikuchi
  approximation methods. Neural Computation 17(8):1836--1873

\bibitem[{Schwing et~al(2012)Schwing, Hazan, Pollefeys, and
  Urtasun}]{SchwingNIPS2012}
Schwing AG, Hazan T, Pollefeys M, Urtasun R (2012) {Globally Convergent Dual
  MAP LP Relaxation Solvers using Fenchel-Young Margins}. In: Proc. NIPS

\bibitem[{Sontag et~al(2008)Sontag, Meltzer, Globerson, Weiss, and
  Jaakkola}]{SontagEtAl_uai08}
Sontag D, Meltzer T, Globerson A, Weiss Y, Jaakkola T (2008) Tightening {LP}
  relaxations for {MAP} using message-passing. In: UAI, AUAI Press, pp 503--510

\bibitem[{Sontag et~al(2011)Sontag, Globerson, and
  Jaakkola}]{SonGloJaa_optbook}
Sontag D, Globerson A, Jaakkola T (2011) Introduction to dual decomposition for
  inference. In: Sra S, Nowozin S, Wright SJ (eds) Optimization for Machine
  Learning, MIT Press

\bibitem[{Sontag et~al(2012)Sontag, Choe, and Li}]{SontagChoeLi_uai12}
Sontag D, Choe DK, Li Y (2012) {Efficiently Searching for Frustrated Cycles in
  {MAP} Inference}. In: UAI, AUAI Press, pp 795--804

\bibitem[{Wainwright and Jordan(2008)}]{wainwright2008graphical}
Wainwright MJ, Jordan MI (2008) Graphical models, exponential families, and
  variational inference. Foundations and Trends{\textregistered} in Machine
  Learning 1(1-2):1--305

\bibitem[{Werner(2008)}]{Werner2008}
Werner T (2008) High-arity interactions, polyhedral relaxations, and cutting
  plane algorithm for soft constraint optimisation (map-mrf). In: CVPR 2008.
  IEEE Conferenceon, IEEE, pp 1--8

\bibitem[{Werner(2010)}]{Werner2010}
Werner T (2010) Revisiting the linear programming relaxation approach to gibbs
  energy minimization and weighted constraint satisfaction. PAMI, IEEE
  Transactions on 32(8):1474--1488

\bibitem[{Yanover et~al(2006)Yanover, Meltzer, and Weiss}]{yanover2006linear}
Yanover C, Meltzer T, Weiss Y (2006) {Linear Programming Relaxations and Belief
  Propagation--An Empirical Study}. JMLR 7:1887--1907

\bibitem[{Yedidia et~al(2005)Yedidia, Freeman, and
  Weiss}]{yedidia2005constructing}
Yedidia J, Freeman W, Weiss Y (2005) {Constructing free-energy approximations
  and generalized belief propagationalgorithms}. Information Theory, IEEE
  Transactions on 51(7):2282--2312

\end{thebibliography}



\section*{Supplementary material}

# 1 Derivation of GDD Message updating

For convenience, we add several redundant constraints to reformulate (20) as

$$\left\{ \boldsymbol{\mu} \;\middle|\; \begin{array}{l} \mu_t(\mathbf{x}_t) \geqslant 0, \forall t \in \mathcal{T}, \mathbf{x}_t \\ \sum_{\mathbf{x}_t} \mu_t(\mathbf{x}_t) = 1, \forall t \in \mathcal{T} \\ \sum_{\mathbf{x}_{c \backslash s}} \mu_c(\mathbf{x}_c) = \mu_s(\mathbf{x}_s), \forall c \in \mathcal{C}', s \in \mathcal{S}(c), \mathbf{x}_s \end{array} \right\}. \tag{58}$$

Now we keep the first two groups of constraints (thus not correspond to any Lagrangian multipliers), and introduce Lagrangian multipliers $\{\lambda_{c \to s}(\mathbf{x}_s) | \forall c \in \mathcal{C}', s \in \mathcal{S}(c), \mathbf{x}_s\}$ to the third group of the constraints. By standard Lagrangian duality, we have the dual objective below,

$$g(\boldsymbol{\lambda}) = \max_{\substack{\forall t \in \mathcal{T}, \mathbf{x}_t, \mu_t(\mathbf{x}_t) \geqslant 0, \\ \sum_{\mathbf{x}_t} \mu_t(\mathbf{x}_t) = 1}} \left[ \sum_{c \in \mathcal{C}} \sum_{\mathbf{x}_c} \mu_c(\mathbf{x}_c)\theta_c(\mathbf{x}_c) + \sum_{c \in \mathcal{C}'} \sum_{s \in \mathcal{S}(c)} \sum_{\mathbf{x}_s} \left(\mu_s(\mathbf{x}_s) - \sum_{\mathbf{x}_{c \backslash s}} \mu_c(\mathbf{x}_c)\right) \lambda_{c \to s}(\mathbf{x}_s) \right]$$

$$= \max_{\substack{\forall t \in \mathcal{T}, \mathbf{x}_t, \mu_t(\mathbf{x}_t) \geqslant 0, \\ \sum_{\mathbf{x}_t} \mu_t(\mathbf{x}_t) = 1}} \left[ \sum_{c \in \mathcal{C}} \sum_{\mathbf{x}_c} \mu_c(\mathbf{x}_c)\theta_c(\mathbf{x}_c) + \sum_{c \in \mathcal{C}'} \sum_{s \in \mathcal{S}(c) \backslash \{c\}} \sum_{\mathbf{x}_s} \left(\mu_s(\mathbf{x}_s) - \sum_{\mathbf{x}_{c \backslash s}} \mu_c(\mathbf{x}_c)\right) \lambda_{c \to s}(\mathbf{x}_s) \right]. \tag{59}$$

Here the last equation holds because if $c \in \mathcal{S}(c)$ for some $c \in \mathcal{C}'$, $\lambda_{c \to c}(\mathbf{x}_c)$ is cancelled out. Rearranging variables in (59), the dual objective of GDD becomes:

$$g(\boldsymbol{\lambda}) = \max_{\substack{\forall t \in \mathcal{T}, \mathbf{x}_t, \mu_t(\mathbf{x}_t) \geqslant 0, \\ \sum_{\mathbf{x}_t} \mu_t(\mathbf{x}_t) = 1}} \left[ \sum_{c \in \mathcal{C}} \sum_{\mathbf{x}_c} \mu_c(\mathbf{x}_c)\theta_c(\mathbf{x}_c) + \sum_{c \in \mathcal{C}'} \sum_{s \in \mathcal{S}(c) \backslash \{c\}} \sum_{\mathbf{x}_s} \mu_s(\mathbf{x}_s)\lambda_{c \to s}(\mathbf{x}_s) - \sum_{c \in \mathcal{C}'} \sum_{s \in \mathcal{S}(c) \backslash \{c\}} \sum_{\mathbf{x}_c} \mu_c(\mathbf{x}_c)\lambda_{c \to s}(\mathbf{x}_s) \right]$$

$$= \max_{\substack{\forall t \in \mathcal{T}, \mathbf{x}_t, \mu_t(\mathbf{x}_t) \geqslant 0, \\ \sum_{\mathbf{x}_t} \mu_t(\mathbf{x}_t) = 1}} \left[ \sum_{c \in \mathcal{C}} \sum_{\mathbf{x}_c} \mu_c(\mathbf{x}_c)\theta_c(\mathbf{x}_c) + \sum_{s \in \left[\cup_{c' \in \mathcal{C}'} \left(\mathcal{S}(c') \backslash \{c'\}\right)\right]} \sum_{c \in \{c' | c' \in \mathcal{C}', s \in \mathcal{S}(c') \backslash \{c'\}\}} \sum_{\mathbf{x}_s} \mu_s(\mathbf{x}_s)\lambda_{c \to s}(\mathbf{x}_s) \right.$$

$$\left. - \sum_{c \in \mathcal{C}'} \sum_{s \in \mathcal{S}(c) \backslash \{c\}} \sum_{\mathbf{x}_c} \mu_c(\mathbf{x}_c)\lambda_{c \to s}(\mathbf{x}_s) \right]. \tag{60}$$

By definition of $\hat{\theta}_t(\mathbf{x}_t), t \in \mathcal{T}$ in (29a) and the fact that $\mathcal{C} \subseteq \mathcal{T}$, we have

$$\forall t \in \mathcal{T} \backslash \mathcal{C}, \quad \hat{\theta}_t(\mathbf{x}_t) = \mathbb{1}(t \in \mathcal{C})\theta_t(\mathbf{x}_t) = 0, \forall \mathbf{x}_t$$
$$\forall t \in \mathcal{C}, \quad \hat{\theta}_t(\mathbf{x}_t) = \mathbb{1}(t \in \mathcal{C})\theta_t(\mathbf{x}_t) = \theta_t(\mathbf{x}_t), \forall \mathbf{x}_t \tag{61}$$

Thus the first term in the most RHS of (60) can be reformulated as:

$$\sum_{c \in \mathcal{C}} \sum_{\mathbf{x}_c} \mu_c(\mathbf{x}_c)\theta_c(\mathbf{x}_c) = \sum_{c \in \mathcal{C}} \sum_{\mathbf{x}_c} \mu_c(\mathbf{x}_c)\hat{\theta}_c(\mathbf{x}_c) + \sum_{t \in \mathcal{T} \backslash \mathcal{C}} \sum_{\mathbf{x}_t} \mu_t(\mathbf{x}_t)\hat{\theta}_t(\mathbf{x}_t)$$
$$= \sum_{t \in \mathcal{T}} \sum_{\mathbf{x}_t} \mu_t(\mathbf{x}_t)\hat{\theta}_t(\mathbf{x}_t). \tag{62}$$

By the definition of $\mathcal{T}$ in (28) it is easy to verify that:

$$\forall t \in \mathcal{T} \backslash \left[\cup_{c' \in \mathcal{C}'} \left(\mathcal{S}(c') \backslash \{c'\}\right)\right], \nexists c \in \mathcal{C}', \text{ s.t. } t \in \mathcal{S}(c) \backslash \{c\}. \tag{63}$$

Thus we have

$$\sum_{t \in \mathcal{T} \backslash \left[\cup_{c' \in \mathcal{C}'} \left(\mathcal{S}(c') \backslash \{c'\}\right)\right]} \sum_{c \in \{c' | c' \in \mathcal{C}', t \in \mathcal{S}(c') \backslash \{c'\}\}} \sum_{\mathbf{x}_t} \mu_t(\mathbf{x}_t)\lambda_{c \to t}(\mathbf{x}_t) = 0. \tag{64}$$

As a result, the second term in the most RHS of (60) can be reformulated as:

$$\sum_{s \in \left[\cup_{c' \in \mathcal{C}'} \left(\mathcal{S}(c') \backslash \{c'\}\right)\right]} \sum_{c \in \{c' | c' \in \mathcal{C}', s \in \mathcal{S}(c') \backslash \{c'\}\}} \sum_{\mathbf{x}_s} \mu_s(\mathbf{x}_s) \lambda_{c \to s}(\mathbf{x}_s)$$

$$= \sum_{s \in \left[\cup_{c' \in \mathcal{C}'} \left(\mathcal{S}(c') \backslash \{c'\}\right)\right]} \sum_{c \in \{c' | c' \in \mathcal{C}', s \in \mathcal{S}(c') \backslash \{c'\}\}} \sum_{\mathbf{x}_s} \mu_s(\mathbf{x}_s) \lambda_{c \to s}(\mathbf{x}_s)$$
$$+ \sum_{t \in \mathcal{T} \backslash \left[\cup_{c' \in \mathcal{C}'} \left(\mathcal{S}(c') \backslash \{c'\}\right)\right]} \sum_{c \in \{c' | c' \in \mathcal{C}', t \in \mathcal{S}(c') \backslash \{c'\}\}} \sum_{\mathbf{x}_t} \mu_t(\mathbf{x}_t) \lambda_{c \to t}(\mathbf{x}_t)$$
$$= \sum_{t \in \mathcal{T}} \sum_{c \in \{c' | c' \in \mathcal{C}', t \in \mathcal{S}(c') \backslash \{c'\}\}} \sum_{\mathbf{x}_t} \mu_t(\mathbf{x}_t) \lambda_{c \to t}(\mathbf{x}_t). \tag{65}$$

For the third term in the most RHS of (60), we simply reformulate it as:

$$\sum_{c \in \mathcal{C}'} \sum_{s \in \mathcal{S}(c) \backslash \{c\}} \sum_{\mathbf{x}_c} \mu_c(\mathbf{x}_c) \lambda_{c \to s}(\mathbf{x}_s) = \sum_{t \in \mathcal{T}} \sum_{s \in \mathcal{S}(t) \backslash \{t\}} \sum_{\mathbf{x}_t} \mathbb{1}(t \in \mathcal{C}') \mu_t(\mathbf{x}_t) \lambda_{t \to s}(\mathbf{x}_s) \tag{66}$$

Using (62), (65) and (66), we have

$$g(\boldsymbol{\lambda}) = \max_{\substack{\forall t \in \mathcal{T}, \mathbf{x}_t, \boldsymbol{\mu}_t(\mathbf{x}_t) \geqslant 0, \\ \sum_{\mathbf{x}_t} \mu_t(\mathbf{x}_t) = 1}} \left[ \sum_{t \in \mathcal{T}} \sum_{\mathbf{x}_t} \mu_t(\mathbf{x}_t) \hat{\theta}_t(\mathbf{x}_t) + \sum_{t \in \mathcal{T}} \sum_{c \in \{c' | c' \in \mathcal{C}', t \in \mathcal{S}(c') \backslash \{c'\}\}} \sum_{\mathbf{x}_t} \mu_t(\mathbf{x}_t) \lambda_{c \to t}(\mathbf{x}_t) \right.$$
$$\left. - \sum_{t \in \mathcal{T}} \sum_{s \in \mathcal{S}(t) \backslash \{t\}} \sum_{\mathbf{x}_t} \mathbb{1}(t \in \mathcal{C}') \mu_t(\mathbf{x}_t) \lambda_{t \to s}(\mathbf{x}_s) \right].$$

$$= \max_{\substack{\forall t \in \mathcal{T}, \mathbf{x}_t, \boldsymbol{\mu}_t(\mathbf{x}_t) \geqslant 0, \\ \sum_{\mathbf{x}_t} \mu_t(\mathbf{x}_t) = 1}} \sum_{t \in \mathcal{T}} \sum_{\mathbf{x}_t} \mu_t(\mathbf{x}_t) \left[ \hat{\theta}_t(\mathbf{x}_t) + \sum_{c \in \{c' | c' \in \mathcal{C}', t \in \mathcal{S}(c') \backslash \{c'\}\}} \lambda_{c \to t}(\mathbf{x}_t) - \mathbb{1}(t \in \mathcal{C}') \sum_{s \in \mathcal{S}(t) \backslash \{t\}} \lambda_{t \to s}(\mathbf{x}_s) \right] \tag{67a}$$

$$= \max_{\substack{\forall t \in \mathcal{T}, \mathbf{x}_t, \boldsymbol{\mu}_t(\mathbf{x}_t) \geqslant 0, \\ \sum_{\mathbf{x}_t} \mu_t(\mathbf{x}_t) = 1}} \sum_{t \in \mathcal{T}} \sum_{\mathbf{x}_t} \mu_t(\mathbf{x}_t) \left[ \hat{\theta}_t(\mathbf{x}_t) + \lambda_t(\mathbf{x}_t) - \gamma_t(\mathbf{x}_t) \right] \tag{67b}$$

$$= \max_{\substack{\forall t \in \mathcal{T}, \mathbf{x}_t, \boldsymbol{\mu}_t(\mathbf{x}_t) \geqslant 0, \\ \sum_{\mathbf{x}_t} \mu_t(\mathbf{x}_t) = 1}} \sum_{t \in \mathcal{T}} \sum_{\mathbf{x}_t} \mu_t(\mathbf{x}_t) b_t(\mathbf{x}_t) \tag{67c}$$

$$= \sum_{t \in \mathcal{T}} \max_{\mathbf{x}_t} b_t(\mathbf{x}_t). \tag{67d}$$

Here from (67a) to (67c) we use the definition of $\gamma_t(\mathbf{x}_t)$, $\lambda_t(\mathbf{x}_t)$ and $b_t(\mathbf{x}_t), t \in \mathcal{T}$ in (29). From (67c) to (67d), as $\forall t \in \mathcal{T}, \mathbf{x}_t, \mu_t(\mathbf{x}_t) \geqslant 0, \sum_{\mathbf{x}_t} \mu_t(\mathbf{x}_t) = 1$, the maximum can be attained by letting $\mu_t(\mathbf{x}_t^*) = 1$ for some $\mathbf{x}_t^* \in$ argmax$_{\mathbf{x}_t} b_t(\mathbf{x}_t)$. When applying coordinate descent to optimise the above problem, we pick a particular $c \in \mathcal{C}'$ and then fix all $\boldsymbol{\lambda}$ except those $\lambda_{c \to s}(\mathbf{x}_s), s \in \mathcal{S}(c)$. Recall the definition of $\mathcal{T}$ in (28), we can reformulate (28) as

$$\mathcal{T} = \bigcup_{c' \in \mathcal{C}'} (\{c'\} \cup \mathcal{S}(c')) = \left[\{c\} \cup \left(\mathcal{S}(c) \backslash \{c\}\right)\right] \cup (\mathcal{T} \backslash (\{c\} \cup \mathcal{S}(c))). \tag{68}$$

Thus by definition of $b_t(\mathbf{x}_t)$ in (29d), $g(\boldsymbol{\lambda})$ can be decomposed to three parts as follows:

$$g(\boldsymbol{\lambda}) = \sum_{t \in \mathcal{T}} \max_{\mathbf{x}_t} \left[ \hat{\theta}_t(\mathbf{x}_t) - \gamma_t(\mathbf{x}_t) + \lambda_t(\mathbf{x}_t) \right]$$

$$= \max_{\mathbf{x}_c} \left[ \hat{\theta}_c(\mathbf{x}_c) - \gamma_c(\mathbf{x}_c) + \lambda_c(\mathbf{x}_c) \right] + \sum_{s \in \mathcal{S}(c) \setminus \{c\}} \max_{\mathbf{x}_s} \left[ \hat{\theta}_s(\mathbf{x}_s) - \gamma_s(\mathbf{x}_s) + \lambda_s(\mathbf{x}_s) \right]$$

$$+ \sum_{t \in \mathcal{T} \setminus (\{c\} \cup \mathcal{S}(c))} \max_{\mathbf{x}_t} \left[ \hat{\theta}_t(\mathbf{x}_t) - \gamma_t(\mathbf{x}_t) + \lambda_t(\mathbf{x}_t) \right]$$

$$= \max_{\mathbf{x}_c} \left[ \hat{\theta}_c(\mathbf{x}_c) - \sum_{s \in \mathcal{S}(c) \setminus \{c\}} \lambda_{c \to s}(\mathbf{x}_s) + \lambda_c(\mathbf{x}_c) \right] \tag{69a}$$

$$+ \sum_{s \in \mathcal{S}(c) \setminus \{c\}} \max_{\mathbf{x}_s} \left[ \hat{\theta}_s(\mathbf{x}_s) - \gamma_s(\mathbf{x}_s) + \lambda_s^{-c}(\mathbf{x}_s) + \lambda_{c \to s}(\mathbf{x}_s) \right] \tag{69b}$$

$$+ \sum_{t \in \mathcal{T} \setminus (\{c\} \cup \mathcal{S}(c))} \max_{\mathbf{x}_t} \left[ \hat{\theta}_t(\mathbf{x}_t) - \gamma_t(\mathbf{x}_t) + \lambda_t(\mathbf{x}_t) \right] \tag{69c}$$

Note that only (69a) and (69b) depend on $\lambda_{c \to s}(\mathbf{x}_s), s \in \mathcal{S}(c)$, thus minimising $g(\boldsymbol{\lambda})$ over all $\lambda_{c \to s}(\mathbf{x}_s), s \in \mathcal{S}(c)$ is equivalent to the sub-optimisation problem in (31).

An optimal solution of (31) is provided in the following proposition.

**Proposition 1** $\forall s \in \mathcal{S}(c) \setminus \{c\}, \mathbf{x}_s$, let

$$\lambda_{c \to s}^*(\mathbf{x}_s) = -\hat{\theta}_s(\mathbf{x}_s) + \gamma_s(\mathbf{x}_s) - \lambda_s^{-c}(\mathbf{x}_s)$$
$$+ \frac{1}{|\mathcal{S}(c) \setminus \{c\}|} \max_{\mathbf{x}_{c \setminus s}} \left[ \hat{\theta}_c(\mathbf{x}_c) + \lambda_c(\mathbf{x}_c) + \sum_{\hat{s} \in \mathcal{S}(c) \setminus \{c\}} \left( \hat{\theta}_{\hat{s}}(\mathbf{x}_{\hat{s}}) - \gamma_{\hat{s}}(\mathbf{x}_{\hat{s}}) + \lambda_{\hat{s}}^{-c}(\mathbf{x}_{\hat{s}}) \right) \right],$$

then $\boldsymbol{\lambda}_{c,\mathcal{S}(c)}^* = (\lambda_{c \to s}^*(\mathbf{x}_s))_{s \in \mathcal{S}(c) \setminus \{c\}}$ is a solution of (31).

**Proof** Considering the objective of (31), we have:

$$g_c(\boldsymbol{\lambda}_{c,\mathcal{S}(c)}) = \max_{\mathbf{x}_c} \left[ \hat{\theta}_c(\mathbf{x}_c) - \sum_{s \in \mathcal{S}(c) \setminus \{c\}} \lambda_{c \to s}(\mathbf{x}_s) + \lambda_c(\mathbf{x}_c) \right] \tag{70a}$$

$$+ \sum_{s \in \mathcal{S}(c) \setminus \{c\}} \max_{\mathbf{x}_s} \left[ \hat{\theta}_s(\mathbf{x}_s) - \gamma_s(\mathbf{x}_s) + \lambda_s^{-c}(\mathbf{x}_s) + \lambda_{c \to s}(\mathbf{x}_s) \right] \tag{70b}$$

$$\geq \max_{\mathbf{x}_c} \left\{ \left[ \hat{\theta}_c(\mathbf{x}_c) - \sum_{s \in \mathcal{S}(c) \setminus \{c\}} \lambda_{c \to s}(\mathbf{x}_s) + \lambda_c(\mathbf{x}_c) \right] \right.$$

$$\left. + \sum_{s \in \mathcal{S}(c) \setminus \{c\}} \left[ \hat{\theta}_s(\mathbf{x}_s) - \gamma_s(\mathbf{x}_s) + \lambda_s^{-c}(\mathbf{x}_s) + \lambda_{c \to s}(\mathbf{x}_s) \right] \right\}$$

$$= \max_{\mathbf{x}_c} \left[ \hat{\theta}_c(\mathbf{x}_c) + \lambda_c(\mathbf{x}_c) + \sum_{s \in \mathcal{S}(c) \setminus \{c\}} \left[ \hat{\theta}_s(\mathbf{x}_s) - \gamma_s(\mathbf{x}_s) + \lambda_s^{-c}(\mathbf{x}_s) \right] \right]. \tag{70c}$$

Clearly the RHS of (70c) is a lower bound of $g_c(\boldsymbol{\lambda}_{c,\mathcal{S}(c)})$ for arbitrary $\boldsymbol{\lambda}_{c,\mathcal{S}(c)}$. Now we show that lower bound is attained when $\boldsymbol{\lambda}_{c,\mathcal{S}(c)} = \boldsymbol{\lambda}_{c,\mathcal{S}(c)}^*$.

When $\lambda_{c \to s}(\mathbf{x}_s) = \lambda^*_{c \to s}(\mathbf{x}_s), \forall \mathbf{x}_s$ for each $s \in \mathcal{S}(c) \setminus \{c\}$, we have

$$\begin{aligned}
&\hat{\theta}_s(\mathbf{x}_s) - \gamma_s(\mathbf{x}_s) + \lambda_s^{-c}(\mathbf{x}_s) + \lambda^*_{c \to s}(\mathbf{x}_s) \\
=& \hat{\theta}_s(\mathbf{x}_s) - \gamma_s(\mathbf{x}_s) + \lambda_s^{-c}(\mathbf{x}_s) + \\
& \left\{ -\hat{\theta}_s(\mathbf{x}_s) + \gamma_s(\mathbf{x}_s) - \lambda_s^{-c}(\mathbf{x}_s) \right. \\
& \left. + \frac{1}{|\mathcal{S}(c) \setminus \{c\}|} \max_{\mathbf{x}_{c \setminus s}} \left[ \hat{\theta}_c(\mathbf{x}_c) + \lambda_c(\mathbf{x}_c) + \sum_{\hat{s} \in \mathcal{S}(c) \setminus \{c\}} \left( \hat{\theta}_{\hat{s}}(\mathbf{x}_{\hat{s}}) - \gamma_{\hat{s}}(\mathbf{x}_{\hat{s}}) + \lambda_{\hat{s}}^{-c}(\mathbf{x}_{\hat{s}}) \right) \right] \right\} \\
=& \frac{1}{|\mathcal{S}(c) \setminus \{c\}|} \max_{\mathbf{x}_{c \setminus s}} \left[ \hat{\theta}_c(\mathbf{x}_c) + \lambda_c(\mathbf{x}_c) + \sum_{\hat{s} \in \mathcal{S}(c) \setminus \{c\}} \left( \hat{\theta}_{\hat{s}}(\mathbf{x}_{\hat{s}}) - \gamma_{\hat{s}}(\mathbf{x}_{\hat{s}}) + \lambda_{\hat{s}}^{-c}(\mathbf{x}_{\hat{s}}) \right) \right], \forall \mathbf{x}_s. \quad (71)
\end{aligned}$$

Thus when $\boldsymbol{\lambda}_{c,\mathcal{S}(c)} = \boldsymbol{\lambda}^*_{c,\mathcal{S}(c)}$, (70b) becomes:

$$\begin{aligned}
&\sum_{s \in \mathcal{S}(c) \setminus \{c\}} \max_{\mathbf{x}_s} \left[ \hat{\theta}_s(\mathbf{x}_s) - \gamma_s(\mathbf{x}_s) + \lambda_s^{-c}(\mathbf{x}_s) + \lambda^*_{c \to s}(\mathbf{x}_s) \right] \\
=& \sum_{s \in \mathcal{S}(c) \setminus \{c\}} \max_{\mathbf{x}_s} \frac{1}{|\mathcal{S}(c) \setminus \{c\}|} \max_{\mathbf{x}_{c \setminus s}} \left[ \hat{\theta}_c(\mathbf{x}_c) + \lambda_c(\mathbf{x}_c) + \sum_{\hat{s} \in \mathcal{S}(c) \setminus \{c\}} \left( \hat{\theta}_{\hat{s}}(\mathbf{x}_{\hat{s}}) - \gamma_{\hat{s}}(\mathbf{x}_{\hat{s}}) + \lambda_{\hat{s}}^{-c}(\mathbf{x}_{\hat{s}}) \right) \right] \\
=& \max_{\mathbf{x}_c} \left[ \hat{\theta}_c(\mathbf{x}_c) + \lambda_c(\mathbf{x}_c) + \sum_{\hat{s} \in \mathcal{S}(c) \setminus \{c\}} \left( \hat{\theta}_{\hat{s}}(\mathbf{x}_{\hat{s}}) - \gamma_{\hat{s}}(\mathbf{x}_{\hat{s}}) + \lambda_{\hat{s}}^{-c}(\mathbf{x}_{\hat{s}}) \right) \right]. \quad (72)
\end{aligned}$$

Thus $g_c(\boldsymbol{\lambda}_{c,\mathcal{S}(c)})$ becomes:

$$\begin{aligned}
g_c(\boldsymbol{\lambda}^*_{c,\mathcal{S}(c)}) =& \max_{\mathbf{x}_c} \left[ \hat{\theta}_c(\mathbf{x}_c) - \sum_{s \in \mathcal{S}(c) \setminus \{c\}} \lambda^*_{c \to s}(\mathbf{x}_s) + \lambda_c(\mathbf{x}_c) \right] + \sum_{s \in \mathcal{S}(c) \setminus \{c\}} \max_{\mathbf{x}_s} \left[ \hat{\theta}_s(\mathbf{x}_s) - \gamma_s(\mathbf{x}_s) + \lambda_s^{-c}(\mathbf{x}_s) + \lambda^*_{c \to s}(\mathbf{x}_s) \right] \\
=& \max_{\mathbf{x}_c} \left[ \hat{\theta}_c(\mathbf{x}_c) - \sum_{s \in \mathcal{S}(c) \setminus \{c\}} \lambda^*_{c \to s}(\mathbf{x}_s) + \lambda_c(\mathbf{x}_c) \right] \\
&+ \max_{\mathbf{x}_c} \left[ \hat{\theta}_c(\mathbf{x}_c) + \lambda_c(\mathbf{x}_c) + \sum_{\hat{s} \in \mathcal{S}(c) \setminus \{c\}} \left( \hat{\theta}_{\hat{s}}(\mathbf{x}_{\hat{s}}) - \gamma_{\hat{s}}(\mathbf{x}_{\hat{s}}) + \lambda_{\hat{s}}^{-c}(\mathbf{x}_{\hat{s}}) \right) \right]. \quad (73)
\end{aligned}$$

As the RHS of (70c) is a lower bound of $g_c(\boldsymbol{\lambda}_{c,\mathcal{S}(c)})$ for arbitrary $\boldsymbol{\lambda}_{c,\mathcal{S}(c)}$, thus we must have

$$\max_{\mathbf{x}_c} \left[ \hat{\theta}_c(\mathbf{x}_c) - \sum_{s \in \mathcal{S}(c)} \lambda^*_{c \to s}(\mathbf{x}_s) + \lambda_c(\mathbf{x}_c) \right] \geqslant 0, \quad (74)$$

which implies that the RHS of (70a) is non-negative.

Now we show that the RHS of (70a) is also non-positive.

$$\max_{\mathbf{x}_c}\big[\hat{\theta}_c(\mathbf{x}_c) - \sum_{s\in\mathcal{S}(c)\setminus\{c\}} \lambda^*_{c\to s}(\mathbf{x}_s) + \lambda_c(\mathbf{x}_c)\big]$$

$$=\max_{\mathbf{x}_c}\big[\hat{\theta}_c(\mathbf{x}_c) + \lambda_c(\mathbf{x}_c) - \sum_{s\in\mathcal{S}(c)\setminus\{c\}} \lambda^*_{c\to s}(\mathbf{x}_s)\big]$$

$$=\max_{\mathbf{x}_c}\Big\{\hat{\theta}_c(\mathbf{x}_c) + \lambda_c(\mathbf{x}_c) - \sum_{s\in\mathcal{S}(c)\setminus\{c\}}\Big[-\hat{\theta}_s(\mathbf{x}_s) + \gamma_s(\mathbf{x}_s) - \lambda_s^{-c}(\mathbf{x}_s)$$
$$+ \frac{1}{|\mathcal{S}(c)\setminus\{c\}|} \max_{\mathbf{x}_{c\setminus s}}\Big[\sum_{\hat{s}\in\mathcal{S}(c)\setminus\{c\}} \big(\hat{\theta}_{\hat{s}}(\mathbf{x}_{\hat{s}}) - \gamma_{\hat{s}}(\mathbf{x}_{\hat{s}}) + \lambda_{\hat{s}}^{-c}(\mathbf{x}_{\hat{s}})\big) + \hat{\theta}_c(\mathbf{x}_c) + \lambda_c(\mathbf{x}_c)\Big]\Big]\Big\}$$

$$=\frac{1}{|\mathcal{S}(c)\setminus\{c\}|}\max_{\mathbf{x}_c}\sum_{s\in\mathcal{S}(c)\setminus\{c\}}\Big\{\hat{\theta}_c(\mathbf{x}_c) + \lambda_c(\mathbf{x}_c) - \sum_{s\in\mathcal{S}(c)\setminus\{c\}}\Big[-\hat{\theta}_s(\mathbf{x}_s) + \gamma_s(\mathbf{x}_s) - \lambda_s^{-c}(\mathbf{x}_s)\Big]$$
$$-\max_{\mathbf{x}_{c\setminus s}}\Big[\hat{\theta}_c(\mathbf{x}_c) + \lambda_c(\mathbf{x}_c) - \sum_{s\in\mathcal{S}(c)\setminus\{c\}}\Big[-\hat{\theta}_s(\mathbf{x}_s) + \gamma_s(\mathbf{x}_s) - \lambda_s^{-c}(\mathbf{x}_s)\Big]\Big]\Big\}$$

$$\leqslant \frac{1}{|\mathcal{S}(c)\setminus\{c\}|}\sum_{s\in\mathcal{S}(c)\setminus\{c\}}\max_{\mathbf{x}_c}\Big\{\hat{\theta}_c(\mathbf{x}_c) + \lambda_c(\mathbf{x}_c) - \sum_{s\in\mathcal{S}(c)\setminus\{c\}}\Big[-\hat{\theta}_s(\mathbf{x}_s) + \gamma_s(\mathbf{x}_s) - \lambda_s^{-c}(\mathbf{x}_s)\Big]$$
$$-\max_{\mathbf{x}_{c\setminus s}}\Big[\hat{\theta}_c(\mathbf{x}_c) + \lambda_c(\mathbf{x}_c) - \sum_{s\in\mathcal{S}(c)\setminus\{c\}}\Big[-\hat{\theta}_s(\mathbf{x}_s) + \gamma_s(\mathbf{x}_s) - \lambda_s^{-c}(\mathbf{x}_s)\Big]\Big]\Big\}$$

$$=\frac{1}{|\mathcal{S}(c)\setminus\{c\}|}\sum_{s\in\mathcal{S}(c)\setminus\{c\}}\max_{\mathbf{x}_s}\Big\{\max_{\mathbf{x}_{c\setminus s}}\Big[\hat{\theta}_c(\mathbf{x}_c) + \lambda_c(\mathbf{x}_c) - \sum_{s\in\mathcal{S}(c)\setminus\{c\}}\Big[-\hat{\theta}_s(\mathbf{x}_s) + \gamma_s(\mathbf{x}_s) - \lambda_s^{-c}(\mathbf{x}_s)\Big]\Big]$$
$$-\max_{\mathbf{x}_{c\setminus s}}\Big[\hat{\theta}_c(\mathbf{x}_c) + \lambda_c(\mathbf{x}_c) - \sum_{s\in\mathcal{S}(c)\setminus\{c\}}\Big[-\hat{\theta}_s(\mathbf{x}_s) + \gamma_s(\mathbf{x}_s) - \lambda_s^{-c}(\mathbf{x}_s)\Big]\Big]\Big\}$$

$$=0. \tag{75}$$

Thus using (74) and (75) we have

$$\max_{\mathbf{x}_c}\big[\hat{\theta}_c(\mathbf{x}_c) - \sum_{s\in\mathcal{S}(c)\setminus\{c\}} \lambda^*_{c\to s}(\mathbf{x}_s) + \lambda_c(\mathbf{x}_c)\big] = 0. \tag{76}$$

Thus using (73) and (76), we have

$$g_c(\boldsymbol{\lambda}^*_{c,\mathcal{S}(c)}) = \max_{\mathbf{x}_c}\Big[\hat{\theta}_c(\mathbf{x}_c) + \lambda_c(\mathbf{x}_c) + \sum_{s\in\mathcal{S}(c)\setminus\{c\}}\big[\hat{\theta}_s(\mathbf{x}_s) - \gamma_s(\mathbf{x}_s) + \lambda_s^{-c}(\mathbf{x}_s)\big]\Big], \tag{77}$$

which implies by specifying $\boldsymbol{\lambda}_{c,\mathcal{S}(c)} = \boldsymbol{\lambda}^*_{c,\mathcal{S}(c)}$, $g_c(\boldsymbol{\lambda}^*_{c,\mathcal{S}(c)})$ achieves the lower bound shown in RHS of (70c). ∎

## 2 Dual Decrease in a single coordinate descent step

**Proposition 2 (Dual Decrease)** *For any $c \in \mathcal{C}'$, the dual decrease*

$$d(c) = \max_{\mathbf{x}_c} b_c(\mathbf{x}_c) + \sum_{s\in\mathcal{S}(c)\setminus\{c\}} \max_{\mathbf{x}_s} b_s(\mathbf{x}_s)$$
$$-\max_{\mathbf{x}_c}\Big[b_c(\mathbf{x}_c) + \sum_{s\in\mathcal{S}(c)\setminus\{c\}} b_s(\mathbf{x}_s)\Big] \geqslant 0.$$

5

**Proof** Considering the sub-optimisation problem in (31), by definition of $b_t(\mathbf{x}_t), t \in \mathcal{T}$ we have

$$b_c(\mathbf{x}_c) = \hat{\theta}_c(\mathbf{x}_c) + \lambda_c(\mathbf{x}_c) - \sum_{s \in \mathcal{S}(c) \setminus \{c\}} \lambda_{c \to s}(\mathbf{x}_s), \forall \mathbf{x}_c, \tag{78a}$$

$$b_s(\mathbf{x}_s) = \hat{\theta}_s(\mathbf{x}_s) - \gamma_s(\mathbf{x}_s) + \sum_{\hat{c} \in \{c' | c' \in \mathcal{C}', s \in \mathcal{S}(c') \setminus \{c'\}\}} \lambda_{\hat{c} \to s}(\mathbf{x}_s)$$

$$= \hat{\theta}_s(\mathbf{x}_s) - \gamma_s(\mathbf{x}_s) + \sum_{\hat{c} \in \{c' | c' \in \mathcal{C}', c' \neq c, s \in \mathcal{S}(c') \setminus \{c'\}\}} \lambda_{\hat{c} \to s}(\mathbf{x}_s) + \lambda_{c \to s}(\mathbf{x}_s)$$

$$= \hat{\theta}_s(\mathbf{x}_s) - \gamma_s(\mathbf{x}_s) + \lambda_s^{-c}(\mathbf{x}_s) + \lambda_{c \to s}(\mathbf{x}_s), \forall s \in \mathcal{S}(c) \setminus \{c\}, \mathbf{x}_s. \tag{78b}$$

When considering the sub-optimisation problem in (31), only $\lambda_{c \to s}(\mathbf{x}_s), s \in \mathcal{S}(c)$ in (78) are flexible. Thus $b_c(\mathbf{x}_c)$ and $b_s(\mathbf{x}_s), s \in \mathcal{S}(c)$ can be determined by $\boldsymbol{\lambda}_{c,\mathcal{S}(c)}$, and the following equation always holds:

$$g_c(\boldsymbol{\lambda}_{c,\mathcal{S}(c) \setminus \{c\}}) = \max_{\mathbf{x}_c} \left[ \hat{\theta}_c(\mathbf{x}_c) - \sum_{s \in \mathcal{S}(c) \setminus \{c\}} \lambda_{c \to s}(\mathbf{x}_s) + \lambda_c(\mathbf{x}_c) \right] + \sum_{s \in \mathcal{S}(c) \setminus \{c\}} \max_{\mathbf{x}_s} \left[ \hat{\theta}_s(\mathbf{x}_s) - \gamma_s(\mathbf{x}_s) + \lambda_s^{-c}(\mathbf{x}_s) + \lambda_{c \to s}(\mathbf{x}_s) \right]$$

$$= \max_{\mathbf{x}_c} b_c(\mathbf{x}_c) + \sum_{s \in \mathcal{S}(c) \setminus \{c\}} \max_{\mathbf{x}_s} b_s(\mathbf{x}_s). \tag{79}$$

Now we evaluate the optimal objective. By (77) and definition of $b_t(\mathbf{x}_t), t \in \mathcal{T}$ in (29d) we have

$$g_c(\boldsymbol{\lambda}_{c,\mathcal{S}(c)}^*) = \max_{\mathbf{x}_c} \left[ \hat{\theta}_c(\mathbf{x}_c) + \lambda_c(\mathbf{x}_c) + \sum_{s \in \mathcal{S}(c) \setminus \{c\}} [\hat{\theta}_s(\mathbf{x}_s) - \gamma_s(\mathbf{x}_s) + \lambda_s^{-c}(\mathbf{x}_s)] \right]$$

$$= \max_{\mathbf{x}_c} \left[ \hat{\theta}_c(\mathbf{x}_c) + \lambda_c(\mathbf{x}_c) - \sum_{s \in \mathcal{S}(c) \setminus \{c\}} \lambda_{c \to s}(\mathbf{x}_s) + \sum_{s \in \mathcal{S}(c) \setminus \{c\}} [\hat{\theta}_s(\mathbf{x}_s) - \gamma_s(\mathbf{x}_s) + \lambda_s^{-c}(\mathbf{x}_s) + \lambda_{c \to s}(\mathbf{x}_s)] \right]$$

$$= \max_{\mathbf{x}_c} \left[ b_c(\mathbf{x}_c) + \sum_{s \in \mathcal{S}(c) \setminus \{c\}} b_s(\mathbf{x}_s) \right]. \tag{80}$$

Thus the dual decrease in a single coordinate descent step is

$$d(c) = g_c(\boldsymbol{\lambda}_{c,\mathcal{S}(c)}) - g_c(\boldsymbol{\lambda}_{c,\mathcal{S}(c)}^*)$$

$$= \max_{\mathbf{x}_c} b_c(\mathbf{x}_c) + \sum_{s \in \mathcal{S}(c) \setminus \{c\}} \max_{\mathbf{x}_s} b_s(\mathbf{x}_s) - \max_{\mathbf{x}_c} \left[ b_c(\mathbf{x}_c) + \sum_{s \in \mathcal{S}(c) \setminus \{c\}} b_s(\mathbf{x}_s) \right]$$

$$\geq 0. \tag{81}$$

∎

## 3 Proof of Proposition 5

**Proposition 5** *If there exists $\mathbf{x}$ that maximises $b_t(\mathbf{x}_t), \forall t \in \mathcal{T}$, the solution of GDD is exact.*

**Proof** Under the above assumptions, we have

$$g(\boldsymbol{\lambda}) = \sum_{t \in \mathcal{T}} \max_{\mathbf{x}_c} b_c(\mathbf{x}_c) = \max_{\mathbf{x}} \sum_{t \in \mathcal{T}} b_t(\mathbf{x}_t)$$

$$= \max_{\mathbf{x}} \left[ \sum_{t \in \mathcal{T}} \hat{\theta}_t(\mathbf{x}_t) + \sum_{t \in \mathcal{T}} \lambda_t(\mathbf{x}_t) - \sum_{t \in \mathcal{T}} \gamma_t(\mathbf{x}_t) \right]$$

$$= \max_{\mathbf{x}} \left[ \sum_{t \in \mathcal{T}} \mathbb{1}(t \in \mathcal{C}) \theta_t(\mathbf{x}_t) + \sum_{t \in \mathcal{T}} \sum_{c \in \{c' | c' \in \mathcal{C}', t \in \mathcal{S}(c') \setminus \{c'\}\}} \lambda_{c \to t}(\mathbf{x}_t) - \sum_{t \in \mathcal{T}} \mathbb{1}(t \in \mathcal{C}') \sum_{\hat{s} \in \mathcal{S}(t) \setminus \{t\}} \lambda_{t \to \hat{s}}(\mathbf{x}_{\hat{s}}) \right].$$

6

As it is easy to verify that

$$\sum_{t\in\mathcal{T}}\sum_{c\in\{c'|c'\in\mathcal{C}',t\in\mathcal{S}(c')\setminus\{c'\}\}}\lambda_{c\to t}(\mathbf{x}_t) = \sum_{c\in\cup_{t\in\mathcal{T}}\{c'|c'\in\mathcal{C}',t\in\mathcal{S}(c')\setminus\{c'\}\}}\sum_{s\in\mathcal{S}(c)\setminus\{c\}}\lambda_{c\to s}(\mathbf{x}_s)$$

$$= \sum_{c\in\mathcal{C}'}\sum_{s\in\mathcal{S}(c)\setminus\{c\}}\lambda_{c\to s}(\mathbf{x}_s)$$

$$= \sum_{t\in\mathcal{T}}\mathbb{1}(t\in\mathcal{C}')\sum_{\hat{s}\in\mathcal{S}(t)\setminus\{t\}}\lambda_{t\to\hat{s}}(\mathbf{x}_{\hat{s}}).$$

Thus we have

$$g(\boldsymbol{\lambda}) = \max_{\mathbf{x}}\bigl[\sum_{t\in\mathcal{T}}\mathbb{1}(t\in\mathcal{C})\theta_t(\mathbf{x}_t)\bigr] = \max_{\mathbf{x}}\sum_{c\in\mathcal{C}}\theta_c(\mathbf{x}_c),$$

which completes the proof. ∎

# 4 Derivation of Belief Propagation Without Messages

**Proposition 6** *When optimising* (31), *the beliefs* $\mathbf{b}^*_{c,\mathcal{S}(c)}$ *can be computed from a* $\mathbf{b}_{c,\mathcal{S}(c)}$ *determined by arbitrary* $\boldsymbol{\lambda}_{c,\mathcal{S}(c)}$ *as following:*

$$b^*_s(\mathbf{x}_s) = \frac{\max_{\mathbf{x}_{c\setminus s}}\left[b_c(\mathbf{x}_c) + \sum_{\hat{s}\in\mathcal{S}(c)\setminus\{c\}}b_{\hat{s}}(\mathbf{x}_{\hat{s}})\right]}{|\mathcal{S}(c)\setminus\{c\}|}, \forall s\in\mathcal{S}(c)\setminus\{c\}, \mathbf{x}_s$$

$$b^*_c(\mathbf{x}_c) = b_c(\mathbf{x}_c) + \sum_{\hat{s}\in\mathcal{S}(c)\setminus\{c\}}b_{\hat{s}}(\mathbf{x}_{\hat{s}}) - \sum_{\hat{s}\in\mathcal{S}(c)\setminus\{c\}}b^*_{\hat{s}}(\mathbf{x}_{\hat{s}}), \forall \mathbf{x}_c.$$

**Proof** By (78) we have:

$$b^*_c(\mathbf{x}_c) = \hat{\theta}_c(\mathbf{x}_c) + \lambda_c(\mathbf{x}_c) - \sum_{s\in\mathcal{S}(c)\setminus\{c\}}\lambda^*_{c\to s}(\mathbf{x}_s), \forall \mathbf{x}_c. \tag{82a}$$

$$b^*_s(\mathbf{x}_s) = \hat{\theta}_s(\mathbf{x}_s) - \gamma_s(\mathbf{x}_s) + \lambda^{-c}_s(\mathbf{x}_s) + \lambda^*_{c\to s}(\mathbf{x}_s), s\in\mathcal{S}(c)\setminus\{c\}, \mathbf{x}_s \tag{82b}$$

According to (82b) and (78b), we have:

$$b^*_s(\mathbf{x}_s) - b_s(\mathbf{x}_s) = \lambda^*_{c\to s}(\mathbf{x}_s) - \lambda_{c\to s}(\mathbf{x}_s), \forall s\in\mathcal{S}(c)\setminus\{c\}, \mathbf{x}_s. \tag{83}$$

According to (82a) and (78a), we have:

$$b^*_c(\mathbf{x}_c) - b_c(\mathbf{x}_c) = \sum_{s\in\mathcal{S}(c)\setminus\{c\}}\lambda_{c\to s}(\mathbf{x}_s) - \sum_{s\in\mathcal{S}(c)\setminus\{c\}}\lambda^*_{c\to s}(\mathbf{x}_s), \forall \mathbf{x}_c. \tag{84}$$

Rearranging (78b) yields:

$$b_s(\mathbf{x}_s) - \lambda_{c\to s}(\mathbf{x}_s) = \hat{\theta}_s(\mathbf{x}_s) - \gamma_s(\mathbf{x}_s) + \lambda^{-c}_s(\mathbf{x}_s), \forall s\in\mathcal{S}(c)\setminus\{c\}, \mathbf{x}_s. \tag{85}$$

By (78a) and (78b), we have:

$$b_c(\mathbf{x}_c) + \sum_{s\in\mathcal{S}(c)\setminus\{c\}}b_s(\mathbf{x}_s) = \hat{\theta}_c(\mathbf{x}_c) + \lambda_c(\mathbf{x}_c) - \sum_{s\in\mathcal{S}(c)\setminus\{c\}}\lambda_{c\to s}(\mathbf{x}_s) + \sum_{s\in\mathcal{S}(c)\setminus\{c\}}\bigl[\hat{\theta}_s(\mathbf{x}_s) - \gamma_s(\mathbf{x}_s) + \lambda^{-c}_s(\mathbf{x}_s) + \lambda_{c\to s}(\mathbf{x}_s)\bigr]$$

$$= \hat{\theta}_c(\mathbf{x}_c) + \lambda_c(\mathbf{x}_c) + \sum_{s\in\mathcal{S}(c)\setminus\{c\}}\left[\hat{\theta}_s(\mathbf{x}_s) - \gamma_s(\mathbf{x}_s) + \lambda^{-c}_s(\mathbf{x}_s)\right], \forall \mathbf{x}_c. \tag{86}$$

Using (85) and (86), we can reformulate (33) to:

$$\lambda_{c\to s}^*(\mathbf{x}_s) = -\hat{\theta}_s(\mathbf{x}_s) + \gamma_s(\mathbf{x}_s) - \lambda_s^{-c}(\mathbf{x}_s) + \frac{1}{|\mathcal{S}(c)\setminus\{c\}|}\max_{\mathbf{x}_{c\setminus s}}\left[\sum_{\hat{s}\in\mathcal{S}(c)\setminus\{c\}}\left(\hat{\theta}_{\hat{s}}(\mathbf{x}_{\hat{s}}) - \gamma_{\hat{s}}(\mathbf{x}_{\hat{s}}) + \lambda_{\hat{s}}^{-c}(\mathbf{x}_{\hat{s}})\right) + \hat{\theta}_c(\mathbf{x}_c) + \lambda_c(\mathbf{x}_c)\right]$$

$$= -\left(b_s(\mathbf{x}_s) - \lambda_{c\to s}(\mathbf{x}_s)\right) + \frac{1}{|\mathcal{S}(c)\setminus\{c\}|}\max_{\mathbf{x}_{c\setminus s}}\left[b_c(\mathbf{x}_c) + \sum_{\hat{s}\in\mathcal{S}(c)\setminus\{c\}}b_{\hat{s}}(\mathbf{x}_{\hat{s}})\right], \forall s\in\mathcal{S}(c)\setminus\{c\}, \mathbf{x}_s. \quad (87)$$

Replacing $\lambda_{c\to s}^*(\mathbf{x}_s)$ in (83) with the most RHS of (87) results in:

$$b_s^*(\mathbf{x}_s) = b_s(\mathbf{x}_s) + \lambda_{c\to s}^*(\mathbf{x}_s) - \lambda_{c\to s}(\mathbf{x}_s)$$

$$= b_s(\mathbf{x}_s) - \left(b_s(\mathbf{x}_s) - \lambda_{c\to s}(\mathbf{x}_s)\right) + \frac{1}{|\mathcal{S}(c)\setminus\{c\}|}\max_{\mathbf{x}_{c\setminus s}}\left[b_c(\mathbf{x}_c) + \sum_{\hat{s}\in\mathcal{S}(c)\setminus\{c\}}b_{\hat{s}}(\mathbf{x}_{\hat{s}})\right] - \lambda_{c\to s}(\mathbf{x}_s)$$

$$= \frac{1}{|\mathcal{S}(c)\setminus\{c\}|}\max_{\mathbf{x}_{c\setminus s}}\left[b_c(\mathbf{x}_c) + \sum_{\hat{s}\in\mathcal{S}(c)\setminus\{c\}}b_{\hat{s}}(\mathbf{x}_{\hat{s}})\right], \forall s\in\mathcal{S}(c)\setminus\{c\}, \mathbf{x}_s, \quad (88)$$

and by reformulating (84) using (83) we get

$$b_c^*(\mathbf{x}_c) = b_c(\mathbf{x}_c) + \sum_{\hat{s}\in\mathcal{S}(c)\setminus\{c\}}\boldsymbol{\lambda}_{c\to\hat{s}}(\mathbf{x}_{\hat{s}}) - \sum_{\hat{s}\in\mathcal{S}(c)\setminus\{c\}}\boldsymbol{\lambda}_{c\to\hat{s}}^*(\mathbf{x}_{\hat{s}})$$

$$= b_c(\mathbf{x}_c) + \sum_{\hat{s}\in\mathcal{S}(c)\setminus\{c\}}\mathbf{b}_{\hat{s}}(\mathbf{x}_{\hat{s}}) - \sum_{\hat{s}\in\mathcal{S}(c)\setminus\{c\}}\mathbf{b}_{\hat{s}}^*(\mathbf{x}_{\hat{s}}), \forall \mathbf{x}_c. \quad (89)$$

Together with (88) and (89) we finished the proof. ∎

# 5 Proof of Proposition 7

**Proposition 9** *Given a graph $G = (\mathcal{V}, \mathcal{C})$ and a marginal polytope diagram $G^M = (\mathcal{V}^M, \mathcal{E}^M)$ of $G$, we have*

1. *$\forall c, s, t \in \mathcal{V}^M, t \subset s \subset c$, if $(c \to s) \in \mathcal{E}^M$, then $(c \to t) \Leftrightarrow (s \to t)$;*
2. *If $(c \to s_1), (c \to s_2) \in \mathcal{E}^M$, then $\forall t \in \mathcal{V}^M, t \subset s_1, t \subset s_2, (s_1 \to t) \Leftrightarrow (s_2 \to t)$.*

**Proof** To show the first case, we let $\mathcal{U}_a$ be

$$\mathcal{U}_a = \{\sum_{\mathbf{x}_{\hat{c}\setminus\hat{s}}}\mu_{\hat{c}}(\mathbf{x}_{\hat{c}}) = \mu_{\hat{s}}(\mathbf{x}_{\hat{s}}), \forall \mathbf{x}_{\hat{s}} \,|\, (\hat{c}\to\hat{s}) \in (\mathcal{E}^M\setminus\{(\bar{c}\to t)|\bar{c}\in\mathcal{V}^M\})\}. \quad (90)$$

By the fact that $(c \to s) \in \mathcal{E}^M$, we know that $(c \to s) \in \mathcal{E}^M\setminus\{(\bar{c}\to t)|\bar{c}\in\mathcal{V}^M\}$. Thus for arbitrary $\bar{\boldsymbol{\mu}} \in \{\boldsymbol{\mu}\,|\,\mathcal{U}_a\}$, constraints

$$\sum_{\mathbf{x}_{c\setminus s}}\bar{\mu}_c(\mathbf{x}_c) = \bar{\mu}_s(\mathbf{x}_s), \forall \mathbf{x}_s \quad (91)$$

must be satisfied. Thus if a $\bar{\boldsymbol{\mu}} \in \{\boldsymbol{\mu}\,|\,\mathcal{U}_a\}$ also satisfies $\sum_{\mathbf{x}_{c\setminus t}}\bar{\mu}_c(\mathbf{x}_c) = \bar{\mu}_t(\mathbf{x}_t), \forall \mathbf{x}_t$, we must have

$$\sum_{\mathbf{x}_{s\setminus t}}\bar{\mu}_s(\mathbf{x}_s) = \sum_{\mathbf{x}_{s\setminus t}}\sum_{\mathbf{x}_{c\setminus s}}\bar{\mu}_c(\mathbf{x}_c) = \sum_{\mathbf{x}_{c\setminus t}}\bar{\mu}_c(\mathbf{x}_c) = \bar{\mu}_t(\mathbf{x}_t), \forall \mathbf{x}_t, \quad (92)$$

which implies such a $\bar{\boldsymbol{\mu}}$ also satisfies $\sum_{\mathbf{x}_{s\setminus t}}\bar{\mu}_s(\mathbf{x}_s) = \bar{\mu}_t(\mathbf{x}_t), \forall \mathbf{x}_t$.

On the other hand, if a $\bar{\boldsymbol{\mu}} \in \{\boldsymbol{\mu}\,|\,\mathcal{U}_a\}$ also satisfies $\sum_{\mathbf{x}_{s\setminus t}}\bar{\mu}_s(\mathbf{x}_s) = \bar{\mu}_t(\mathbf{x}_t), \forall \mathbf{x}_t$, we must have

$$\sum_{\mathbf{x}_{c\setminus t}}\bar{\mu}_c(\mathbf{x}_c) = \sum_{\mathbf{x}_{s\setminus t}}\sum_{\mathbf{x}_{c\setminus s}}\bar{\mu}_c(\mathbf{x}_c) = \sum_{\mathbf{x}_{s\setminus t}}\bar{\mu}_s(\mathbf{x}_s) = \bar{\mu}_t(\mathbf{x}_t), \forall \mathbf{x}_t, \quad (93)$$

which implies such a $\bar{\boldsymbol{\mu}}$ also satisfies $\sum_{\mathbf{x}_{c\setminus t}} \bar{\mu}_c(\mathbf{x}_c) = \bar{\mu}_t(\mathbf{x}_t), \forall \mathbf{x}_t$. Thus we have

$$\Big\{\boldsymbol{\mu} \,|\, \mathcal{U}_a \cup \{\sum_{\mathbf{x}_{c\setminus t}} \mu_c(\mathbf{x}_c) = \mu_t(\mathbf{x}_t), \forall \mathbf{x}_t\}\Big\} = \Big\{\boldsymbol{\mu} \,|\, \mathcal{U}_a \cup \{\sum_{\mathbf{x}_{s\setminus t}} \mu_s(\mathbf{x}_s) = \mu_t(\mathbf{x}_t), \forall \mathbf{x}_t\}\Big\},$$

which shows that edges $(c \to t)$ and $(s \to t)$ are equivalent by definition.

For the second case, we define

$$\mathcal{U}_b = \Big\{\sum_{\mathbf{x}_{\hat{c}\setminus\hat{s}}} \mu_{\hat{c}}(\mathbf{x}_{\hat{c}}) = \mu_{\hat{s}}(\mathbf{x}_{\hat{s}}),\ \forall \mathbf{x}_{\hat{s}} \,|\, (\hat{c} \to \hat{s}) \in (\mathcal{E}^M \setminus \{(\bar{c} \to t) | c \in \mathcal{V}^M\})\Big\}. \tag{94}$$

By the fact that $(c \to s_1), (c \to s_2) \in \mathcal{E}^M$, we must have $(c \to s_1), (c \to s_2) \in \mathcal{E}^M \setminus \{(\bar{c} \to t) | c \in \mathcal{V}^M\}$. Thus for arbitrary $\bar{\boldsymbol{\mu}} \in \{\boldsymbol{\mu} \,|\, \mathcal{U}_b\}$, constraints

$$\sum_{\mathbf{x}_{c\setminus s_1}} \bar{\mu}_c(\mathbf{x}_c) = \bar{\mu}_{s_1}(\mathbf{x}_{s_1}), \sum_{\mathbf{x}_{c\setminus s_2}} \bar{\mu}_c(\mathbf{x}_c) = \bar{\mu}_{s_2}(\mathbf{x}_{s_2}), \forall \mathbf{x}_{s_1}, \mathbf{x}_{s_2}$$

must be satisfied. Thus if a $\bar{\boldsymbol{\mu}} \in \{\boldsymbol{\mu} \,|\, \mathcal{U}_b\}$ also satisfies $\sum_{\mathbf{x}_{s_1\setminus t}} \bar{\mu}_{s_1}(\mathbf{x}_{s_1}) = \bar{\mu}_t(\mathbf{x}_t), \forall \mathbf{x}_t$, we have

$$\sum_{\mathbf{x}_{s_2\setminus t}} \bar{\mu}_{s_2}(\mathbf{x}_{s_2}) = \sum_{\mathbf{x}_{s_2\setminus t}} \sum_{\mathbf{x}_{c\setminus s_2}} \bar{\mu}_c(\mathbf{x}_c) = \sum_{\mathbf{x}_{c\setminus t}} \bar{\mu}_c(\mathbf{x}_c) = \sum_{\mathbf{x}_{c\setminus s_1}} \sum_{\mathbf{x}_{s_1\setminus t}} \bar{\mu}_c(\mathbf{x}_c) = \bar{\mu}_t(\mathbf{x}_t), \forall \mathbf{x}_t, \tag{95}$$

which implies such $\bar{\boldsymbol{\mu}}$ also satisfies $\sum_{\mathbf{x}_{s_2\setminus t}} \bar{\mu}_{s_2}(\mathbf{x}_{s_2}) = \bar{\mu}_t(\mathbf{x}_t)$.

On the other hand, if a $\bar{\boldsymbol{\mu}} \in \{\boldsymbol{\mu} \,|\, \mathcal{U}_b\}$ also satisfies $\sum_{\mathbf{x}_{s_2\setminus t}} \bar{\mu}_{s_2}(\mathbf{x}_{s_2}) = \bar{\mu}_t(\mathbf{x}_t)$, we have

$$\sum_{\mathbf{x}_{s_1\setminus t}} \bar{\mu}_{s_1}(\mathbf{x}_{s_1}) = \sum_{\mathbf{x}_{s_1\setminus t}} \sum_{\mathbf{x}_{c\setminus s_1}} \bar{\mu}_c(\mathbf{x}_c) = \sum_{\mathbf{x}_{c\setminus t}} \bar{\mu}_c(\mathbf{x}_c) = \sum_{\mathbf{x}_{s_2\setminus t}} \sum_{\mathbf{x}_{c\setminus s_2}} \bar{\mu}_c(\mathbf{x}_c) = \bar{\mu}_t(\mathbf{x}_t), \tag{96}$$

which implies such $\bar{\boldsymbol{\mu}}$ also satisfies $\sum_{\mathbf{x}_{s_1\setminus t}} \bar{\mu}_{s_1}(\mathbf{x}_{s_1}) = \bar{\mu}_t(\mathbf{x}_t)$. Thus we have

$$\Big\{\boldsymbol{\mu} \,|\, \mathcal{U}_b \cup \{\sum_{\mathbf{x}_{s_1\setminus t}} \mu_{s_1}(\mathbf{x}_{s_1}) = \mu_t(\mathbf{x}_t), \forall \mathbf{x}_t\}\Big\} = \Big\{\boldsymbol{\mu} \,|\, \mathcal{U}_b \cup \{\sum_{\mathbf{x}_{s_2\setminus t}} \mu_{s_2}(\mathbf{x}_{s_2}) = \mu_t(\mathbf{x}_t), \forall \mathbf{x}_t\}\Big\},$$

which shows that $(s_1 \to t)$ and $(s_2 \to t)$ are equivalent by definition. ∎

## 6 Proof of Proposition 8

**Proposition 10** *For arbitrary $G = (\mathcal{V}, \mathcal{C})$, and marginal polytope diagram $G^M = (\mathcal{V}^M, \mathcal{E}^M)$ of $G$, edge equivalence w.r.t. $G^M$ is an equivalence relation.*

**Proof** By definition, reflexivity and symmetry naively holds. Thus we only prove the transitivity by proving the claim that $\forall c_1, c_2, c_3, t \in \mathcal{V}^M$, $t \subseteq c_1$, $t \subseteq c_2$, $t \subseteq c_3$, if $(c_1 \to t) \Leftrightarrow (c_2 \to t)$ and $(c_2 \to t) \Leftrightarrow (c_3 \to t)$ are true, then $(c_1 \to t) \Leftrightarrow (c_3 \to t)$ must be true.

Now we prove the claim. Let $\mathcal{U}$ be

$$\mathcal{U} = \Big\{\sum_{\mathbf{x}_{c\setminus s}} \mu_c(\mathbf{x}_c) = \mu_s(\mathbf{x}_s), \forall (c \to s) \in (\mathcal{E}^M \setminus \{(\hat{c} \to t) | \hat{c} \in \mathcal{V}^M, t \subseteq \hat{c}\}), \mathbf{x}_s\Big\}.$$

Then by definition of edge equivalence, we must have

$$\Big\{\boldsymbol{\mu} \,|\, \mathcal{U} \cup \{\sum_{\mathbf{x}_{c_1\setminus t}} \mu_{c_1}(\mathbf{x}_{c_1}) = \mu_t(\mathbf{x}_t), \forall \mathbf{x}_t\}\Big\} = \Big\{\boldsymbol{\mu} \,|\, \mathcal{U} \cup \{\sum_{\mathbf{x}_{c_2\setminus t}} \mu_{c_2}(\mathbf{x}_{c_2}) = \mu_t(\mathbf{x}_t), \forall \mathbf{x}_t\}\Big\},$$

$$\Big\{\boldsymbol{\mu} \,|\, \mathcal{U} \cup \{\sum_{\mathbf{x}_{c_2\setminus t}} \mu_{c_2}(\mathbf{x}_{c_2}) = \mu_t(\mathbf{x}_t), \forall \mathbf{x}_t\}\Big\} = \Big\{\boldsymbol{\mu} \,|\, \mathcal{U} \cup \{\sum_{\mathbf{x}_{c_3\setminus t}} \mu_{c_3}(\mathbf{x}_{c_3}) = \mu_t(\mathbf{x}_t), \forall \mathbf{x}_t\}\Big\}, \tag{97}$$

which implies that

$$\left\{\boldsymbol{\mu}\,|\,\mathcal{U}\cup\{\sum_{\mathbf{x}_{c_1\setminus t}}\mu_{c_1}(\mathbf{x}_{c_1})=\mu_t(\mathbf{x}_t),\forall\mathbf{x}_t\}\right\}=\left\{\boldsymbol{\mu}\,|\,\mathcal{U}\cup\{\sum_{\mathbf{x}_{c_3\setminus t}}\mu_{c_3}(\mathbf{x}_{c_3})=\mu_t(\mathbf{x}_t).\forall\mathbf{x}_t\}\right\}. \tag{98}$$

Thus we must have $(c_1 \to t) \Leftrightarrow (c_3 \to t)$ by definition.
Over all, edge equivalence w.r.t. $G^M$ is an equivalence relation. ∎

## 7 Proof of Proposition 11

**Proposition 15** *All $v \in \mathcal{V}_0^M \setminus (\mathcal{C} \cup \mathcal{I}_m)$ are redundant nodes w.r.t. $G_0^M$.*

**Proof**
Since $\mathcal{C} \subseteq (\mathcal{C} \cup \mathcal{I}_m)$, for any $v \in \mathcal{V}_0^M \setminus (\mathcal{C} \cup \mathcal{I}_m)$, we have $v \in \mathcal{V}_0^M \setminus \mathcal{C}$. Now we prove the proposition by proving that all edges to $v$ are equivalent.
Let $P_v = \{p|(p \to v) \in \mathcal{E}_0^M\}$. By definition of $\mathcal{C}_m$, we have $\forall p_1, p_2 \in P_v$, $\exists c_1, c_2 \in \mathcal{C}_m$, s.t. $p_1 \subseteq c_1, p_2 \subseteq c_2$. Thus let $s = c_1 \cap c_2$, by the fact that $v \subseteq p_1 \subseteq c_1$ and $v \subseteq p_2 \subseteq c_2$, we must have $v \subseteq s$. Moreover, if $v = s$ we must have $v = c_1 \cap c_2 \in \mathcal{I}_m$, which contradict to the fact that $v \in \mathcal{V}^M \setminus (\mathcal{C} \cup \mathcal{I}_m)$. Thus we must have $v \subset s$. By the fact $s = c_1 \cap c_2$, we have $s \subseteq c_i, i \in \{1, 2\}$. Thus $\forall i \in \{1, 2\}$, if $p_i = c_i = s$, $(p_i \to v) \Leftrightarrow (s \to v)$ naively holds; if only one of $p_i$ and $s$ is equal to $c_i$, we have $(p_i \to v) \Leftrightarrow (s \to v)$ by Proposition 7 (the first case); if both $p_i$ and $s$ are not equal to $c_i$, by Proposition 7 (the second case) we have $(p_i \to v) \Leftrightarrow (s \to v)$. Thus by transitivity we have all $(p \to v), p \in P_v$ are equivalent, which implies that $v$ is redundant node w.r.t. $G_0^M$ by Proposition 9. ∎

## 8 Experiment

We present more experiment here.

### 8.1 Results on Synthetic Data

Additional results on the synthetic data for the convergence in terms of both running time and the number of iterations are provided in Figure 1, which consistently shows faster convergence of the proposed methods.

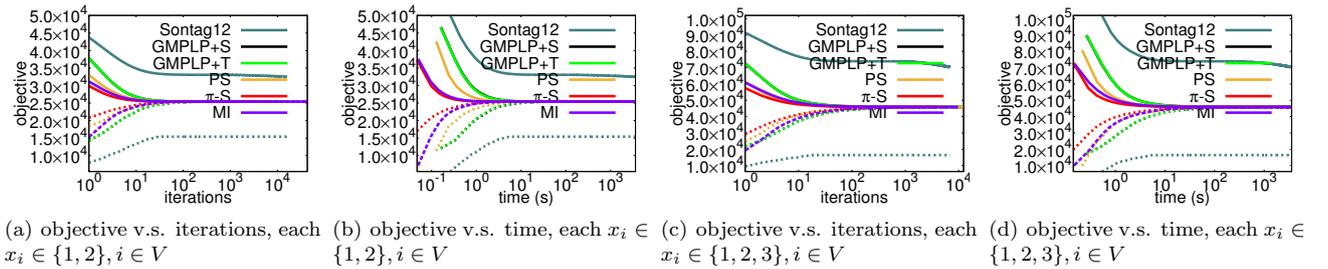

(a) objective v.s. iterations, each $x_i \in \{1, 2\}, i \in V$
(b) objective v.s. time, each $x_i \in \{1, 2\}, i \in V$
(c) objective v.s. iterations, each $x_i \in \{1, 2, 3\}, i \in V$
(d) objective v.s. time, each $x_i \in \{1, 2, 3\}, i \in V$

Figure 1: Objective decrease on synthetic data

### 8.2 Results on PPI dataset

Additional results on PPI dataset for the convergence in terms of both running time and the number of iterations are provided in Figure 2.

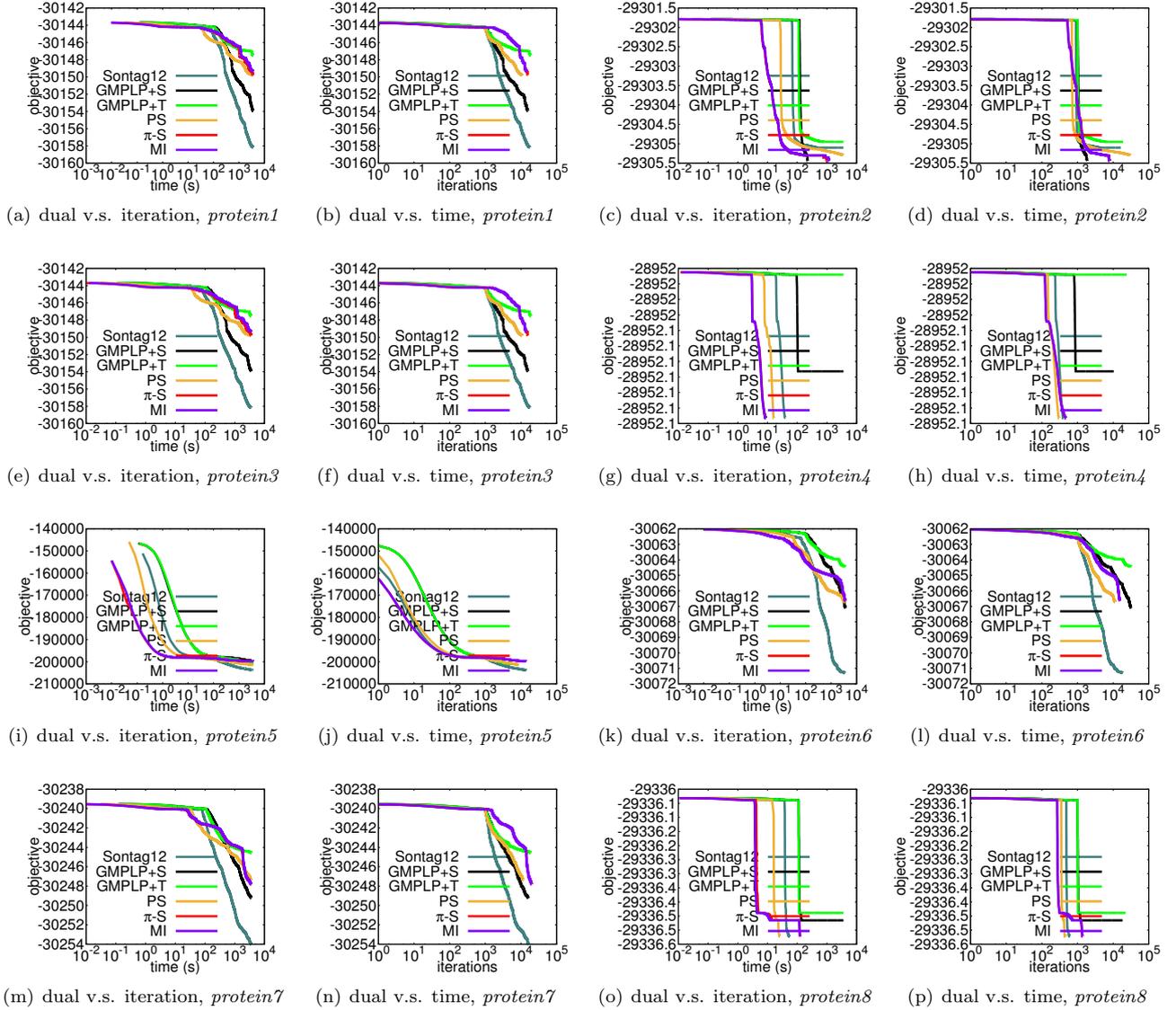

Figure 2: Dual objective decrease of inference problems in PPIs

### 8.3 Image segmentation

Here the node potentials $\theta_i(x_i), i \in \mathcal{V}$ are computed according to (5) in Kohli et al [2009] as follows,

$$\theta_i(x_i) = \theta_T \varphi_T(x_i) + \theta_{col}\varphi_{col}(x_i) + \theta_l \varphi_l(x_i). \tag{99}$$

We choose $\theta_T = 0$, $\theta_{col} = 1$ and $\theta_l = 0$, thus we have $\theta_i(x_i) = \varphi_{col}(x_i)$. We learn $\varphi_{col}(x_i)$ from the data using Gaussian Mixture Models as in Blake et al [2004]. We follow (12) and (10) in Kohli et al [2009] to compute edge potentials and high order potentials (with $\theta_\alpha = 0$, $\theta_p^h = 0$, $\theta_v^h = 25$ and $\theta_\beta^h$ being set to the reciprocal of variance of all pixels *i.e.* grey value in $[0, 255]$).

Additional results are provided in Figures 3 and 4. We can see that the proposed methods, PS, $\pi$-S and MI, not only find the exact solution on all three problems, but also converge much faster than others.

11

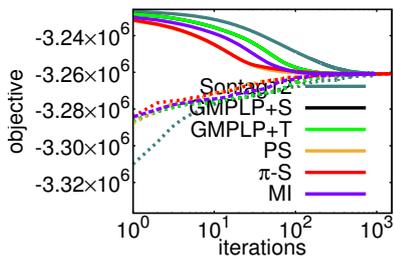
(a) objective v.s. iteration, *banana1*
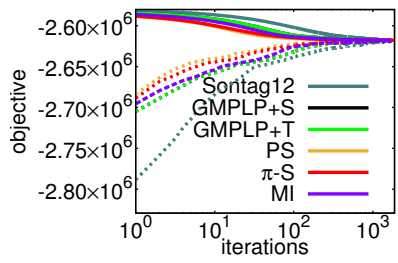
(b) objective v.s. iteration, *bool*
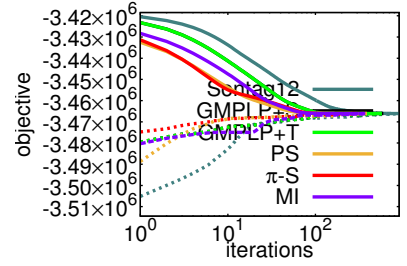
(c) objective v.s. iteration, *book*
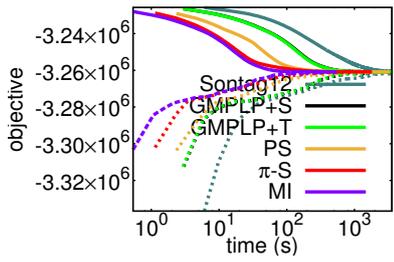
(d) objective v.s. time, *banana1*
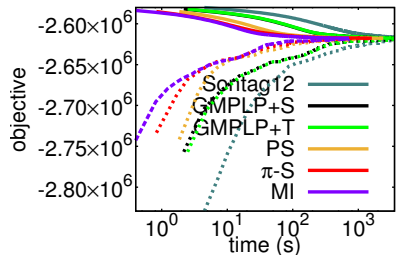
(e) objective v.s. time, *bool*
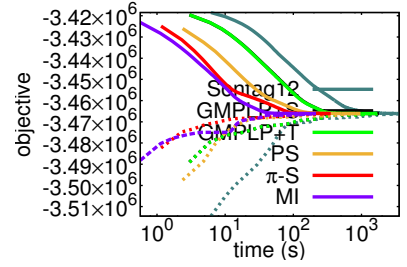
(f) objective v.s. time, *book*

Figure 3: Objective decrease of inference in image segmentation

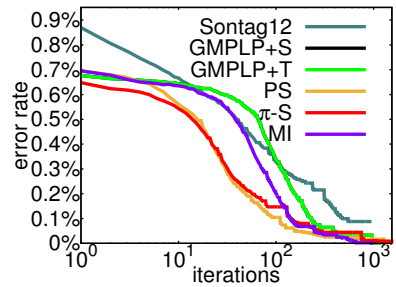
(a) objective v.s. iteration, *banana1*
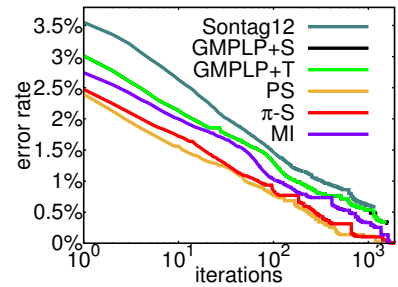
(b) objective v.s. iteration, *bool*
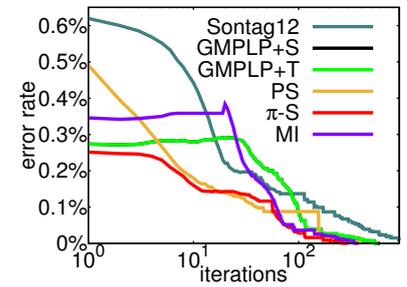
(c) objective v.s. iteration, *book*
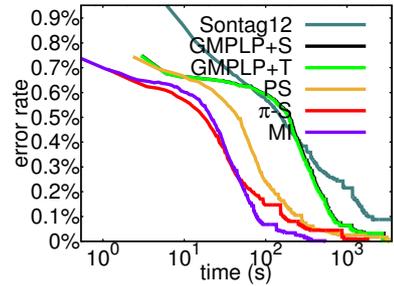
(d) objective v.s. time, *banana1*
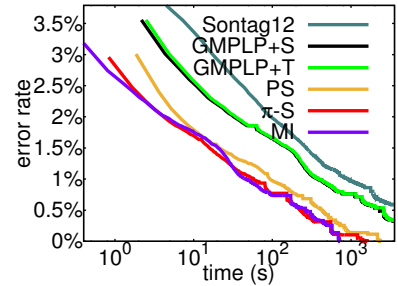
(e) objective v.s. time, *bool*
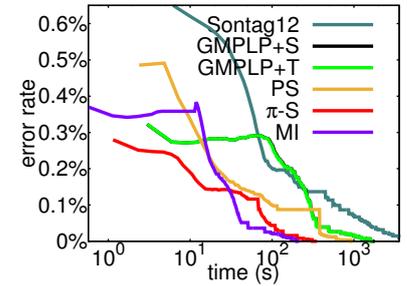
(f) objective v.s. time, *book*

Figure 4: Error rate decrease of inference in image segmentation

12

## 8.4 Image Matching

We first detect key points from source images and destination images by using SIFT detectors implemented in OpenCV with default parameters. As in Li et al [2010], for every $p(i), i \in \mathcal{V}$, if its nearest 3 neighbours $p(i_1)$, $p(i_2)$, $p(i_3)$ are not in a line, there must exist a column vector $W'_f$ in $\mathbb{R}^3$ s.t.

$$p(i) = [p(i_1), p(i_2), p(i_3)]W'_f. \tag{100}$$

Then let $W_f = [W'_f, -1]$ we must have

$$[p(i_1), p(i_2), p(i_3), p(i)]W_f = 0, \tag{101}$$

and the equation is invariant to affine transformation Li et al [2010]. Thus for every $p(i), i \in \mathcal{V}$, if its nearest 3 nearest neighbours $p(i_1)$, $p(i_2)$, $p(i_3)$ are not in a line, there is an order-4 cluster $f = \{i_1, i_2, i_3, i\}$. Let $P_{\mathbf{x}_f} = [q(x_{i_1}), q(x_{i_2}), q(x_{i_3}), q(x_i)]$. $||P_{\mathbf{x}_f} W_f||_1$ can be used as a geometry cost.

Additional results are provided in the following figures. Although all algorithms achieves exact solutions on all data sets, the proposed methods often show better convergence rate in terms of both iterations and running time. The plot of dataset $ubc45$ is excluded since all algorithms achieve the exact solution at the first iteration.

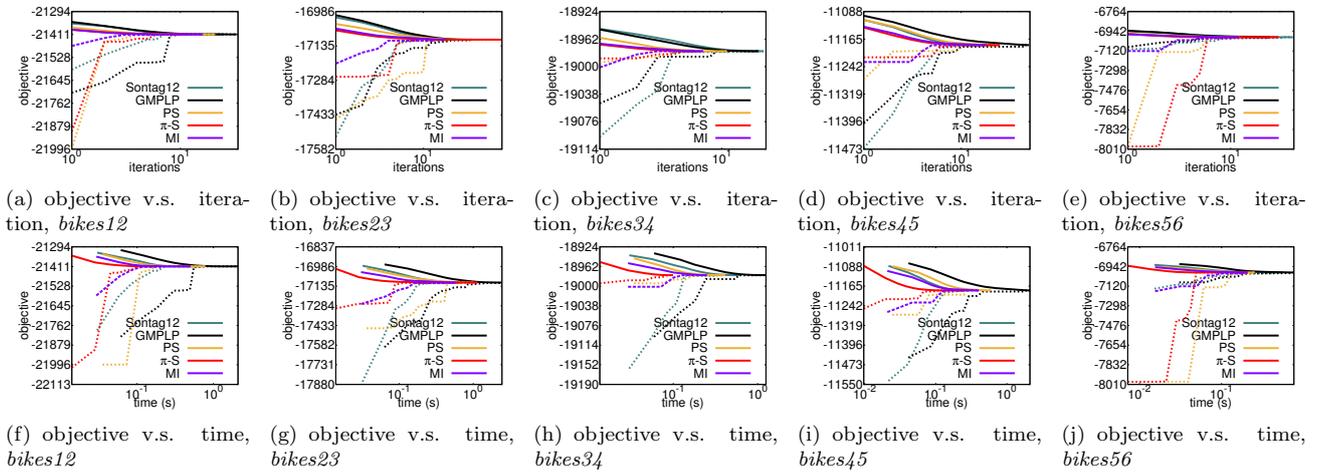

Figure 5: Objective decrease of inference in image matching, *bikes*

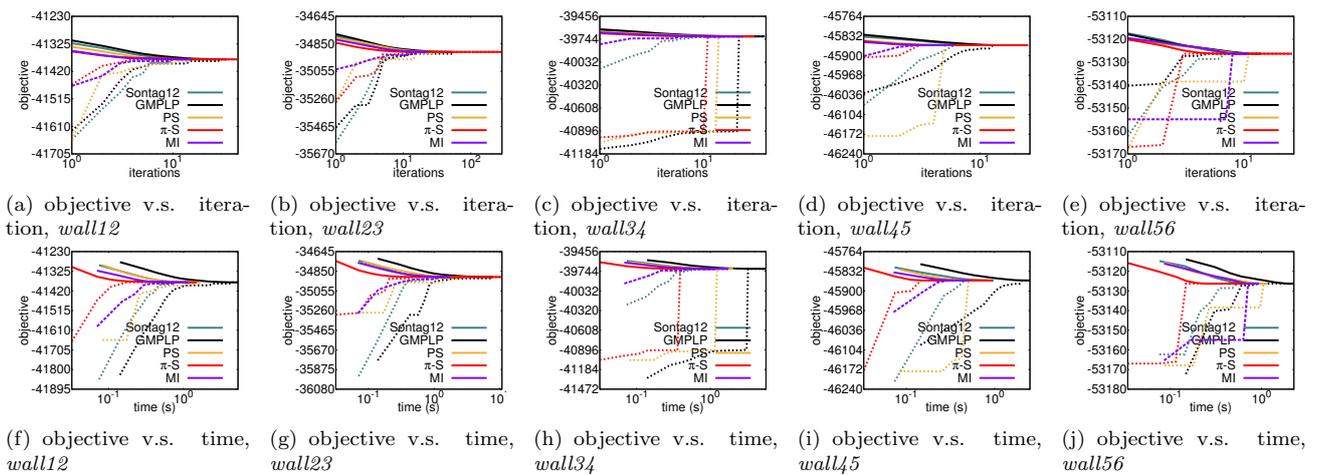

Figure 6: Objective decrease of inference in image matching, *wall*

13

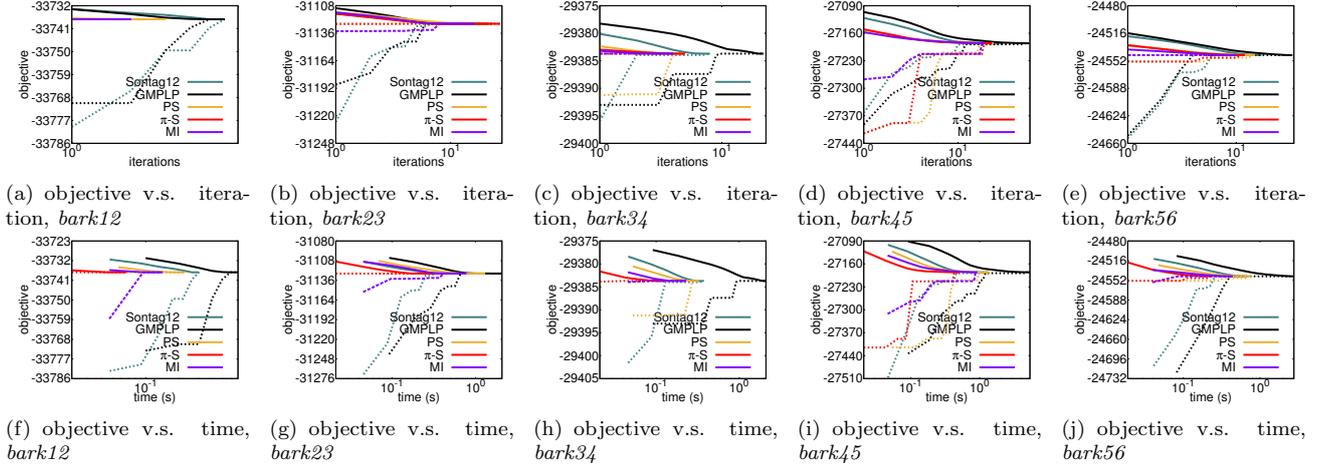

Figure 7: Objective decrease of inference in image matching, *bark*

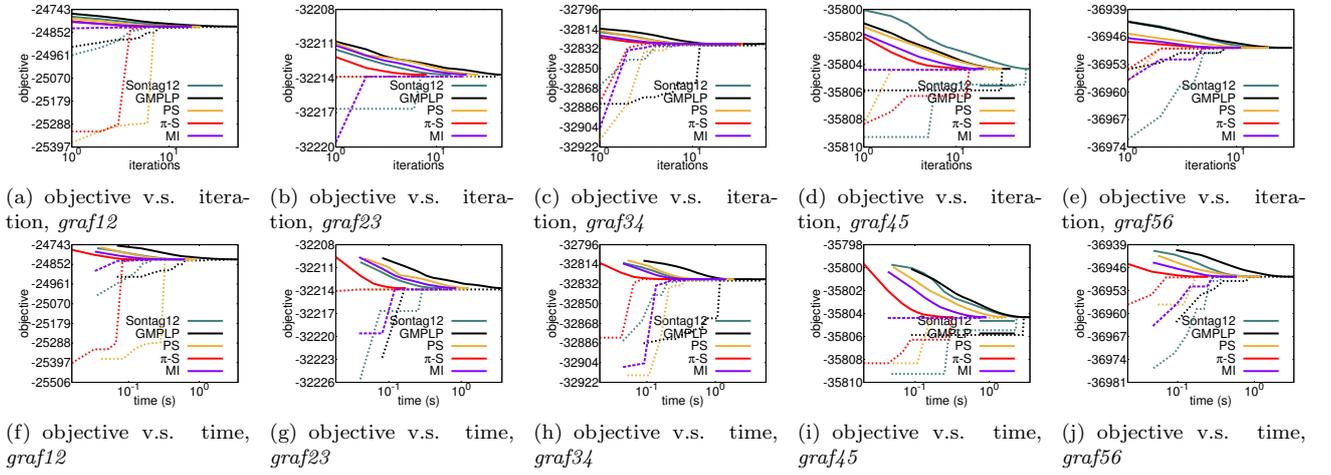

Figure 8: Objective decrease of inference in image matching, *graf*

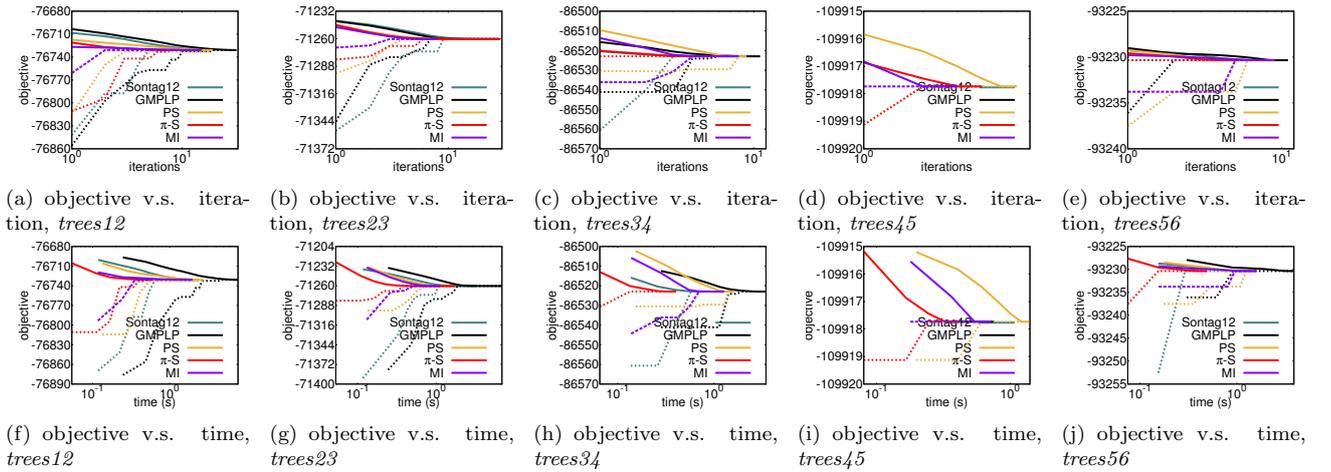

Figure 8: Objective decrease of inference in image matching, *trees*

14

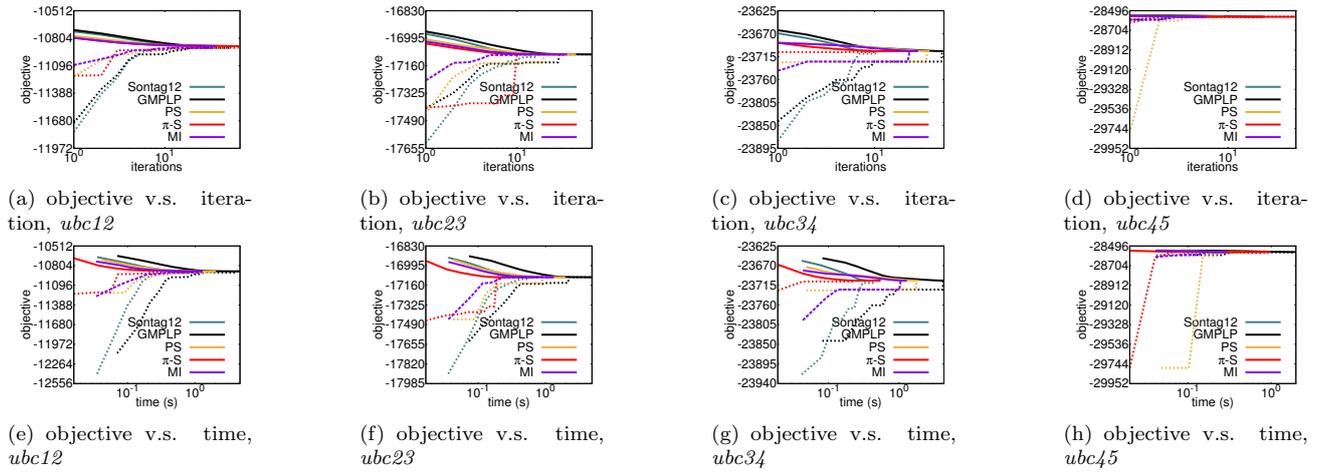

Figure 9: Objective decrease of inference in image matching, *ubc*

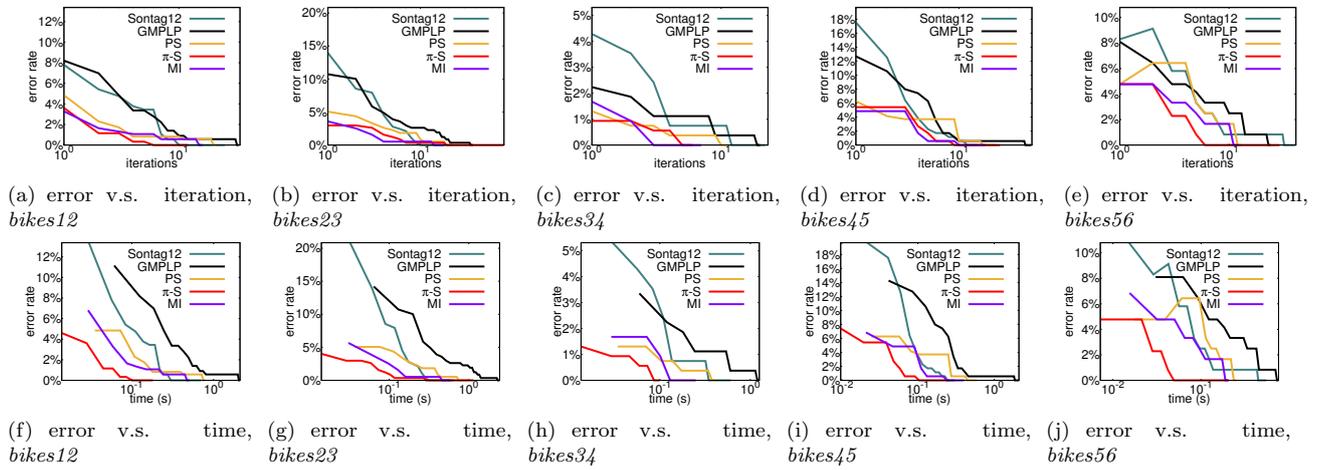

Figure 10: Error rate decrease of inference in image matching, *bikes*

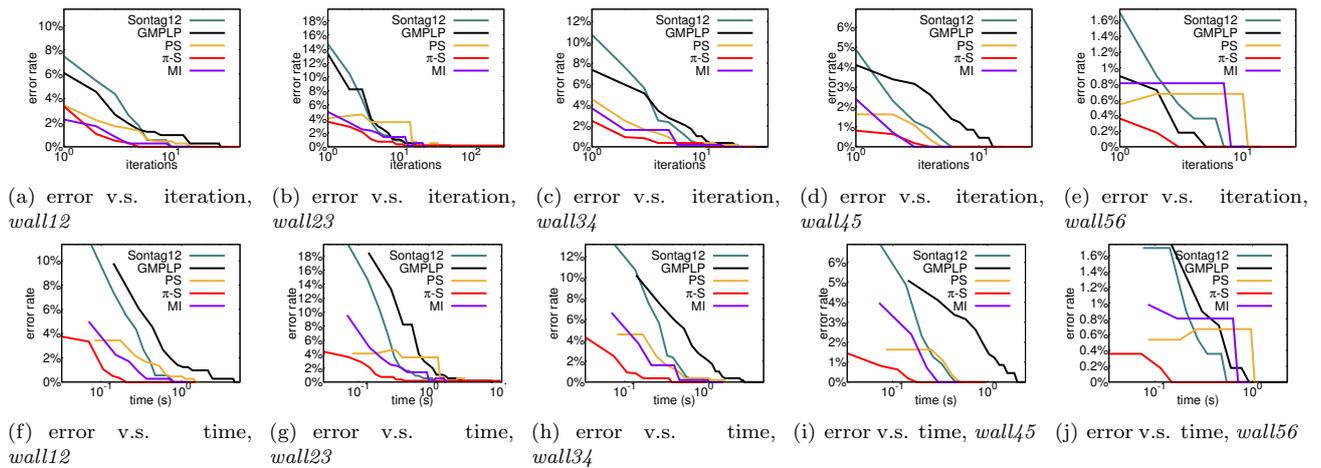

Figure 11: Error rate decrease of inference in image matching, *wall*

15

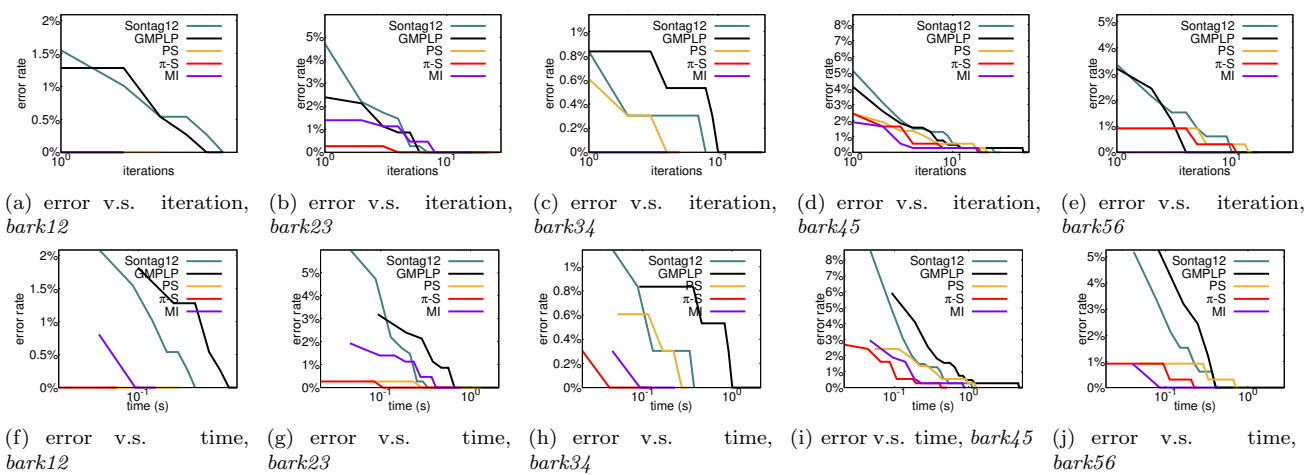

Figure 12: Error rate decrease of inference in image matching, *bark*

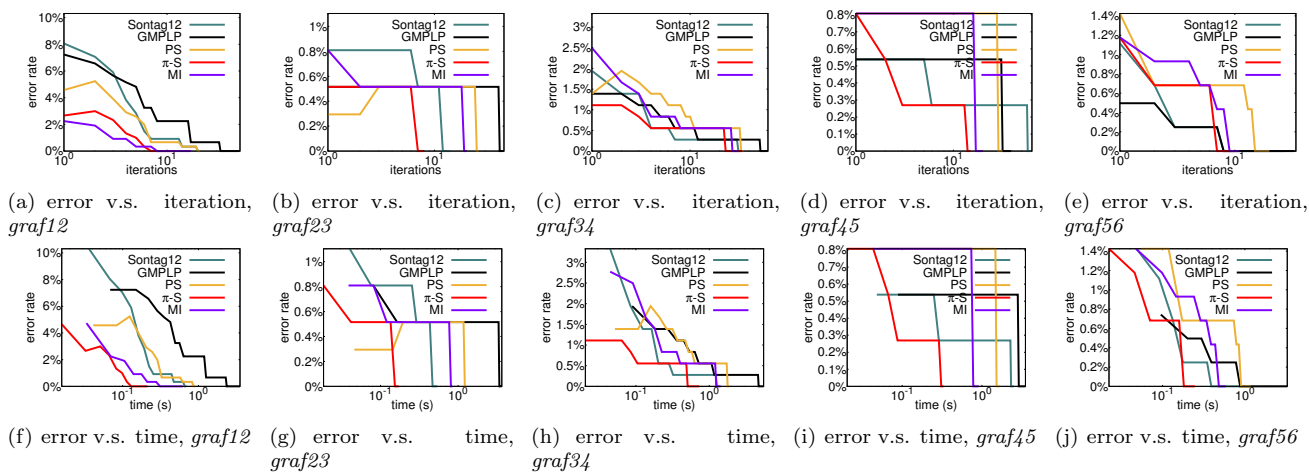

Figure 12: Error rate decrease of inference in image matching, *graf*

16


\end{document}